\documentclass[11pt]{article} %
\pdfoutput=1
\usepackage[T1]{fontenc}

\usepackage{lmodern} \normalfont %
\usepackage{anyfontsize}
\DeclareFontShape{T1}{lmr}{bx}{sc} { <-> ssub * cmr/bx/sc }{}
\DeclareFontShape{T1}{lmr}{m}{scit}{ <-> ssub * cmr/m/sc }{}
\DeclareFontShape{T1}{lmr}{bx}{scit}{ <-> ssub * cmr/bx/sc }{}

\usepackage{etoolbox}
\newtoggle{preprint}
\newcommand{\preprint}[1]{\iftoggle{preprint}{#1}{}}
\newcommand{\aos}[1]{\iftoggle{preprint}{}{#1}}

\toggletrue{preprint} %
\preprint{
\usepackage{
	amsmath,
	amsthm,
	amssymb,
	wrapfig,
	amssymb,
	cases,
	mathtools,
	thmtools,
	array,
	bbm,
	booktabs,
	upgreek,
	xkvltxp,
	fixme,
	stmaryrd,
	listings,
	multirow,
	xargs,
	xstring,
	multicol,
	graphicx,
	float,
	color,
	enumerate,
	algpseudocode,
	indentfirst,
	accents,
	ifthen,
	wasysym,
	tikz,
	lmodern,
	stmaryrd,
	wasysym,
	caption,
	subcaption,
	multicol,
	mathrsfs
	}

\usepackage[utf8]{inputenc}
\usepackage[ruled,vlined]{algorithm2e}
\usepackage[colorlinks=true, linkcolor=blue!20!black, citecolor=blue!20!black,urlcolor=black,breaklinks=true]{hyperref}

\usepackage[square,sort,comma,numbers]{natbib}
\setcitestyle{citesep={,},aysep={,},yysep={,}}

\usepackage[T2A,T1]{fontenc}
\DeclareSymbolFont{cyrillic}{T2A}{cmr}{m}{n}
\DeclareMathSymbol{\sha}{\mathalpha}{cyrillic}{248}

\newcolumntype{C}[1]{>{\centering}m{#1}}
\lstset{basicstyle=\ttfamily\small,columns=fullflexible,keepspaces=true}

\usetikzlibrary{shapes,arrows,trees,automata,positioning}

\makeatletter
\renewcommand{\maketag@@@}[1]{\hbox{\m@th\normalsize\normalfont#1}}%
\makeatother

\makeatletter
\let\reftagform@=\tagform@
\def\tagform@#1{\maketag@@@{\ignorespaces\textcolor{gray}{(#1)}\unskip\@@italiccorr}}
\renewcommand{\eqref}[1]{\textup{\reftagform@{\ref{#1}}}}
\makeatother

\usepackage[capitalize,nameinlink]{cleveref}
\usepackage{crossreftools}
\pdfstringdefDisableCommands{%
    \let\Cref\crtCref
    \let\cref\crtcref
}

\hypersetup{%
    bookmarksnumbered, bookmarksopen=true, bookmarksopenlevel=1,%
}

\crefname{subsection}{Section}{Sections}
\crefname{lemma}{Lemma}{Lemmas}
\crefname{corollary}{Corollary}{Corollaries}
\crefname{theorem}{Theorem}{Theorems}

\declaretheorem[name=Theorem]{theorem}

\declaretheorem[name=Lemma]{lemma}
\declaretheorem[name=Proposition]{proposition}
\declaretheorem[name=Corollary]{corollary}

\declaretheorem[name=Definition]{definition}

\declaretheorem[name=Assumption, numbered=no]{assumption*}
\declaretheorem[qed=$\triangleleft$,name=Example]{example}
\declaretheorem[qed=$\triangleleft$,name=Remark]{remark}

\newcommand{\EE}{\mathbb{E}}

\newcommand{\II}{\mathbb{I}}

\newcommand{\NN}{\mathbb{N}}

\newcommand{\RR}{\mathbb{R}}

\newcommand{\ZZ}{\mathbb{Z}}

\newcommand{\Aa}{\mathcal{A}}
\newcommand{\Bb}{\mathcal{B}}

\newcommand{\Dd}{\mathcal{D}}

\newcommand{\Hh}{\mathcal{H}}

\newcommand{\Kk}{\mathcal{K}}

\newcommand{\Mm}{\mathcal{M}}

\newcommand{\Oo}{\mathcal{O}}

\newcommand{\Xx}{\mathcal{X}}
\newcommand{\Yy}{\mathcal{Y}}

\newcommand{\one}{\mathbf{1}}

\def\[#1\]{\begin{equation}\begin{aligned}#1\end{aligned}\end{equation}}
\def\*[#1\]{\begin{equation*}\begin{aligned}#1\end{aligned}\end{equation*}}

\def\s*[#1\s]{\small\begin{align*}#1\end{align*}\normalsize}

\newcommand{\lcrx}[4][{-1}]{
	\IfEq{#1}{-1}{\left #2 {{{{#3}}}} \right #4}{
   	\IfEq{#1}{0}{#2 {{{{#3}}}} #4}{
	\IfEq{#1}{1}{\bigl #2 {{{{#3}}}} \bigr #4}{
	\IfEq{#1}{2}{\Bigl #2 {{{{#3}}}} \Bigr #4}{
	\IfEq{#1}{3}{\biggl #2 {{{{#3}}}} \biggr #4}{
	\IfEq{#1}{4}{\Biggl #2 {{{{#3}}}} \Biggr #4}{
    \GenericWarning{"4th argument to lcrx must be -1, 0, 1, 2, 3, or 4"}
    }}}}}}} %

\newcommand{\stk}[2]{\ensuremath{\stackrel{\text{#2}}{#1}}}

\newcommand{\upper}[1]{^{(#1)}}

\makeatletter
\newcommand{\subalign}[1]{%
  \vcenter{%
    \Let@ \restore@math@cr \default@tag
    \baselineskip\fontdimen10 \scriptfont\tw@
    \advance\baselineskip\fontdimen12 \scriptfont\tw@
    \lineskip\thr@@\fontdimen8 \scriptfont\thr@@
    \lineskiplimit\lineskip
    \ialign{\hfil$\m@th\scriptstyle##$&$\m@th\scriptstyle{}##$\crcr
      #1\crcr
    }%
  }
}
\makeatother

\renewcommand{\Pr}{\mathbb{P}} %
\def\EE{\mathbb{E}} %

\newcommand{\EEE}[1]{\underset{#1}{\EE}}

\DeclareMathOperator*{\Var}{Var} %
\newcommand{\VVar}[1]{\underset{#1}{\Var}} %

\newcommand{\stT}{\ \text{s.t.}\ }
\newcommand{\andT}{\ \text{and}\ }

\DeclareMathOperator{\sign}{sign}
\newcommand{\ind}[2][{-1}]{\II_{\sbra[#1]{{#2}}}} %

\def\multiset#1#2{\ensuremath{\left(\kern-.3em\left(\genfrac{}{}{0pt}{}{#1}{#2}\right)\kern-.3em\right)}}

\DeclareMathOperator*{\argmin}{\arg\min} %
\DeclareMathOperator*{\newlim}{\mathrm{lim}\vphantom{\mathrm{infsup}}}
\DeclareMathOperator*{\newmin}{\mathrm{min}\vphantom{\mathrm{infsup}}}
\DeclareMathOperator*{\newmax}{\mathrm{max}\vphantom{\mathrm{infsup}}}
\DeclareMathOperator*{\newinf}{\mathrm{inf}\vphantom{\mathrm{infsup}}}
\DeclareMathOperator*{\newsup}{\mathrm{sup}\vphantom{\mathrm{infsup}}}
\renewcommand{\lim}{\newlim}
\renewcommand{\min}{\newmin}
\renewcommand{\max}{\newmax}
\renewcommand{\inf}{\newinf}
\renewcommand{\sup}{\newsup}

\newcommand{\Tr}{^\text{T}} %
\DeclareMathOperator{\diag}{diag} %
\newcommand{\dee}{\mathrm{d}} %
\newcommand{\grad}{\nabla} %
\newcommand{\hess}{\grad^2} %

\DeclareMathOperator{\supp}{Support}

\newcommand{\distas}{\sim}
\newcommand{\distiidas}{\stk{\distas}{iid}}

\newcommand{\pushfwdmeas}[2]{{{#1}_{\sharp} #2}}

\newcommand{\binomdist}{\mathrm{Bin}}
\newcommand{\bernoullidist}{\mathrm{Ber}}

\newcommand{\normaldist}{\mathrm{N}}

\newcommand{\unifdist}{\mathrm{Unif}}

\newcommand{\rbra}[2][{-1}]{\lcrx[#1] ( {#2} ) }
\newcommand{\cbra}[2][{-1}]{\lcrx[#1] \{ {#2} \} }
\newcommand{\sbra}[2][{-1}]{\lcrx[#1] [ {#2} ] }

\newcommand{\abs}[2][{-1}]{\lcrx[#1] \vert {#2} \vert }
\newcommand{\set}[2][{-1}]{\lcrx[#1] \{ {#2} \}}
\newcommand{\floor}[2][{-1}]{\lcrx \lfloor {#2} \rfloor}
\newcommand{\ceil}[2][{-1}]{\lcrx[#1] \lceil {#2} \rceil}
\newcommand{\norm}[2][{-1}]{\lcrx[#1] \Vert {#2} \Vert}
\newcommand{\inner}[3][{-1}]{\lcrx[#1] \langle {{#2},\ {#3}} \rangle}
\newcommand{\card}[2][{-1}]{\lcrx[#1] \vert {#2} \vert }

\newcommand{\defas}{=}

\newcommand{\Nats}{\NN}
\newcommand{\NatsO}{\Nats\cup\set{0}}
\newcommand{\Ints}{\ZZ}

\newcommand{\Reals}{\RR}
\newcommand{\ExtReals}{\overline{\Reals}}
\newcommand{\PosInts}{\Ints_+}
\newcommand{\PosReals}{\Reals_+}

\newcommand{\range}[2][{1}]{
	\IfEq{#1}{1}{\sbra{#2}}{\sbra{#2}_{#1}}}
\newcommand{\rangeO}[2][{0}]{
	\IfEq{#1}{0}{\sbra{#2}_0}{\sbra{#2}_{#1}}}

\newcommand{\ointer}[2][{-1}]{\lcrx[#1] ( {#2} ) }
\newcommand{\cinter}[2][{-1}]{\lcrx[#1] [ {#2} ] }

\newcommand{\union}{\cup}

\DeclareMathOperator{\interior}{interior}
\newcommand{\boundary}{\partial}

\newcommand{\el}[1]{
	\IfEqCase{#1}{
		{0}{x}
		{1}{y}
		{2}{z}
		{3}{w}
		{4}{v}
		{5}{u}
		}
	[\PackageError{el}{Undefined option to el: #1}{}]
}

\newcommand{\se}[1]{%
	\IfEqCase{#1}{
		{0}{A}
		{1}{B}
		{2}{C}
		{3}{D}
		{4}{E}
		{5}{F}
		}
	[\PackageError{se}{Undefined option to se: #1}{}]
}

\newcommand{\SA}[1]{%
	\IfEqCase{#1}{
		{0}{\Sigma}
		}
	[\PackageError{SA}{Undefined option to se: #1}{}]
}

\newcommand{\fu}[1]{%
	\IfEqCase{#1}{
		{0}{f}
		{1}{g}
		{2}{h}tel
		{3}{a}
		{4}{b}
		{5}{c}
		}
	[\PackageError{se}{Undefined option to se: #1}{}]
}

\newcommand{\tel}[1]{%
	\IfEqCase{#1}{
		{0}{t}
		{1}{s}
		{2}{r}
		{3}{q}
		}
	[\PackageError{tel}{Undefined option to tel: #1}{}]
}

\newcommand{\cel}[1]{%
	\IfEqCase{#1}{
		{0}{m}
		{1}{n}
		{2}{r}
		{3}{k}
		}
	[\PackageError{cel}{Undefined option to cel: #1}{}]
}

\newcommand{\prel}[1]{%
	\IfEqCase{#1}{
		{0}{\theta}
		{1}{\psi}
		}
	[\PackageError{prel}{Undefined option to prel: #1}{}]
}
\newcommand{\rael}[1]{%
	\IfEqCase{#1}{
		{0}{\lambda}
		{1}{\mu}
		}
	[\PackageError{rael}{Undefined option to rael: #1}{}]
}

\newcommand{\shel}[1]{%
	\IfEqCase{#1}{
		{0}{\alpha}
		{1}{\beta}
		}
	[\PackageError{shel}{Undefined option to shel: #1}{}]
}

\newcommand{\rv}[1]{%
	\IfEqCase{#1}{
		{0}{X}
		{1}{Y}
		{2}{Z}
		{3}{W}
		{4}{V}
		{5}{U}
		}
	[\PackageError{rv}{Undefined option to rv: #1}{}]
}

\newcommand{\cdf}[1]{%
	\IfEqCase{#1}{
		{0}{F}
		{1}{G}
		}
	[\PackageError{cdf}{Undefined option to cdf: #1}{}]
}

\newcommand{\pmf}[1]{%
	\IfEqCase{#1}{
		{0}{p}
		{1}{q}
		}
	[\PackageError{pmf}{Undefined option to pmf: #1}{}]
}

\newcommand{\pdf}[1]{%
	\IfEqCase{#1}{
		{0}{f}
		{1}{g}
		}
	[\PackageError{pdf}{Undefined option to pdf: #1}{}]
}

\newcommand{\trv}[1]{%
	\IfEqCase{#1}{
		{0}{T}
		{1}{S}
		{2}{R}
		}
	[\PackageError{trv}{Undefined option to trv: #1}{}]
}

\newcommand{\crv}[1]{%
	\IfEqCase{#1}{
		{0}{M}
		{1}{N}
		}
	[\PackageError{crv}{Undefined option to crv: #1}{}]
}

\newcommand{\idx}[1]{%
	\IfEqCase{#1}{
		{0}{{i}}
		{1}{{j}}
		{2}{{k}}
		}
	[\PackageError{idx}{Undefined option to idx: #1}{}]
}

\newcommand{\iid}{\text{i.i.d.}}
\newcommand{\IID}{\text{I.I.D.}}

\newcommand{\as}{\text{a.s.}}

\newcommand{\setdelim}{\ \lvert \ }
\newcommand{\Bigsetdelim}{\ \Big\lvert \ }

\newcommand{\simp}{\textup{\texttt{simp}}}
\newcommand{\unitvec}{e}
\newcommand{\tmid}{t_0}
\newcommand{\eps}{\varepsilon}

\newcommand{\numexperts}{N} %
\newcommand{\numeffexperts}{N_{0}} %

\newcommand{\experts}{{\range{\numexperts}}} %
\newcommand{\effexperts}{\mathcal{I}_{0}} %
\newcommand{\neffexperts}{\experts\setminus\effexperts} %
\newcommand{\Deltaeff}{\Delta_0}

\newcommand{\expidx}{{i}} %
\newcommand{\Expidx}{{I}}
\newcommand{\Expidxdum}{{\Expidx'}}
\newcommand{\expidxdum}{{\expidx'}} %
\newcommand{\effexpidx}{{\expidx_{0}}} %

\newcommand{\intExpidx}{\hat\Expidx}

\newcommand{\expidxoptE}{\expidx^*} %
\newcommand{\expidxoptpath}{\Expidx^*} %

\newcommand{\intexpidxoptpath}{\intExpidx^*}

\newcommand{\expspace}{{\smash{\predspace}\vphantom{\dataspace}}^\numexperts}

\newcommand{\dataspace}{\Yy}
\newcommand{\predspace}{\hat\Yy}
\newcommand{\historyspace}{\Hh}

\newcommand{\data}{y}
\newcommand{\exppred}{x}
\newcommand{\pred}{{\hat \data}}
\newcommand{\history}{h}
\newcommand{\loss}{\ell}

\newcommand{\expidxpred}{\exppred_{\expidx}}

\newcommand{\kernelspace}{\Kk}
\newcommand{\meas}{\Mm}

\newcommand{\distn}{\mu}
\newcommand{\dumdistn}{\nu}
\newcommand{\refdistn}{\distn_{0}}

\newcommand{\distnball}[1]{\Dd_{#1}}

\newcommand{\distnballsuperset}[3]{\mathcal{V}(#1, #2, #3)}

\newcommand{\policyfamily}{\mathscr{P}}
\newcommand{\policyspace}{\mathscr{P}}
\newcommand{\predpolicyspace}{\hat{\mathscr{P}}}

\newcommand{\policy}{\pi}
\newcommand{\predpolicy}{\hat{\policy}}
\newcommand{\predpolicydum}{\hat{\policy}'}
\newcommand{\properpredpolicy}{\hat{\policy}^\star}

\newcommand{\predalg}{\mathfrak{a}}

\newcommand{\kernel}[1]{\pi_{#1}}
\newcommand{\predkernel}[1]{\hat{\pi}_{#1}}
\newcommand{\predkerneldum}[1]{\hat{\pi}'_{#1}}
\newcommand{\properpredkernel}[1]{\hat{\pi}^\star_{#1}}
\newcommand{\properpredkernelweight}[1]{\weightvec^{\star}_{#1}}

\newcommand{\constraintset}{constraint set\xspace}
\newcommand{\constraintsets}{constraint sets\xspace}

\newcommand{\constraintparams}{characterizing quantities}

\newcommand{\playername}{player}

\newcommand{\playerpolicyname}{prediction policy}

\newcommand{\playerpolicynames}{prediction policies}

\newcommand{\playeralgname}{prediction algorithm}

\newcommand{\playeralgnames}{prediction algorithms}

\newcommand{\properpredpolicyname}{proper prediction policy}

\newcommand{\properpredpolicynames}{proper prediction policies}

\newcommand{\envpolicyname}{data-generating mechanism}

\newcommand{\envpolicynames}{data-generating mechanisms}

\newcommand{\envdistname}{data-generating distribution}

\newcommand{\regretname}{regret}

\newcommand{\rateregretname}{rate of regret}

\newcommand{\expregretname}{expected regret}

\newcommand{\playerexpregretname}{quasi-regret}
\newcommand{\playerexpregretName}{Quasi-regret}

\newcommand{\intweightname}{intermediate weights}

\newcommand{\stochadvminimaxoptimal}{stochastic-and-adversarially minimax optimal}
\newcommand{\adaptminimaxoptimal}{adaptively minimax optimal}
\newcommand{\Adaptminimaxoptimal}{Adaptively minimax optimal}
\newcommand{\minimaxoptimal}{minimax optimal}
\newcommand{\minimaxoptimally}{minimax optimally}

\newcommand{\visibleparamname}{problem size}
\newcommand{\invisibleparamname}{problem hardness}

\newcommand{\semiadvshortname}{semi-adversarial}

\newcommand{\spectrumname}{spectrum}

\newcommand{\semiadvspectrum}{semi-adversarial spectrum}
\newcommand{\semiadvSpectrum}{Semi-adversarial spectrum}

\newcommand{\apriori}{\emph{a priori}}

\newcommand{\acronymstyle}[1]{{\small \textsc{#1}}}
\newcommand{\mathacronymstyle}[1]{{\scriptscriptstyle \textsc{#1}}}

\newcommand{\Hedge}{\acronymstyle{D.Hedge}}
\newcommand{\Hedgelong}{\acronymstyle{Decreasing Hedge}}
\newcommand{\OGHedge}{\acronymstyle{Hedge}}

\newcommand{\FTRL}{\acronymstyle{FTRL}}
\newcommand{\CARE}{\acronymstyle{CARE}}
\newcommand{\FTRLCARE}{\FTRL-\CARE}

\newcommand{\MetaCARE}{\acronymstyle{Meta-CARE}}
\newcommand{\prodalg}{\acronymstyle{prod}}
\newcommand{\adaptprodalg}{\acronymstyle{Adapt-ML-Prod}}
\newcommand{\AdaHedge}{\acronymstyle{AdaHedge}}

\newcommand{\FTRLeqn}{\text{\small \textsc{FTRL}}}
\newcommand{\FTRLalg}[3]{{\upshape \FTRL(#1, #2, #3)}}
\newcommand{\FTRLH}{{$\FTRLeqn_{\entropy}$}}
\newcommand{\FTRLHeqn}{{\text{\tiny \textsc{FTRL}}_{\entropy}}}
\newcommand{\FTRLHalg}[2]{{\upshape \FTRLH(#1, #2)}}

\newcommand{\TsallisInf}{{\small$\frac{1}{2}$}\acronymstyle{-Tsallis-INF}}

\newcommand{\shortHedge}{\mathacronymstyle{H}}
\newcommand{\shortCare}{\mathacronymstyle{C}}
\newcommand{\shortMeta}{\mathacronymstyle{M}}

\newcommand{\visibleparam}{\numexperts}
\newcommand{\visibleparamspace}{\Nats}
\newcommand{\invisibleparam}{(\numeffexperts,\Deltaeff)}
\newcommand{\invisibleparamspace}{[\numexperts] \times \PosReals}

\newcommand{\fullregret}{R}

\newcommand{\playerexpregret}{\hat{\fullregret}_{\predpolicy}}
\newcommand{\playerexpregretdum}{\hat{\fullregret}_{\predpolicy'}}
\newcommand{\playerexpregretHedge}{\hat{\fullregret}_{\shortHedge}}
\newcommand{\playerexpregretCare}{\hat{\fullregret}_{\shortCare}}
\newcommand{\playerexpregretMeta}{\hat{\fullregret}_{\shortMeta}}

\newcommand{\playerexpregretFTRLHeqn}{\hat{\fullregret}_{\FTRLHeqn}}

\newcommand{\vect}[1]{#1}
\newcommand{\lossvec}{\vect{\loss}}
\newcommand{\Loss}{{L}}
\newcommand{\Lossvec}{\vect{\Loss}}
\newcommand{\Lossdum}{\xi}
\newcommand{\Lossdumvec}{\vect{\Lossdum}}

\newcommand{\weight}{w}
\newcommand{\weightdum}{u}
\newcommand{\weightdumdum}{v}
\newcommand{\intweight}{v}
\newcommand{\weightvec}{\vect{\weight}}

\newcommand{\weightdumvec}{\vect{\weightdum}}
\newcommand{\weightdumdumvec}{\vect{\weightdumdum}}
\newcommand{\weightdumvecopt}{\weightdumvec_{*}}
\newcommand{\intweightvec}{\vect{\intweight}}

\newcommand{\hedgeweight}{w^{\shortHedge}}
\newcommand{\hedgeweightvec}{\vect{\hedgeweight}}

\newcommand{\careweight}{w^{\shortCare}}
\newcommand{\careweightvec}{\vect{\careweight}}

\newcommand{\metaweight}{w^{\shortMeta}}
\newcommand{\metaweightvec}{\vect{\metaweight}}

\newcommand{\EEboth}{\EE_{\policy,\predpolicy}}

\newcommand{\EEbothHedge}{\EE_{\policy,\shortHedge}}
\newcommand{\EEbothCare}{\EE_{\policy,\shortCare}}
\newcommand{\EEbothMeta}{\EE_{\policy,\shortMeta}}
\newcommand{\EEbothalg}{\EE_{\policy,\predalg}}

\newcommand{\Prboth}{\Pr_{\policy,\predpolicy}}

\newcommand{\regularizer}[1]{r_{#1}}
\newcommand{\cumregularizer}[1]{r_{0:#1}}

\newcommand{\timefunc}{\beta}
\newcommand{\entropyfunc}{\psi}

\newcommand{\entropy}{H}

\newcommand{\lrfunc}[1]{\tilde\eta_{#1}}
\newcommand{\lr}{\eta}
\newcommand{\ilr}{{\lr}}
\newcommand{\lblr}{{\overline\lr}}

\newcommand{\ilrint}{\vartheta}

\newcommand{\hedgelr}{g}
\newcommand{\lbhelper}{\theta}

\newcommand{\redentropy}{{\hat \entropy}}
\newcommand{\refexp}{{\expidx_1}}
\newcommand{\redexperts}{{\hat\experts}}
\newcommand{\redregularizer}[1]{{{\hat r}_{#1}}}
\newcommand{\redcumregularizer}[1]{{{\hat r}_{0:#1}}}

\newcommand{\constdum}{b}
\newcommand{\simpbijection}{\phi}
\newcommand{\simpequivalence}{\Phi}
\newcommand{\olodomain}{F}
\newcommand{\redolodomain}{\hat F}
\newcommand{\ololossdomain}{G}
\newcommand{\redololossdomain}{\hat G}
\newcommand{\olodim}{d}

\newcommand{\ololoss}{\lambda}
\newcommand{\oloLoss}{\Lambda}
\newcommand{\ololossvec}{\vect{\ololoss}}
\newcommand{\ololossvecgen}{\vect{y}}
\newcommand{\oloLossvec}{\vect{\oloLoss}}

\newcommand{\oloplayer}{\mu}
\newcommand{\oloplayerdum}{\nu}
\newcommand{\oloplayervec}{\vect{\oloplayer}}
\newcommand{\oloplayervecgen}{\vect{x}}
\newcommand{\oloplayerdumvec}{\vect{\oloplayerdum}}
\newcommand{\oloplayervecopt}{\vect{\oloplayer}_*}

\newcommand{\oloregularizer}[1]{\rho_{{#1}}}
\newcommand{\olocumregularizer}[1]{\rho_{0:{#1}}}
\newcommand{\oloplayerregret}{\fullregret_{\texttt{\upshape olo}}}
\newcommand{\oloplayerregretname}{regret}

\newcommand{\oloLossdumvec}{\zeta}

\newcommand{\careconst}{C}

\newcommand{\paramRecAllsimplified}{$c_\shortHedge=\sqrt{\log \numexperts}$ and $c_{\shortCare,1} = c_{\shortCare,2} = c_{\shortMeta} = 1$}

\newcommand{\EEmeas}[1]{{#1}}
\newcommand{\Prmeas}[1]{{#1}}

\newcommand\blfootnote[1]{%
  \begingroup
  \renewcommand\thefootnote{}\footnote{#1}%
  \addtocounter{footnote}{-1}%
  \endgroup
}
\title{Relaxing the \IID{} Assumption: Adaptively Minimax Optimal Regret via Root-Entropic Regularization}
\author{Blair Bilodeau$^{\ast,1}$, Jeffrey Negrea$^{\ast,1}$, and Daniel M.~Roy$^1$}
\date{}

\setlength{\parindent}{0pt}
\setlength{\parskip}{6pt}
\usepackage{titlesec}
\titlespacing{\section}{0pt}{\parskip}{0pt}
\titlespacing{\subsection}{0pt}{\parskip}{0pt}
\titlespacing{\subsubsection}{0pt}{\parskip}{0pt}
\usepackage[letterpaper, margin=1in]{geometry}
}
\aos{
\RequirePackage{amsthm,amsmath,amsfonts,amssymb}

\makeatletter
\newtheorem*{rep@theorem}{\rep@title}
\newcommand{\newreptheorem}[2]{%
\newenvironment{rep#1}[1]{%
 \def\rep@title{#2 \ref*{##1} (Quantitative Version)}%
 \begin{rep@theorem}}%
 {\end{rep@theorem}}}
\makeatother

\newreptheorem{theorem}{Theorem}
\newreptheorem{lemma}{Lemma}

\RequirePackage{graphicx}

}

\newcommand{\manualendproof}{\hfill\qedsymbol\\[2mm]}

\begin{document}

\preprint{
\maketitle
\blfootnote{${}^\ast$Blair Bilodeau and Jeffrey Negrea are equal-contribution authors; order was determined randomly.}
\blfootnote{${}^1$University of Toronto, Vector Institute}
\blfootnote{\, Correspondence: blair.bilodeau[at]mail.utoronto.ca}
}
\aos{\input{ims-frontmatter}}

\preprint{
\vspace{-30pt} %
\begin{abstract}

We consider prediction with expert advice when data are generated from distributions varying arbitrarily within an unknown constraint set. 
This \emph{\semiadvshortname{}} setting includes (at the extremes) the classical \iid{} setting, when the unknown constraint set is restricted to be a singleton, and the unconstrained adversarial setting, when the constraint set is the set of all distributions.
The Hedge algorithm---long known to be minimax (rate) optimal in the adversarial regime---was recently shown to be simultaneously minimax optimal for \iid{} data.
In this work,
we propose to relax the \iid{} assumption by seeking adaptivity at
 all levels
 of a natural ordering on constraint sets.
We provide matching upper and lower bounds on the minimax regret at all levels,
show that Hedge with deterministic learning rates is suboptimal outside of the extremes,
and prove that one can adaptively obtain minimax regret at all levels. %
We achieve this optimal adaptivity using
the follow-the-regularized-leader (FTRL) framework, 
with a novel adaptive regularization scheme
that implicitly scales as the square root of the entropy of the current predictive distribution, rather than the entropy of the initial predictive distribution. 
Finally, we provide novel technical tools to study the statistical performance of FTRL along the \semiadvspectrum{}.

\end{abstract}
}

\section{Introduction}
\label{sec:introduction}

In this work, we are concerned with obtaining guarantees on the quality of methods used to make decisions in light of data. Often, such guarantees are obtained via assumptions on the 
distribution of data. 
One important example of such an assumption is 
that data are independent and identically distributed (\iid{}). While this type of assumption on the joint dependence structure of data may be pragmatic, and can motivate methods that seem to perform well in practice, it is \emph{impossible} to be sure that 
apparent 
structure 
observed in past data will continue.
This
impossibility highlights the inherent limitations of such assumptions: any guarantees about performance may fail in practice if the assumed dependency does not hold,
and statistical methods that are optimal under a specific family of dependence structures may be far from optimal under another. 
It is of practical interest to determine when the performance of statistical methods is robust to the dependence structures that they are designed for, and to quantify how performance guarantees degrade as assumptions on the dependence structure are relaxed. 
Thus, contrary to guarantees that hold only under a specific dependence modelling assumption, guarantees should, ideally, hold regardless of the true nature of the data.

One way to formalize such guarantees is through the lens of
\emph{adaptation theory} (e.g., \citep{cai04adaptation}).
We do so by first introducing a new notion of regularity that quantifies the degree to which a sequence deviates from being i.i.d. 
The natural question that must be answered when one introduces a new notion of regularity is whether adaptivity is even possible; it may be the case that no single method obtains minimax optimal rates in every setting simultaneously.
Our main contribution in this work is answering this question in the affirmative for the specific type of regularity we introduce, demonstrating it is possible to optimally adapt to a specific relaxation of the \iid{} assumption.
In particular, we introduce the novel \emph{\semiadvspectrum{}}, which is an ordering of dependence structures characterized by their deviation from the \iid{} assumption, and quantify the performance of statistical methods at all levels of this \spectrumname{}.
This new notion of regularity can be applied to a wide range of decision tasks,
and can be combined with 
existing notions of regularity (such as smoothness).

Without the \iid{} assumption, future observations may depend on both past observations and predictions, 
and so we study performance in a \emph{sequential decision making} context. 
While the relaxation of the \iid{} assumption that we introduce is generically applicable to sequential decision making, in this work we consider specifically its application to
the problem of prediction with expert advice \citep{vovk98,littlestone94}.
Prediction with expert advice is a classical problem in statistics dating back to \citet{cover1965},
with close connections to empirical process theory \citep{cesa-bianchi1999prediction} and statistical aggregation \citep{tsybakov03aggregation,tsybakov04aggregation,audibert2009,rakhlin17}.
We show that several state-of-the-art methods for prediction with expert advice cannot be optimal at all
deviations from \iid{} without oracle knowledge of the deviation, but provide a novel algorithm that adaptively achieves 
minimax
optimal rates
at all deviations from \iid{} along the \semiadvspectrum{}.

Finally, we remark on the existing literature that 
studies benign \envpolicynames{} without relying on the \iid{} assumption (for a detailed survey, see \cref{sec:lit-easy}).
Many of these works obtain performance guarantees in terms of data-dependent (random) quantities; examples include error bounds that replace the dependence on the number of observations with the $\ell_\infty$ norm or empirical variance of the incurred losses.
In the present work, we take the perspective that performance guarantees should provide guidance on the quality of methods in advance of their use.
Data-dependent guarantees are not immediately satisfactory when viewed through this lens,
since one must still have a prior belief of which data is likely in order to evaluate the quality of the method in advance.
The choice of prior belief is important, since notions of data that make a data-dependent guarantee ``good'' (e.g., a small error bound) may not be compatible with which data is likely under a prior belief that the setting is ``easy''. 
As a concrete example, error bounds in terms of the empirical variance are large when the observed losses vary significantly, yet this may occur even when the data is truly \iid{}, a setting for which much smaller error bounds than those prescribed by the empirical variance bounds are possible. 

To address this discrepancy, we examine how the best possible performance degrades as the \envpolicyname{} varies between the \iid{} and adversarial cases. 
We explicitly incorporate the notion that the \iid{} case should be ``easiest'', and performance should degrade smoothly 
as we relax the \iid{} assumption towards the adversarial worst-case.
This perspective distinguishes our work from existing work: (a) we describe a formal spectrum of beliefs characterizing likely observations with \iid{} and adversarial data as its extremes, (b) we apply this spectrum to a novel data-dependent guarantee for a family of methods, identifying precisely which plausible \envpolicynames{} lead to better performance, and (c) we leverage this spectrum to understand performance when data is ``nearly \iid{}'', and how performance degrades as the \envpolicyname{} varies between \iid{} and adversarial.

\paragraph*{Contributions}
First, 
we formalize a relaxation of the \iid{} assumption for prediction with expert advice and a corresponding notion of adaptive minimax optimality, which requires identifying the optimal performance at each element of this \semiadvspectrum{}.
Then, our main contribution is to show it is possible to optimally adapt along the entire \semiadvspectrum{}, 
achieving minimax regret at each level of the \spectrumname{} without any advance knowledge of the \envpolicyname{}.
The \Hedgelong{} algorithm (\Hedge{}), which corresponds to prediction via a tempered Bayesian posterior for an expert-valued parameter, was recently shown to be simultaneously optimal for \iid{}\ and adversarial data \citep{mourtada2019optimality}.
However, we show that \Hedge{} (and its variants) requires oracle knowledge of the nature of the \envpolicyname{} to optimally tune its learning rate (a.k.a.\ the tempering parameter), and hence does not adapt 
along the \semiadvspectrum{} between these endpoints.
In light of this negative result, we introduce a novel algorithm \MetaCARE{}, which implicitly and adaptively adjusts the learning rate of \OGHedge{} without the need for oracle knowledge of the nature of \envpolicyname{}, and prove that it is \adaptminimaxoptimal{} along the entire \semiadvspectrum{}. 
\MetaCARE{} consists of
\emph{boosting}
 our novel \emph{follow-the-regularized-leader} (\FTRL{}) algorithm \FTRLCARE{} with \Hedge{} using a second application of \Hedge{}, and hence a major component of our analysis is devoted to a general study of \FTRL{} algorithms along the \semiadvspectrum{}.
A pivotal analytic tool that we develop for this analysis is a concentration of measure inequality under our relaxation of the \iid{} assumption,
which we expect to be useful beyond the present setting of 
prediction with expert advice.

\paragraph*{Organization}
In \cref{sec:framework} we formalize the problem setting of interest. 
In \cref{sec:new-setting}, we rigorously define the \semiadvspectrum{} and illustrate its relevance via several examples.
We present our notion of adaptive minimax optimality and summarize our main results on the minimax rates for the \semiadvspectrum{} in \cref{sec:results}. 
In \cref{sec:concentration} we provide our novel concentration of measure inequality for the \semiadvspectrum{}.
We precisely state the minimax lower bounds for performance along the \semiadvspectrum{} in \cref{sec:minimax-lower-bounds}, thus characterizing what an adaptively minimax optimal algorithm must achieve.
\cref{sec:decreasing-hedge} is devoted to quantitative upper and lower bounds for \Hedge{}, including our results on the non-adaptivity of \Hedge{}. \cref{sec:ftrl} introduces \FTRLCARE{} and provides a quantitative upper bound for its regret. An outline of the proofs of the regret upper bounds for \Hedge{} and \FTRLCARE{} is given in \cref{sec:main-quant-sketch}.
We introduce \MetaCARE{} in \cref{sec:metacare} along with the corresponding upper bound and proof, and then end with a review of the relevant literature in \cref{sec:literature}.
Technical details for the proofs of our results, and a brief simulation study, are deferred to the supplementary material.
\section{Notation and problem setup}
\label{sec:framework}
Prediction with expert advice is characterized by the manner in which
experts and the \playername{} make their predictions and the mechanism by which a response observation is generated.
At every time %
$t\in\Nats$,
each of the $\numexperts \in \Nats$ experts (arbitrarily indexed by $\experts = \set{1,\dots,\numexperts}$) formulate their predictions for the $t$th round, jointly denoted by $\exppred(t) \in \expspace$, the player makes a prediction for the $t$th round, $\pred(t) \in \predspace$, and the environment generates a response observation for the $t$th round, $\data(t)\in\dataspace$.
The \emph{history} of the game up to time $t$ is summarized by $\history(t)=(\exppred(s),\pred(s),\data(s))_{s\in\range{t}} \in \historyspace^t$, where $\historyspace = \expspace \times \predspace \times \dataspace$, with the convention that $h(0)$ is the empty tuple.
For each time $t\in\Nats$, the prediction $\pred(t)$ and response observation $\data(t)$ are conditionally independent given the history $\history(t-1)$ and the recent expert predictions, $\exppred(t)$. This conditional independence reflects the fact that the \playername{} does not have access to the response until after making their prediction, and that the \playername{} has some private source of stochasticity with which to randomize their predictions.

The conditional distribution of the experts' predictions and the data observed at round $t$ given $\history(t-1)$ is  uniquely described by a probability kernel $\kernel{t} \in \kernelspace(\historyspace^{t-1}, \expspace \times \dataspace)$, where $\kernelspace(\Aa,\Bb)$ denotes the set of probability kernels (regular conditional distributions) from $\Aa$ to $\Bb$.
Letting $\policyspace_{\numexperts} \defas\prod_{t\in\Nats} \kernelspace(\historyspace^{t-1}, \expspace \times \dataspace)$,
a \emph{\envpolicyname{}} is
any sequence $\policy = (\kernel{t})_{t\in\Nats}\in \policyspace_{\numexperts}$.
Similarly, the  conditional distribution of the \playername{}'s prediction at time $t$ given $\history(t-1)$ and $\exppred(t)$ is uniquely described by
a probability kernel $\predkernel{t}\in \kernelspace(\historyspace^{t-1} \times \expspace, \predspace)$.
Letting $\predpolicyspace{}_\numexperts \defas\prod_{t\in\Nats} \kernelspace(\historyspace^{t-1} \times \expspace, \predspace)$,
a \emph{\playerpolicyname{}} is
any sequence $\predpolicy = (\predkernel{t})_{t\in\Nats}\in \predpolicyspace{}_\numexperts$.
Finally,
a \emph{\playeralgname{}} is
any sequence $\predalg = (\predpolicy(\numexperts))_{\numexperts\in\Nats}$ with $\predpolicy(\numexperts) \in \predpolicyspace{}_\numexperts$ for each $\numexperts$.

In a sequential prediction task, prior to any data being generated or predictions being made, the \playername{} selects a \playeralgname{}
and the environment determines a \envpolicyname{}.
Without loss of generality, the \playername{} knows the number of experts $\numexperts$,
and so they predict according to the \playerpolicyname{} $\predpolicy = \predalg(\numexperts)$ based on their \playeralgname{}.
Due to the conditional independence assumption for $\pred(t)$ and $\data(t)$ given $\history(t-1)$ and $\exppred(t)$,
the joint distribution of $(\exppred(t),\pred(t),\data(t))_{t\in\Nats}$ is fully determined by the \envpolicyname{} and the \playerpolicyname{} selected by each party.
For a \envpolicyname{} $\policy$ and a \playerpolicyname{} $\predpolicy$, expectation under this joint law is denoted by $\EEboth$.
When the \playerpolicyname{} is determined by the \playeralgname{} $\predalg$, for any number of experts $\numexperts$ and \envpolicyname{} $\policy\in\policyspace_\numexperts$ we use $\EEbothalg$ to denote $\EE_{\policy,\predpolicy}$,
where $\predpolicy = \predalg(\numexperts)$.

The accuracy of the \playername{} and experts is measured on each round using a loss function $\loss: \predspace\times\dataspace\to[0,1]$,
and the \playername{}'s performance at the end of $T\in\Nats$ rounds of the game is measured by \emph{\regretname{}}, defined as the $\sigma(\history(T))$-measurable random variable
\*[
	\fullregret(T) \defas \sum_{t=1}^T \loss(\pred(t), \data(t)) - \min_{\expidx \in \experts}\sum_{t=1}^T \loss(\expidxpred(t),\data(t)).
\]
In this work, we focus on bounding the \expregretname{} $\EEboth \fullregret(T)$
for three specific prediction algorithms,
so we use $\EEbothHedge$, $\EEbothCare$, and $\EEbothMeta$ to denote $\EEbothalg$ under the \Hedge{}, \FTRLCARE{}, and \MetaCARE{} algorithms respectively (see \cref{sec:decreasing-hedge,sec:ftrl,sec:metacare} for the respective definitions of these \playeralgnames{}).

Since $\fullregret(T)$ only depends on $\history(T)$ through the loss function, \expregretname{} bounds are often characterized using quantities that push the \envdistname{}s forward through the loss function.
Specifically, we define the \emph{losses} $\loss_\expidx(t) = \loss(\expidxpred(t),\data(t))$ and \emph{cumulative losses} $\Loss_\expidx(t) = \sum_{s=1}^t \loss_\expidx(s)$ for each expert $\expidx \in \experts$ and $t\in\Nats$.
Similarly, we define the \emph{loss vector} $\lossvec(t) = \rbra{\loss_\expidx(t)}_{\expidx \in \experts}$
and \emph{cumulative loss vector} $\Lossvec(t) = \sum_{s=1}^t \lossvec(s)$.

Let $\meas(\Aa)$ denote the set of all probability distributions on $\Aa$. For a distribution $\distn \in \meas(\Aa)$ and measurable function $f: \Aa \to \Reals$, we define $\EEmeas{\distn} f = \int_{\Aa} f(a) \mu(\dee a)$ .
We will frequently use this notation for measures in $\meas(\expspace \times \dataspace)$. In particular, for each expert $\expidx\in\experts$, the expert's loss $\loss_\expidx:(\exppred,\data)\mapsto \loss(\exppred_\expidx,\data)$ is a function on $\expspace \times \dataspace$, and $\distn \loss_\expidx$ is the expectation of expert $\expidx$'s loss when the expert predictions and response observation are jointly distributed as $\distn$.

\section{\semiadvSpectrum{}}
\label{sec:new-setting}

Consider a fixed number of experts $\numexperts$.
For any \emph{time-homogeneous convex constraint} $\distnball{} \subseteq \meas(\expspace \times \dataspace)$,
let $\policyspace(\distnball{})$ denote the collection of \envpolicynames{} $\policy = (\kernel{t})_{t\in\Nats}$ such that for all $t\in\Nats$ and $\history \in \historyspace^{t-1}$, $\kernel{t}(\history,\cdot) \in \distnball{}$.
That is,
$\policyspace(\distnball{})$ is the set of \envpolicynames{} under which the conditional distribution of the expert predictions and response data given the history is 
constrained to $\distnball{}$, but can vary arbitrarily within $\distnball{}$ depending on the history.

\begin{figure}
    \centering
    \begin{tikzpicture}
        \newcommand{\leftanchorX}{-3} 
        \newcommand{\leftanchorY}{0} %
        
        \newcommand{\intraGroupOffsetX}{1.5} 
        \newcommand{\intraGroupOffsetY}{1.5} %
        
        \newcommand{\interGroupOffsetX}{4.0}
        \newcommand{\interGroupOffsetY}{-0.5} %
        
        \newcommand{\sidelength}{2} %

        \newcommand{\semiFillShape}[2]{ \draw[fill=lightgray, draw opacity = 0, rotate around={45:(#1+0.31*\sidelength,#2+0.22*\sidelength)}, thick] (#1+0.45*\sidelength,#2+0.2*\sidelength) ellipse (0.6cm and 0.3cm);
        }

        \newcommand{\triangleAt}[2]{\draw[black, thick] (#1,#2) -- (#1+0.5*\sidelength,#2+0.5*1.73*\sidelength) -- (#1+\sidelength,#2) -- cycle;} %
        \newcommand{\triangleFilledAt}[2]{\fill[lightgray, thick] (#1,#2) -- (#1+0.5*\sidelength,#2+0.5*1.73*\sidelength) -- (#1+\sidelength,#2) -- cycle;} %
        \newcommand{\triangleSemiFilledAt}[2]{ \semiFillShape{#1}{#2};} %

        \newcommand{\triangleGroupAt}[2]{
            \triangleAt{#1+0*\intraGroupOffsetX}{#2+0*\intraGroupOffsetY}
            \triangleAt{#1+1*\intraGroupOffsetX}{#2+1*\intraGroupOffsetY}
            \triangleAt{#1+2*\intraGroupOffsetX}{#2+2*\intraGroupOffsetY}
        }   %
        
        \newcommand{\IIDTriangleGroupAt}[2]{
            \triangleAt{#1+0*\intraGroupOffsetX}{#2+0*\intraGroupOffsetY}
            \triangleAt{#1+1*\intraGroupOffsetX}{#2+1*\intraGroupOffsetY}
            \triangleAt{#1+2*\intraGroupOffsetX}{#2+2*\intraGroupOffsetY}
            
            \filldraw [lightgray] (#1+0*\intraGroupOffsetX+0.5*\sidelength,#2+0*\intraGroupOffsetY+0.25*1.73*\sidelength) circle (3pt);
            \filldraw [lightgray] (#1+1*\intraGroupOffsetX+0.5*\sidelength,#2+1*\intraGroupOffsetY+0.25*1.73*\sidelength) circle (3pt);
            \filldraw [lightgray] (#1+2*\intraGroupOffsetX+0.5*\sidelength,#2+2*\intraGroupOffsetY+0.25*1.73*\sidelength) circle (3pt);
            
            \filldraw [black] (#1+0*\intraGroupOffsetX+0.5*\sidelength,#2+0*\intraGroupOffsetY+0.25*1.73*\sidelength) circle (1.5pt);
            \filldraw [black] (#1+1*\intraGroupOffsetX+0.5*\sidelength,#2+1*\intraGroupOffsetY+0.25*1.73*\sidelength) circle (1.5pt);
            \filldraw [black] (#1+2*\intraGroupOffsetX+0.5*\sidelength,#2+2*\intraGroupOffsetY+0.25*1.73*\sidelength) circle (1.5pt);
            
            \draw[->,black, thick] (#1+0*\intraGroupOffsetX+0.5*\sidelength,#2+0*\intraGroupOffsetY+0.25*1.73*\sidelength) -> 
            (#1+2.5*\intraGroupOffsetX+0.5*\sidelength,#2+2.5*\intraGroupOffsetY+0.25*1.73*\sidelength);
            \node[align=center] at (#1+0.5*\sidelength,#2-0.25*\sidelength) {\IID{}};
            
            \node[align=left] at (#1-0.1*\sidelength+0*\intraGroupOffsetX,#2+0.5*\sidelength+0*\intraGroupOffsetY) {$t=1$};
            \node[align=left] at (#1-0.1*\sidelength+1*\intraGroupOffsetX,#2+0.5*\sidelength+1*\intraGroupOffsetY) {$t=2$};
            \node[align=left] at (#1-0.1*\sidelength+2*\intraGroupOffsetX,#2+0.5*\sidelength+2*\intraGroupOffsetY) {$t=3$};
            \node[align=left] at (#1+2.7*\intraGroupOffsetX+0.5*\sidelength,#2+2.7*\intraGroupOffsetY+0.25*1.73*\sidelength) {\dots};
        } %
        
        \newcommand{\adversarialTriangleGroupAt}[2]{
            \triangleFilledAt{#1+0*\intraGroupOffsetX}{#2+0*\intraGroupOffsetY}
            \triangleAt{#1+0*\intraGroupOffsetX}{#2+0*\intraGroupOffsetY}
            
            \triangleFilledAt{#1+1*\intraGroupOffsetX}{#2+1*\intraGroupOffsetY}
            \triangleAt{#1+1*\intraGroupOffsetX}{#2+1*\intraGroupOffsetY}
            
            \triangleFilledAt{#1+2*\intraGroupOffsetX}{#2+2*\intraGroupOffsetY}
            \triangleAt{#1+2*\intraGroupOffsetX}{#2+2*\intraGroupOffsetY}
            
            \filldraw [black] (#1+0*\intraGroupOffsetX,#2) circle (1.5pt);
            \filldraw [black] (#1+1*\intraGroupOffsetX+1.0*\sidelength,#2+1*\intraGroupOffsetY) circle (1.5pt);
            \filldraw [black] (#1+2*\intraGroupOffsetX+0.5*\sidelength,#2+2*\intraGroupOffsetY+0.5*1.73*\sidelength) circle (1.5pt);
            
            \draw[-,black, thick] (#1+0*\intraGroupOffsetX,#2) -> (#1+1*\intraGroupOffsetX+1.0*\sidelength,#2+1*\intraGroupOffsetY);
            \draw[-,black, thick] (#1+1*\intraGroupOffsetX+1.0*\sidelength,#2+1*\intraGroupOffsetY) -> 
                (#1+2*\intraGroupOffsetX+0.5*\sidelength,#2+2*\intraGroupOffsetY+0.5*1.73*\sidelength);
            \draw[->,black, thick] (#1+2*\intraGroupOffsetX+0.5*\sidelength,#2+2*\intraGroupOffsetY+0.5*1.73*\sidelength) -> 
                (#1+2.5*\intraGroupOffsetX+0.5*\sidelength,#2+2.5*\intraGroupOffsetY+0.25*1.73*\sidelength);
                
            \node[align=center] at (#1+0.5*\sidelength,#2-0.25*\sidelength) {Adversarial};
            \node[align=left] at (#1+2.7*\intraGroupOffsetX+0.5*\sidelength,#2+2.7*\intraGroupOffsetY+0.25*1.73*\sidelength) {\dots};
        } %

        \newcommand{\semiadversarialTriangleGroupAt}[2]{
            \triangleSemiFilledAt{#1+0*\intraGroupOffsetX}{#2+0*\intraGroupOffsetY}
            \triangleAt{#1+0*\intraGroupOffsetX}{#2+0*\intraGroupOffsetY}

            \triangleSemiFilledAt{#1+1*\intraGroupOffsetX}{#2+1*\intraGroupOffsetY}
            \triangleAt{#1+1*\intraGroupOffsetX}{#2+1*\intraGroupOffsetY}
            
            \triangleSemiFilledAt{#1+2*\intraGroupOffsetX}{#2+2*\intraGroupOffsetY}
            \triangleAt{#1+2*\intraGroupOffsetX}{#2+2*\intraGroupOffsetY}
            
            \filldraw [black] (#1+0.34*\sidelength,#2+0.4*\sidelength) circle (1.5pt);
            \filldraw [black] (#1+1*\intraGroupOffsetX+0.5*\sidelength,#2+1*\intraGroupOffsetY+0.19*\sidelength) circle (1.5pt);
            \filldraw [black] (#1+2*\intraGroupOffsetX+0.56*\sidelength,#2+2*\intraGroupOffsetY+0.51*\sidelength) circle (1.5pt);
            
            \draw[-,black, thick] (#1+0.34*\sidelength,#2+0.4*\sidelength) -> (#1+1*\intraGroupOffsetX+0.5*\sidelength,#2+1*\intraGroupOffsetY+0.19*\sidelength);
            \draw[-,black, thick] (#1+1*\intraGroupOffsetX+0.5*\sidelength,#2+1*\intraGroupOffsetY+0.19*\sidelength) -> 
                (#1+2*\intraGroupOffsetX+0.56*\sidelength,#2+2*\intraGroupOffsetY+0.51*\sidelength);
            \draw[->,black, thick] (#1+2*\intraGroupOffsetX+0.56*\sidelength,#2+2*\intraGroupOffsetY+0.51*\sidelength) -> 
                (#1+2.5*\intraGroupOffsetX+0.5*\sidelength,#2+2.5*\intraGroupOffsetY+0.25*1.73*\sidelength);
                
            \node[align=center] at (#1+0.5*\sidelength,#2-0.25*\sidelength) {Semi-Adversarial};
            \node[align=left] at (#1+2.7*\intraGroupOffsetX+0.5*\sidelength,#2+2.7*\intraGroupOffsetY+0.25*1.73*\sidelength) {\dots};
        } %
        
        \newcommand{\triangleGroupsAt}[2]{
            \IIDTriangleGroupAt{#1+0*\interGroupOffsetX}{#2+0*\interGroupOffsetY}
            \adversarialTriangleGroupAt{#1+1*\interGroupOffsetX}{#2+1*\interGroupOffsetY}
            \semiadversarialTriangleGroupAt{#1+2*\interGroupOffsetX}{#2+2*\interGroupOffsetY}
        }
        
        \triangleGroupsAt{\leftanchorX}{\leftanchorY}

    \end{tikzpicture}
    \caption{Visualising the difference between \iid{} data, adversarial data, and a constraint set in between these two extremes. In each part of the figure, the triangles depict the set of conditional distributions for the tuple of expert predictions and response (an ``instance'') given the history at each time. The grey regions depict the space of conditional distributions for the next instance given the history that are possible for a given constraint set.}
    \label{fig:semi-adv-spectrum}
\end{figure}

In \cref{fig:semi-adv-spectrum}, we visualize possible trajectories of \envpolicynames{} for the \iid{} endpoint, adversarial endpoint, and a constraint set that lies between these.
In the \iid{} case, the conditional distribution of the next instance given the history is fixed, and hence the constraint set corresponds to a single distribution on instances. In the adversarial case, the conditional distribution of the next instance given the history can vary arbitrarily in the space of all probability distributions on instances; in particular, it can be a point-mass at an adversarial instance for the \playername{}'s strategy, depicted here as the extreme points of the space of distributions. Since the \iid{} case corresponds to a singleton set of distributions on instances, and the adversarial case corresponds to the whole space of distributions on instances, a natural concept of ``in between'' these extremes is a proper subset of the set of distributions on instances. Our relaxation captures this by allowing the conditional distribution of the next instance given the history to vary within some convex constraint set that is not known by the \playername{} in advance (visualized here as an ellipse), and measuring performance relative to the properties of that unknown constraint set. 

We use two \emph{\constraintparams{}} to describe $\distnball{}$. For each expert $\expidx\in\experts$, let
\*[
	\Delta_\expidx(\distnball{}) = \inf_{\distn\in\distnball{}} \, \max_{\expidxdum\in\experts} \, \distn [\loss_\expidx - \loss_{\expidxdum}],
\]
and define the \emph{effective stochastic gap}
\*[
	\Deltaeff(\distnball{}) = \min\{\Delta_\expidx(\distnball{}) \setdelim \expidx\in\experts, \Delta_\expidx(\distnball{})>0\}.
\]
Second, define the set of \emph{effective experts}
\*[
	\effexperts(\distnball{}) = \{\expidx\in\experts \setdelim \Delta_\expidx(\distnball{}) = 0\}.
\]
$\effexperts(\distnball{})$ contains the experts that could be the best (in conditional expectation given the history) on any particular round.
The size of the effective expert set is denoted by $\numeffexperts(\distnball{}) = \card{\effexperts(\distnball{})}$.
$\Deltaeff(\distnball{})$ is the minimal excess expected loss of an \emph{ineffective expert} over the best effective expert on any round.
When $\distnball{}$ is clear, we simplify notation to $\effexperts$, $\numeffexperts$, and $\Deltaeff$.

For a fixed $\numexperts$, $\numeffexperts$, and $\Deltaeff$, the collection of convex constraint sets that have these \constraintparams{} is
\*[
	\distnballsuperset{\numexperts}{\numeffexperts}{\Deltaeff}
		& = \set{\distnball{}\subseteq \meas(\expspace\times\dataspace) \setdelim \distnball{} \text{ convex}, \, \numeffexperts(\distnball{}) = \numeffexperts, \, \Deltaeff(\distnball{}) \geq \Deltaeff},
\]
and the corresponding set of \envpolicynames{} is
\*[
	\policyfamily_{\numexperts, (\numeffexperts, \Deltaeff)} = \bigcup_{\distnball{}\in\distnballsuperset{\numexperts}{\numeffexperts}{\Deltaeff}} \policyspace(\distnball{}).
\]
Let 
$
	\policyfamily = \{\policyfamily_{\numexperts,(\numeffexperts,\Deltaeff)} \setdelim \numeffexperts \leq \numexperts \in \Nats, \Deltaeff > 0\}
$
denote the collection of all such sets.
Together, $\numeffexperts$ and $\Deltaeff$ induce a total ordering on constraint sets, and the \emph{\semiadvspectrum} is the collection of equivalence classes this ordering induces.

\subsection{Motivation for \constraintparams{}}
The \constraintparams{} $\numeffexperts$ and $\Deltaeff$ reduce to the standard \constraintparams{} for the \rateregretname{} from the \iid{} setting.
To see this, observe that the \iid{} setting corresponds to $\distnball{}$ defined by a single distribution $\refdistn \in \meas(\expspace \times \dataspace)$; that is, for all $t\in\Nats$ and $\history \in \historyspace^{t-1}$, the \envpolicyname{} satisfies $\kernel{t}(\history,\cdot) = \refdistn$.
It is well known that the \minimaxoptimal{} \expregretname{} under the \iid{} assumption depends on the \emph{stochastic gap}.
Letting
 $\effexperts(\refdistn) = \argmin_{\expidx\in\experts} \EEmeas{\refdistn} \loss_\expidx$
 be the set of experts that are optimal (w.r.t.\ $\loss$) in expectation under $\refdistn$, each expert $\expidx \in \experts$ has stochastic gap $\Delta_\expidx(\refdistn) = \EEmeas{\refdistn} \loss_{\expidx} - \min_{\effexpidx\in\experts} \EEmeas{\refdistn} \loss_{\effexpidx}$, and
the stochastic gap is defined by $\Deltaeff(\refdistn) = \min_{\expidx \in \neffexperts(\refdistn)} \Delta_\expidx(\refdistn)$.
The \minimaxoptimal{} \expregretname{} in the \emph{stochastic-with-a-gap} setting (\iid{} with $\card{\effexperts(\refdistn)}=1$) satisfies (cf.\ \citep{mourtada2019optimality})
\*[
	\EE \, \fullregret(T) \in \Theta\left(\frac{\log\numexperts}{\Deltaeff(\refdistn)} \right).
\]
The effective experts and the effective stochastic gap generalize $\effexperts(\refdistn)$ and $\Deltaeff(\refdistn)$ beyond the \iid{} case, and our \expregretname{} bounds depend on these \constraintparams{} 
in a similar way to the dependence on $\numexperts$ and $\Deltaeff$ in the stochastic and adversarial settings respectively.

\subsection{Practical relevance of convex constraints}\label{sec:practical}

A standard application of prediction with expert advice is to the setting of statistical aggregation (cf.\ \citep{nemirovski00,yang2004,audibert2009}). We now describe an example of an aggregation task where the time-homogeneous convex constraint setting is the canonical representation of the \envpolicyname{}.
Suppose the statistician has $\numexperts$ models that map from a covariate space $\Xx$ to a response space $\Yy$. 
Further, suppose that the $t$th observation $(X_t,Y_t)$ is sampled from one of $K$ unknown distributions on $\Xx\times\Yy$, where this distribution is selected in a potentially adversarial and non-\iid{} way using the previous $t-1$ observations. 
That is, the observed dataset is an adversarial mixture of $K$ different stochastic sources. 
The ability of the \envpolicyname{} to randomize its selection of the source distribution gives rise to a time-homogeneous convex constraint, where $\distnball{}$ is the convex hull of the $K$ source distributions.
If the source distributions and models are reasonably distinct, this will likely satisfy $\numeffexperts = K$, which may be much smaller than $\numexperts$.

\subsection{Examples of convex constraints}\label{sec:examples}

The following examples 
illustrate the flexibility of time-homogeneous convex constraints and the \semiadvspectrum{}.

\begin{example}[\IID{}-$\refdistn$, Stochastic-with-a-gap]
	When the constraint set is the singleton $\distnball{\refdistn} = \set{\refdistn}$, then there is only one possible \envpolicyname{}, and under that \envpolicyname{} the data and expert predictions are \iid{} according to $\refdistn$. Furthermore, if there exists $\effexpidx\in\experts$ and $\Delta>0$ such that
	\*[
		\inf_{\expidx\in\experts\setminus\set{\effexpidx}}\ \EEmeas{\distn}\sbra{ \loss_\expidx - \loss_{\effexpidx}} = \Delta,
	\]
	(i.e., there is a best expert in expectation under $\refdistn$ and there is a gap of $\Delta$ from the best to the second best expert in expectation)
	then $\effexperts(\distnball{\refdistn})=\set{\effexpidx}$, $\numeffexperts(\distnball{\refdistn}) = 1$, and $\Deltaeff(\distnball{\refdistn})=\Delta$. This is called the \emph{stochastic-with-a-gap} setting. Since any singleton is convex, $\distnball{\refdistn}$ is convex.
\end{example}

\begin{example}[Adversarial]
	When the constraint set is the space of all probability measures $\distnball{\mathrm{adv}} = \meas(\expspace\times\dataspace)$, then the constrained setting reduces to the fully adversarial setting, since $\distnball{}$ contains all point-mass distributions. In this case, $\effexperts(\distnball{\mathrm{adv}}) = \experts$, $\numeffexperts(\distnball{\mathrm{adv}})=\numexperts$, and $\Deltaeff(\distnball{\mathrm{adv}}) = +\infty$ (by convention, as it is the $\inf$ over an empty set). Since the set of all probability measures is convex, $\distnball{\mathrm{adv}}$ is convex.
\end{example}

\begin{example}[Adversarial-with-an-instantaneous-gap]\label{ex:adv-gap}
	For any $\effexpidx\in\experts$ and $\Delta\geq 0$,
	\*[
		\distnball{\effexpidx,\Delta}\upper{\mathrm{a.s.}}
		 \defas
		 \set[3]{
			\distn\in\meas(\expspace\times\dataspace)
			\Bigsetdelim
			\Prmeas{\distn}\Big(\loss_\effexpidx +\Delta\leq \min_{\expidx\in\experts\setminus\set{\effexpidx}}\loss_\expidx\Big)=1
			}
	\]
	is convex (since $\min$ is concave), and satisfies $\effexperts(\distnball{\effexpidx,\Delta}\upper{\mathrm{a.s.}}) = \set{\effexpidx}$, $\numeffexperts(\distnball{\effexpidx,\Delta}\upper{\mathrm{a.s.}}) = 1$, and $\Deltaeff(\distnball{\effexpidx,\Delta}\upper{\mathrm{a.s.}}) = \Delta$. This contains all mixtures of point-mass distributions with common best expert $\effexpidx$ that satisfy the gap constraint almost surely.
\end{example}

\begin{example}[Adversarial-with-an-$\EE$-gap, \aos{\citet{mourtada2019optimality}}\preprint{\citet{mourtada2019optimality}}]
	For any $\effexpidx\in\experts$ and $\Delta\geq 0$,
	\*[
		\distnball{\effexpidx,\Delta}
		 \defas
		 \set[3]{
			\distn\in\meas(\expspace\times\dataspace)
			\Bigsetdelim
			\EEmeas{\distn} \loss_\effexpidx +\Delta\leq \min_{\expidx\in\experts\setminus\set{\effexpidx}} \EEmeas{\distn} \loss_\expidx
			}
	\]
	is convex (since $\min$ is concave), and satisfies $\effexperts(\distnball{\effexpidx,\Delta}) = \set{\effexpidx}$, $\numeffexperts(\distnball{\effexpidx,\Delta})=1$ and $\Deltaeff(\distnball{\effexpidx,\Delta}) = \Delta$. This relaxes the adversarial-with-an-instantaneous-gap setting, since $\distnball{\effexpidx,\Delta}\upper{\mathrm{a.s.}}\subseteq\distnball{\effexpidx,\Delta}$. This constraint set is equivalent to the formulation used in Corollary~6 of \citet{mourtada2019optimality}; it is also the same setting as Section~4.2 of \citet{wei18adaptive}, although they consider bandit feedback.
\label{ex:adv-E-gap}
\end{example}

\begin{example}[Ball-around-\IID{}]\label{ex:ball}
	For any pseudometric $d$, radius $r>0$, and probability measure $\mu_0$, 
	\*[
		\distnball{\mu_0,d,r} = B_d(\mu_0,r) = \set{\mu\in\meas(\expspace\times\dataspace) \setdelim d(\mu,\mu_0)\leq r}
	\]
	is convex. The exact values of $\effexperts(\distnball{\mu_0,d,r})$, $\numeffexperts(\distnball{\mu_0,d,r})$, and $\Deltaeff(\distnball{\mu_0,d,r})$ will depend on $\mu_0$, $r$, and $d$.
	In general, $\effexperts$ and $\numeffexperts$ are increasing with $r$ (w.r.t.\ $\subseteq$ and $\leq$ respectively), while $\Deltaeff$ will decrease as $r$ increases between the jumps in $\numeffexperts$, but increase sharply at the jumps.
	Thus, the lexicographical ordering on $(\numeffexperts,\Deltaeff^{-1})$ coincides with increasing the radius, $r$.
	Since for nested constraint sets it should be more difficult to compete with the larger of the two constraints, it is intuitive that the lexicographical order on $(\numeffexperts,\Deltaeff^{-1})$ is an assessment of the difficulty of competing with a given constraint set.

\end{example}

\begin{example}[Convex hull of basic distributions]\label{ex:convex-hull}
As motivated in \cref{sec:practical}, a natural setting is where $\distnball{}$ is the convex hull of some basic underlying distributions. Suppose
$\numexperts=3$, 
and there exist $\distn,\dumdistn\in\meas(\expspace\times\dataspace)$ satisfying $\distn \lossvec = (0,1,0.5+\eps)$ and $\dumdistn \lossvec = (1,0,0.5+\eps)$, where $\eps>0$ is arbitrary. 
Set $\distnball{} = \{\alpha \distn + (1-\alpha) \dumdistn \setdelim \alpha \in [0,1] \}$, which gives $\effexperts(\distnball{})= \{1,2\}$ and $\Deltaeff(\distnball{}) = \eps$.

However, on any given round it is possible for the data to be sampled from either $\distn$ or $\dumdistn$, in which case one of the effective experts is as separated (in expectation) as possible from the best expert and separated by an arbitrarily large multiplicative factor of $\Deltaeff$ from the ineffective expert. That is, this example demonstrates effective experts need not be better or even close to ineffective experts on any given round.
\end{example}

Note that \cref{ex:adv-gap} is related to the setting of \citet{seldin2014one} and \cref{ex:ball} is related to the setting of \citet{lykouris18corruption} (both focusing on bandit feedback), with the distinction that the existing literature considers constraints on the \emph{cumulative} losses.
In contrast, our constraints apply to the distributions allowed on any instantaneous round, and are not restricted in how they accumulate.
This distinction is subtle, yet crucial to the type of adaptivity we propose in this work. 
While existing ``easy data'' results are about adapting to post-hoc summary statistics of the data, we provide adaptivity to the unknown, underlying dependence structure, and propose that statistical methods should be designed to adapt to this as well (beyond adaptivity to model regularity assumptions).

\section{Adaptive optimality for the \semiadvspectrum{}}\label{sec:results}

In this section we will state our main results that characterize the minimax regret over time-homogeneous convex constraints. We begin by precisely defining what it means for a \playeralgname{} to be \emph{\adaptminimaxoptimal{}}. 

\subsection{\Adaptminimaxoptimal{} \playeralgnames{}}
\label{sec:notation-adapt}
Informally, an \adaptminimaxoptimal{} \playeralgname{}
achieves the \minimaxoptimal{} regret (asymptotically in $T$) for the \constraintparams{} constraining the allowable \envpolicyname{} without \apriori{} information on what values these \constraintparams{} take. 
For collections of sequences $a = \{(a_{\visibleparam,\invisibleparam}(T))_{T\in\Nats} \setdelim \visibleparam\in\visibleparamspace, \invisibleparam\in\invisibleparamspace\}$ and $b = \{(b_{\visibleparam,\invisibleparam}(T))_{T\in\Nats} \setdelim \visibleparam\in\visibleparamspace, \invisibleparam\in\invisibleparamspace\}$, 
we write 
\*[
	a_{\visibleparam,\invisibleparam}(T) \lesssim b_{\visibleparam,\invisibleparam}(T) &&(\textrm{abbreviated $a\lesssim b$})
\] 
when 
\[\label{eqn:adaptive-defn}
	&\exists C>0 \quad
	\forall \visibleparam\in\visibleparamspace,\ 
	\invisibleparam\in\invisibleparamspace
	\quad \exists T_0 \in\Nats
	\quad \forall T>T_0
	\\
	&\hspace{2em} a_{\visibleparam,\invisibleparam}(T)
	\leq
	C \, b_{\visibleparam,\invisibleparam}(T).
	\qquad
\]
If $a\lesssim b$ and $b \lesssim a$, we write $a_{\visibleparam,\invisibleparam}(T) \asymp b_{\visibleparam,\invisibleparam}(T)$ (abbreviated $a\asymp b$).
For a \playeralgname{} $\predalg = (\predalg(\visibleparam))_{\visibleparam\in\visibleparamspace}$,
we refer to the equivalence class under $\asymp$ of 
\*[
	\visibleparam,\invisibleparam,T \mapsto \sup_{\smash{\policy\in\policyfamily_{\visibleparam,\invisibleparam}}}
	\EEbothalg \fullregret(T)
\] 
as the \emph{\rateregretname{}} or simply the \emph{rate} of $\predalg$, and the equivalence class under $\asymp$ of 
\*[
	\visibleparam,\invisibleparam,T \mapsto \inf_{\smash{\predpolicy \in \predpolicyspace_\visibleparam}}
	\sup_{\smash{\policy\in\policyfamily_{\visibleparam,\invisibleparam}}}
	\EEboth \fullregret(T)
\] 
as the \emph{\minimaxoptimal{} \rateregretname{}}.
Then, we say a \playeralgname{} $\predalg$
is \emph{\adaptminimaxoptimal{}} if
\[\label{eqn:adaptive-alg-defn}
	\sup_{\smash{\policy\in\policyfamily_{\visibleparam,\invisibleparam}}}
	\EE_{\policy,\predalg} \fullregret(T)
	\asymp
	\inf_{\smash{\predpolicy \in \predpolicyspace_\visibleparam}}
	\sup_{\smash{\policy\in\policyfamily_{\visibleparam,\invisibleparam}}}
	\EEboth \fullregret(T).
\]
Further, we say that $\predalg$ is \emph{adaptive} if $\sup_{\smash{\policy\in\policyfamily_{\visibleparam}}}
\EEbothalg \fullregret(T)$ is always sublinear in $T$ and, for some $\invisibleparam$, its \rateregretname{} is strictly better than the rate of $\inf_{\smash{\predpolicy \in \predpolicyspace_\visibleparam}}
\sup_{\smash{\policy\in\policyfamily_{\visibleparam}}}
\EEboth \fullregret(T)$; otherwise, we say $\predalg$ is \emph{non-adaptive}.
This definition formalizes the notion that an adaptive \playeralgname{} must realize potential benefits from at least some instance of ``easier''  \constraintparams{} and simultaneously have average regret at least converge to zero in all cases.

Importantly, we do not demand that the \playeralgname{} perform as well as if they had \apriori{} knowledge of the true \envpolicyname{}, since with this information the minimax regret can be quite small (zero or even negative). Instead, the \playeralgname{} is only adapting to the \emph{\invisibleparamname{}}, as measured by the \constraintparams{}, and consequently there is still freedom in the minimax definition for the \playername{} to face its worst-case \envpolicyname{} subject to these \constraintparams{}. Mathematically, this is ensured by placing $\inf_{\smash{\predpolicy \in \predpolicyspace_\visibleparam}}$ after the choice of \constraintparams{}, but before the choice of \envpolicyname{} (i.e., $\sup_{\smash{\policy\in\policyfamily_{\visibleparam,\invisibleparam}}}$).

More abstractly, our definition of \adaptminimaxoptimal{} can be interpreted under a generic adaptive decision problem, with a generic \emph{\visibleparamname{}} given by $\numexperts$ and a generic \emph{\invisibleparamname{}} replacing \constraintparams{}. 
For example, in the case of density estimation, the \visibleparamname{} may correspond to the dimension of the data space, which the statistician knows, and the \invisibleparamname{} may correspond to the H\"{o}lder continuity parameter of the true data-generating density, which the statistician does not know. For a further discussion of our definition of \adaptminimaxoptimal{}, see \cref{sec:adaptive-discussion}.

\subsection{Minimax rates}
\label{sec:minimax-rates}

We are now able to state our main result, establishing the \minimaxoptimal{} \rateregretname{} and that it is achieved by our novel algorithm \MetaCARE{}, which follows from the 
conjunction of \cref{thm:oracle-lb,prop:mg-lb,thm:metacare}.

\begin{theorem}[Main result]
\*[
	\sup_{\smash{\policy\in\policyfamily_{\numexperts,(\numeffexperts,\Deltaeff)}}}
	\EEbothMeta \, \fullregret(T)
	\asymp
	\inf_{\smash{\predpolicy \in \predpolicyspace_\numexperts}}
	\sup_{\smash{\policy\in\policyfamily_{\numexperts,(\numeffexperts,\Deltaeff)}}}
	\EEboth \fullregret(T) 
	&\asymp
	\sqrt{T\log\numeffexperts} + \frac{\log\numexperts}{\Deltaeff}.
\]
\end{theorem}

In \cref{thm:hedge-lb}, we show that \Hedge{} using
any parametrization that simultaneously achieves the \minimaxoptimal{} \rateregretname{} in both the stochastic-with-a-gap and adversarial
settings is non-adaptive.
That is, 
for $\numeffexperts \geq 2$,
\*[
	\sup_{\smash{\policy\in\policyfamily_{\numexperts,(\numeffexperts,\Deltaeff)}}}
	\EEbothHedge \, \fullregret(T) \gtrsim  \sqrt{T \log\numexperts }.
\]
In fact, from \cref{thm:hedge-lb,thm:hedge-bound-quant}, we find that without an \emph{oracle parametrization} of \Hedge{} (one where $\numeffexperts$ is made available to the \playername{} in advance), it is only possible to achieve
\*[
\log(\numeffexperts)\sqrt{T} + \frac{(\log\numexperts)}{\Deltaeff} \lesssim \sup_{\smash{\policy\in\policyfamily_{\numexperts,(\numeffexperts,\Deltaeff)}}}\EEbothHedge \fullregret(T)
	& \lesssim
	\log(\numeffexperts)\sqrt{T} + \frac{(\log\numexperts)^2}{\Deltaeff}
\]
or
\*[
\sup_{\smash{\policy\in\policyfamily_{\numexperts,(\numeffexperts,\Deltaeff)}}}
\EEbothHedge \fullregret(T)
	& \asymp
	\ind{\numeffexperts \geq 2}\sqrt{T\log\numexperts}+ \frac{\log\numexperts}{\Deltaeff},
\]
but not both.

As an intermediary step, we introduce another novel algorithm, \FTRLCARE{}, and show in
\cref{thm:ftrl-care-bound-quant}
that it adapts with a better rate:
\*[
	\sup_{\smash{\policy\in\policyfamily_{\numexperts,(\numeffexperts,\Deltaeff)}}}
	\EEbothCare \, \fullregret(T)
	\lesssim
	\sqrt{T \log\numeffexperts} + \frac{(\log\numexperts)^{3/2}}{\Deltaeff}.
\]
To also achieve the \minimaxoptimal{} rate for $\numeffexperts=1$ (and consequently be \adaptminimaxoptimal{}), we introduce \MetaCARE{} in \cref{thm:metacare}, which corresponds to another application of \Hedge{} to the ``meta-experts'' corresponding to \FTRLCARE{} and \Hedge{} on all $\numexperts$ experts.

Our quantitative upper bounds also explicitly demonstrate how large $T$ must be for algorithms to have adaptive rates (i.e., \expregretname{} that depends on $\numeffexperts$ and $\Deltaeff$), as opposed to the pessimistic adversarial rate (i.e., $\sqrt{T\log\numexperts}$). In particular, for both \Hedge{} and \FTRLCARE{}, roughly $\Deltaeff^{-2}$ rounds of adversarial \regretname{} are incurred before the level of adaptation is sufficient to reduce the \rateregretname{} accumulation. This demonstrates that as $\Deltaeff$ tends to $0$, the \playername{} does not incur infinite \regretname{} from the $\Deltaeff^{-1}$ terms, but rather incurs adversarial \regretname{} for a longer amount of time. We emphasize that the \playername{} does not need to know when they will stop incurring adversarial \regretname{} ahead of time to parametrize either algorithm, so knowledge of $\numeffexperts$ or $\Deltaeff$ is not required.

Our theoretical results are further supported by a simulation study that appears in \cref{sec:simulations}.  The simulation study is based on the \envpolicynames{} that achieve the lower bound in the stochastic-with-a-gap setting and the algorithm specific lower bound for \Hedge{} with two effective experts. The results of the simulations agree with our theoretical results.

\subsection{Discussion on adaptive minimax optimality}\label{sec:adaptive-discussion}
One might ask whether it's possible to strengthen the notion of adaptivity to be \emph{uniform-in-$T$}, where the rate has to be achieved up to a constant at all $T$, rather than only for sufficiently large $T$ depending on $\invisibleparam$.
This corresponds to replacing the relation $a\precsim b$ with the one defined by
\*[
	\exists C>0 \quad
	\forall T,\ 
	\visibleparam\in\visibleparamspace,\ 
	\invisibleparam\in\invisibleparamspace
	\qquad
	a_{\visibleparam,\invisibleparam}(T)
	\leq
	C \, b_{\visibleparam,\invisibleparam}(T).
	\qquad
\]

In the context of minimax regret, uniform adaptivity would require understanding the entire path of the regret (over $T$) rather than simply its eventual upper bound. 
This is not understood even in the stochastic setting; regret bounds of the form $1/\Delta$ in both the bandit and full-information settings \citep[e.g.,][]{auer2002stochastic,gaillard14,mourtada2019optimality} are all eventual upper bounds that are only known to be tight (i.e., have matching lower bounds) for sufficiently large $T$. 
Since it remains open to identify the minimax optimal regret uniformly in $T$ even for this basic setting, we do not attempt to also solve this in our more general setting beyond \iid{} data.

Beyond prediction with expert advice, the lack of uniform adaptivity also persists.
For example, the leading constant of the minimax rates for smoothness-adaptation in statistics  often depends on the smoothness parameter, which violates uniformity.
For general questions of adaptive minimax optimality in sequential prediction, it is not clear how to demonstrate that either form of adaptivity is possible other than by constructing adaptive algorithms, as we have done in the present work.

Finally, one could consider adapting to a different collection of \constraintparams{} than $\invisibleparam$. For our setting, a natural extension is to consider the individual expectation gaps of each expert, rather than only the smallest gap. While our upper bounds can be extended to handle multiple gaps without much difficulty, tight lower bounds that depend on all the gaps simultaneously are again unknown even in the \iid{} setting for 
full-information feedback. Since our work is about identifying minimax optimality, which would require such lower bounds, we do not consider this refinement.
Beyond the extension to multiple gaps, it is an interesting avenue for future work to identify other \constraintparams{} that could provide a finer characterization of the \envpolicyname{}.

\section{Concentration of measure for the \semiadvspectrum{}}
\label{sec:concentration}

In this section, we state and prove a concentration of measure result for  
\envpolicynames{} permitted by time-homogeneous convex constraints,
which we use repeatedly to establish upper bounds on \expregretname{} for \Hedge{}, \FTRLCARE{}, and \MetaCARE{}.
The result demonstrates that,
even though
the best expert may vary from round to round, the gap between the best effective expert along the observed data path and any ineffective expert grows like a sum of uniformly sub-Gaussian random variables with mean below $-\Deltaeff$.

\begin{theorem}\label{thm:minimax_mgf}\label{THM:MINIMAX_MGF}
For all $\numexperts\geq 2$, 
\playerpolicynames{} $\predpolicy\in\predpolicyspace_\numexperts$,
convex sets $\distnball{} \subseteq \meas(\expspace \times \dataspace)$,
$\lambda > 0$,
$T_0 < T_1$,
and $\expidx\in\neffexperts$,
\*[
  \sup_{\policy\in\policyspace(\distnball{})} \,
  \EEboth \,
  \min_{\effexpidx\in\effexperts} \,
  \exp\left\{\lambda \sum_{t=T_0+1}^{T_1} \sbra{\loss_{\effexpidx}(t) - \loss_{\expidx}(t)} \right\}
  \leq \exp\left\{(T_1-T_0)\sbra{\lambda^2/2 - \lambda \Deltaeff}\right\}.
\]
\end{theorem}

Note that we require the \constraintset $\distnball{}$ to be convex.
If $\distnball{}$ is not natively convex, our results clearly apply to its convex hull.
There is, however, a natural reason to consider convex \constraintsets:
given a set $\distnball{}$ of joint distributions available for the \envpolicynames{}, requiring the set to be convex is equivalent to also allowing mixtures of the original available distributions. That is, the environment and experts together can randomly select a distribution from $\distnball{}$ to generate data from at each round.

One may wonder whether this result follows from an application of the Azuma--Hoeffding inequality.
However, as demonstrated in \cref{ex:convex-hull}, there exist simple \constraintsets such that on any round, any effective expert (including the best overall) may have an arbitrarily larger expected loss than any ineffective expert.
That is, $\Loss_{\expidx}(t) - \Loss_{\effexpidx}(t)$ need not be a (sub)martingale, and consequently Azuma--Hoeffding does not directly apply.
Instead, the proof of this result first uses a variant of von Neumann's minimax theorem---which is the technical reason why we require the \constraintset $\distnball{}$ to be convex---before applying Hoeffding's inequality to the instantaneous rounds.
We restate the minimax theorem we require for completeness here.
\begin{proposition}[\citet{plg07}, Theorem~7.1]\label{prop:vn-minimax}
	Let $\Xx$ and $\Yy$ be convex subsets of linear topological spaces, and suppose that $\Xx$ is compact. Let $f:\Xx\times\Yy\to \Reals$ be such that:
	\begin{enumerate}[\indent(i)]
		\item for all $y\in\Yy$, $f(\cdot,y):\Xx\to\Reals$ is convex and continuous; and
		\item for all $x\in\Xx$, $f(x,\cdot):\Yy\to\Reals$ is concave.
	\end{enumerate}
	Then,
	\*[
		\inf_{x\in\Xx}\sup_{y\in\Yy} f(x,y) = \sup_{y\in\Yy}\inf_{x\in\Xx} f(x,y).
	\]
\end{proposition}

\begin{proof}[Proof of {\cref{thm:minimax_mgf}}]
Let $\PosReals = [0,\infty)$ and $\simp(\effexperts) = \set[1]{\weightdumdumvec\in\PosReals^{\effexperts} : \sum_{\effexpidx\in\effexperts} \weightdumdum_{\effexpidx} = 1}$. First, since at least one optimal solution to a linear program on a compact convex polytope must be at a vertex,
\*[
	\min_{\effexpidx\in\effexperts}
		\sum_{t={T_0+1}}^{T_1} \Big[\loss_{\effexpidx}(t) - \loss_{\expidx}(t)\Big]
	= \inf_{\weightdumdumvec\in\simp(\effexperts)}
		\sum_{t={T_0+1}}^{T_1} \Big[\inner{\weightdumdumvec}{\lossvec_{\effexperts}(t)} - \loss_{\expidx}(t)\Big].
\]
Further, since $\exp$ is a monotone function, this identity implies
\*[
	\min_{\effexpidx\in\effexperts} e^{\lambda \sum_{t={T_0+1}}^{T_1} \big[\loss_{\effexpidx}(t) - \loss_{\expidx}(t)\big]}
	& = \inf_{\weightdumdumvec\in\simp(\effexperts)} e^{\lambda \sum_{t={T_0+1}}^{T_1} \big[\inner{\weightdumdumvec}{\lossvec_{\effexperts}(t)} - \loss_{\expidx}(t)\big] }.
	\label{eqn:minimax_mgf_1}
\]

Then, applying Jensen's and the max--min inequality gives
\*[
	&\hspace{-2em}
	\sup_{\policy\in\policyspace(\distnball{})} \,
	\EEboth \,
	\inf_{\weightdumdumvec\in\simp(\effexperts)} \,
	e^{\lambda \sum_{t={T_0+1}}^{T_1} \big[\inner{\weightdumdumvec}{\lossvec_{\effexperts}(t)} - \loss_{\expidx}(t)\big] } \\
	&\leq
	\inf_{\weightdumdumvec\in\simp(\effexperts)} \,
	\sup_{\policy\in\policyspace(\distnball{})} \,
	\EEboth \,
	e^{\lambda \sum_{t={T_0+1}}^{T_1} \big[\inner{\weightdumdumvec}{\lossvec_{\effexperts}(t)} - \loss_{\expidx}(t)\big] }.
\]

By the tower rule for conditional expectation and the definition of the kernel $\kernel{T_1}$,
\*[
	&\hspace{-2em}
	\EEboth \,
	e^{\lambda \sum_{t={T_0+1}}^{T_1} \big[\inner{\weightdumdumvec}{\lossvec_{\effexperts}(t)} - \loss_{\expidx}(t)\big] }\\
	&\leq
	\bigg(
	\EEboth \,
	\Big[
	e^{\lambda \sum_{t={T_0+1}}^{T_1-1} \big[\inner{\weightdumdumvec}{\lossvec_{\effexperts}(t)} - \loss_{\expidx}(t)\big] }\Big]\bigg)
	\bigg(\sup_{\distn\in\distnball{}} \,
		\EEmeas{\distn}\rbra{
		e^{\lambda \sbra{\inner{\weightdumdumvec}{\lossvec_{\effexperts}} - \loss_{\expidx}} }}\bigg).
\]

Iterating this argument $T_1-T_0-1$ more times, and using monotonicity of power functions, gives
\*[
	&\hspace{-2em}\inf_{\weightdumdumvec\in\simp(\effexperts)} \,
	\sup_{\policy\in\policyspace(\distnball{})} \,
	\EEboth \,
	e^{\lambda \sum_{t={T_0+1}}^{T_1} \big[\inner{\weightdumdumvec}{\lossvec_{\effexperts}(t)} - \loss_{\expidx}(t)\big] } \\
	&\leq \sbra{\inf_{\weightdumdumvec\in\simp(\effexperts)} \,
	\sup_{\distn\in\distnball{}} \,
	\EEmeas{\distn}\rbra{
	e^{\lambda \sbra{\inner{\weightdumdumvec}{\lossvec_{\effexperts}} - \loss_{\expidx}} }}}^{T_1-T_0}.
\]
Noting that $\simp(\effexperts)$ is convex, that $\distnball{}$ is convex,
and that the objective function $f(\weightdumdumvec,\distn) = \EEmeas{\distn}\rbra[1]{e^{\lambda \sbra{\inner{\weightdumdumvec}{\lossvec_{\effexperts}} - \loss_{\expidx}} }}$
is continuous and convex in $\weightdumdumvec$ and linear (and hence concave) in $\distn$, \cref{prop:vn-minimax} gives
\*[
\inf_{\weightdumdumvec\in\simp(\effexperts)} \sup_{\distn \in \distnball{}} f(\weightdumdumvec,\distn)
&= \sup_{\distn \in \distnball{}} \inf_{\weightdumdumvec\in\simp(\effexperts)} f(\weightdumdumvec,\distn).
\label{eqn:minimax_mgf_2}
\]
Thus,
\*[
	&\hspace{-2em}\sup_{\policy\in\policyspace(\distnball{})} \,
	\EEboth \,
	\inf_{\weightdumdumvec\in\simp(\effexperts)} \,
	e^{\lambda \sum_{t={T_0+1}}^{T_1} \big[\inner{\weightdumdumvec}{\lossvec_{\effexperts}(t)} - \loss_{\expidx}(t)\big] } \\
	&\leq \sbra{\sup_{\distn\in\distnball{}} \,
	\inf_{\weightdumdumvec\in\simp(\effexperts)} \,
	\EEmeas{\distn}\rbra{
	e^{\lambda \sbra{\inner{\weightdumdumvec}{\lossvec_{\effexperts}} - \loss_{\expidx}} }}}^{T_1-T_0}.
\label{eqn:minimax_mgf_3}
\]

Consider any $\distn \in \distnball{}$, and let $\expidxoptE(\distn) \in \argmin_{\expidx\in\experts} \EEmeas{\distn} \loss_\expidx$. By the definition of $\Deltaeff$, $\EEmeas{\distn}\rbra{\loss_{\expidxoptE(\distn)} - \loss_{\expidx}} \leq - \Deltaeff$ for every $\expidx\in\neffexperts$.
Finally, since $\lossvec\in[0,1]^\numexperts$ $\distn$-a.s., by Hoeffding's lemma,
\*[
	\inf_{\weightdumdumvec\in\simp(\effexperts)} \,
	\EEmeas{\distn} \rbra{
	e^{\lambda \sbra{\inner{\weightdumdumvec}{\lossvec_{\effexperts}} - \loss_{\expidx}} }}
	&\leq
	\EEmeas{\distn} \rbra{
	e^{\lambda \sbra{\loss_{\expidxoptE(\distn)} - \loss_{\expidx}}}}
	\leq e^{\lambda^2/2 - \lambda \Deltaeff}.
\]
Since this holds for all $\distn\in\distnball{}$,
\*[
	\sup_{\policy\in\policyspace(\distnball{})} \,
	\EEboth \,
	\inf_{\weightdumdumvec\in\simp(\effexperts)} \,
	e^{\lambda \sum_{t={T_0+1}}^{T_1} \big[\inner{\weightdumdumvec}{\lossvec_{\effexperts}(t)} - \loss_{\expidx}(t)\big] }
	& \leq \Big[e^{\lambda^2/2 - \lambda \Deltaeff} \Big]^{T_1-T_0}
	= e^{(T_1-T_0)[\lambda^2/2 - \lambda \Deltaeff]}.
\]
\end{proof}
\section{Minimax lower bounds}
\label{sec:minimax-lower-bounds}

In this section, we characterize the best possible performance under relaxations of the \iid{} assumption.
In particular, we quantify the best any \playerpolicyname{}
can do with oracle knowledge of the number of effective experts.
The proof of this result is found in \cref{ssec:oracle-lb-proof}.
While we do not expect a \playername{} to be able to know the nature of the constraint set,
we use this oracle lower bound to conclude that since our novel algorithm \MetaCARE{} achieves the same performance without using oracle knowledge, it is \adaptminimaxoptimal{}.

\begin{theorem}\label{thm:oracle-lb}
There exist $\predspace$, $\dataspace$, and $\loss$ such that, 
for all $\numeffexperts\in\Nats$, 
there exists $\tmid\in\Nats$ such that 
for all $\numexperts\in\Nats$ with $\numexperts\geq \numeffexperts$ and $T \geq \tmid$,
\*[
	\sup_{\smash{\distnball{}\in\distnballsuperset{\numexperts}{\numeffexperts}{1/2}}}\
	\sup_{\smash{\policy\in\policyspace(\distnball{})}}\
	\inf_{\smash{\predpolicy \in \predpolicyspace_\numexperts}}\
	\EEboth \, \fullregret(T)
	\geq \frac{\sqrt{T \log \numeffexperts}}{10}.
\]
\end{theorem}

\cref{thm:oracle-lb} allows us to characterize the \minimaxoptimal{} dependence on $T$ and $\numeffexperts$ of a \playerpolicyname{}.
However, for the case of $\numeffexperts=1$, the leading term instead depends on $\Deltaeff$. Consequently, to determine the \minimaxoptimal{} \rateregretname{} at all relaxations of the \iid{} assumption, we must also use the the following result by \citet{mourtada2019optimality}, which establishes a lower bound for
when there is only one effective expert.%

\begin{proposition}[\aos{\citet{mourtada2019optimality}}\preprint{\citet{mourtada2019optimality}}, Proposition~4]
\label{prop:mg-lb}
For all $\numexperts\in\Nats$,
there exist $\predspace$, $\dataspace$, and $\loss$ such that 
for all $\Delta\in(0,1/4)$ and $T\geq \frac{\log \numexperts}{16\Delta^2}$,
\*[
	\inf_{\smash{\predpolicy \in \predpolicyspace_\numexperts}}\
	\sup_{\smash{\distnball{}\in\distnballsuperset{\numexperts}{1}{\Delta}}}\
	\sup_{\smash{\policy\in\policyspace(\distnball{})}}\
	\EEboth \fullregret(T)
	\geq \frac{\log\numexperts}{256\Delta}.
\]
\end{proposition}

These two lower bounds set the bar for what one should hope to achieve. In order to adapt to an unknown number of effective experts $\numeffexperts\leq\numexperts$ and identity of the effective experts, the \playername{} can be forced to incur $\max(\sqrt{T\log\numeffexperts}, \Deltaeff^{-1}\log\numexperts)$ \rateregretname{}.
Because $\max\{\sqrt{T\log\numeffexperts}, \Deltaeff^{-1}\log\numexperts\} \asymp \sqrt{T\log\numeffexperts} + \Deltaeff^{-1}\log\numexperts$, a \playeralgname{} with a \rateregretname{}  $\lesssim \sqrt{T\log\numeffexperts} + \Deltaeff^{-1}\log\numexperts$ is \adaptminimaxoptimal{}.
\section{Performance of {\Hedge{}}}
\label{sec:decreasing-hedge}

In this section, we show that without oracle knowledge of the \constraintparams{}, \Hedge{} can be parametrized to either (a) be \minimaxoptimal{} for the special case when $\numeffexperts\in\set{1,\numexperts}$, but incur adversarial \regretname{} in between, or (b) adapt suboptimally to every value of the \constraintparams{}.
Following this section, we introduce \FTRLCARE{} and prove it adapts \minimaxoptimally{} when there are multiple effective experts.
We then \emph{boost} these two algorithms together in \MetaCARE{}, and prove this is \adaptminimaxoptimal.

All of these \playeralgnames{} produce \emph{\properpredpolicynames{}}%
, which means that rather than picking $\pred$ from the entirety of $\predspace$, at each round the \playername{} chooses one of the experts $\expidx \in [\numexperts]$ to emulate and predicts $\pred(t) = \exppred_{\expidx}(t)$. To choose the expert to emulate, the history is used to choose a distribution on $\experts$, and then $\expidx$ is sampled from this distribution.

Formally, for $\exppred \in \expspace$ and $\weightvec\in\simp(\experts)$, let $\smash{\pushfwdmeas{\exppred}{\weightvec} = \sum_{\expidx\in\experts} \weightvec_\expidx \delta_{\exppred_\expidx} \in \meas(\predspace)}$ be the pushforward of $\weightvec\in\simp(\experts)$ 
through $\exppred$, viewing the vector $\exppred$ as a function $\smash{\exppred: \experts \to \predspace}$ and identifying $\simp(\experts)$ with $\meas(\experts)$.
A \properpredpolicyname{} $\properpredpolicy = (\properpredkernel{t})_{t\in\Nats}$ is any \playerpolicyname{} such that, for all $t \in \Nats$, there exists a measurable map $\properpredkernelweight{t}: \historyspace^{t-1} \to \simp(\experts)$ satisfying, 
for all $\history\in\historyspace^{t-1}$ and $\exppred \in \expspace$, $\properpredkernel{t}((\history, \exppred), \cdot) = \pushfwdmeas{\exppred}{[\properpredkernelweight{t}(\history)]}$. The $\sigma(\history(t-1))$-measurable random variable $\weightvec(t) = \properpredkernelweight{t}(\history(t-1))$ is called the \emph{weight vector}, 
or simply the \emph{weights}. 
For each $\expidx\in\experts$, $\weightvec_{\expidx}(t)$ corresponds to the probability that the \playername{} will emulate the $\expidx$th expert's prediction at time $t$.

The \playeralgname{} \OGHedge{} is parametrized by a sequence of measurable functions $(\lrfunc{t})_{t\in\Nats} \in \prod_{t\in\Nats} \{\historyspace^{t-1} \to \PosReals \}$. The $\sigma(\history(t-1))$-measurable random variable $\lr(t) = \lrfunc{t}(\history(t-1))$ is called the \emph{learning rate}, and the weights are defined by
\*[
	\hedgeweight_{\expidx}(t) = \frac{\exp\left\{-\lr(t) \Loss_\expidx(t-1) \right\}}{\sum_{\expidxdum\in\experts}\exp\left\{-\lr(t) \Loss_{\expidxdum}(t-1) \right\}}, \quad \expidx\in\experts.
\]
The \playeralgname{} \Hedgelong{} (\Hedge{}) is parametrized by a function $\hedgelr: \Nats \to \PosReals$, and corresponds to \OGHedge{} with the deterministic learning rate $\lr(t) = \hedgelr(\numexperts)/\sqrt{t}$ for all $t\in\Nats$.

It is well-known (see, for example, Theorem~2.3 of \aos{\citep{plg07}}\preprint{\citet{plg07}}) that \Hedge{} with $\hedgelr(\numexperts) \propto \sqrt{\log\numexperts}$ is \minimaxoptimal{} in the adversarial setting, which corresponds to $\distnball{} = \Mm(\expspace\times\dataspace)$.
Recently, \citet{mourtada2019optimality} showed that \Hedge{} with this parametrization is also \minimaxoptimal{} in the \iid{} setting, which corresponds to $\abs{\distnball{}} = 1$.
One might hope that this \stochadvminimaxoptimal{} parametrization would also perform well for all convex $\distnball{}$ in between these two cases.
However, part (i) of \cref{thm:hedge-lb} shows that, in fact, this parametrization fails to adapt to the number of effective experts when $\numeffexperts\not\in\set{1,\numexperts}$.
Further, we show that a different parametrization can adapt in some ways, but does not achieve the \minimaxoptimal{} dependence
on $T$.
\subsection{Algorithm-specific lower bounds for \Hedge{}}
\label{ssec:hedge-lower-bounds}
First, we observe that \Hedge{} with $\hedgelr(\numexperts)\propto\sqrt{\log\numexperts}$, which is \minimaxoptimal{} for both the stochastic and adversarial cases, does not adapt to an intermediate number of effective experts.
Additionally, \Hedge{} with constant $\hedgelr$ can do better than the \stochadvminimaxoptimal{} parametrization, but still cannot do as well as the oracle knowledge dependence on $T$ given in \cref{thm:oracle-lb}.
We prove this result in \cref{ssec:hedge-lb-proof}.

\begin{theorem}
\label{thm:hedge-lb} 
\begin{enumerate}[(i)]
	\item
	For all $c>0$,
	\*[
		\numexperts
			& \geq \exp\left\{\left(\frac{72\log2}{c^2} + 9\right)e^{c^2/4} \right\}, &\andT &&
		2 \leq \numeffexperts
			&\leq e^{ - c^2/8} \numexperts^{c^2\exp(c^2/4) / 72} -1,
	\]
	there exist $\predspace$, $\dataspace$, and $\loss$
	such that for all
	$T\geq 16 c^{-2}{ \log\numexperts}$,
	\Hedge{} with $\hedgelr(\numexperts) = c\sqrt{\log\numexperts}$ satisfies
	\*[
		\sup_{\distnball{} \in \distnballsuperset{\numexperts}{\numeffexperts}{1/2}}\
		\sup_{\policy \in \policyspace(\distnball{})}\
		\EEbothHedge \, \fullregret(T)
		\geq \frac{c\sqrt{T\log\numexperts}}{72\exp\left\{c^2/4 \right\}}
			- \frac{1}{3c^2}
			- \frac{\log\numexperts}{3}.
	\]
	\item
	Suppose the \playername{} is allowed oracle knowledge of $\numeffexperts$ in addition to $\numexperts$, and consequently can parametrize \Hedge{} by any $\hedgelr:\Nats^2 \to \PosReals$.
	For all $81 < \numeffexperts \leq \numexperts$
	there exist  $\predspace$, $\dataspace$, and $\loss$ such that
	\Hedge{} with $\hedgelr(\numexperts,\numeffexperts) \leq 2\sqrt{\log\numeffexperts - 4\log3}$ satisfies that for all
	$T\geq 32 [\hedgelr(\numexperts,\numeffexperts)]^{-2} \log\numexperts$,
	\*[
		\sup_{\distnball{} \in \distnballsuperset{\numexperts}{\numeffexperts}{1/2}}\
		\sup_{\policy \in \policyspace(\distnball{})}\
		\EEbothHedge \, \fullregret(T)
		\geq \frac{\log(\numeffexperts)\sqrt{T}}{4\hedgelr(\numexperts,\numeffexperts)} - \frac{3\log\numeffexperts}{[\hedgelr(\numexperts,\numeffexperts)]^2}.
	\]
\end{enumerate}
\end{theorem}

The proof of \cref{thm:hedge-lb} can be used to argue that other ``adaptive'' variants of \OGHedge{} will also fail to be \adaptminimaxoptimal{} along the \semiadvspectrum{}. We highlight this argument using well-known \OGHedge{}-variants from the literature. This is not meant to disparage these works, as they should not be expected to design algorithms for a notion of optimality defined years later, but to exemplify that adapting along the \semiadvspectrum{} is non-trivial and that the objectives of earlier works are insufficient to capture the notion of optimality we introduce.

The algorithm \prodalg{} of \citet{cesabianchi2007secondorder} is essentially \Hedge{} with an adaptive learning rate shared by all experts. 
This adaptive learning rate is comprised of the reciprocal-square-root of the cumulative squared losses, which will be (essentially) a constant multiple of $t$ under the \envpolicyname{} described in the proof of \cref{thm:hedge-lb}. Thus, the learning rate will behave the same as the data-independent learning rate of \Hedge{}, and consequently a similar lower bound on performance applies. A similar argument would also hold for \AdaHedge{} \citep{derooij14FTL}.

The refined algorithm \adaptprodalg{} of \citet{gaillard14} is more subtle, since it has a different learning rate for each expert. 
However, the recommended learning rate (Corollary~4 of their paper) would not achieve this since it uses $\log\numexperts$ for all experts, as opposed to an adaptive quantity as in \FTRLCARE{}. Consequently, for large enough $\numeffexperts$ and $t$, the \envpolicyname{} of \cref{thm:hedge-lb} will make the average loss with respect to the \adaptprodalg{} weights roughly $1/2$, and thus 
\adaptprodalg{}
inherits the same order of lower bound as \Hedge{}.
\subsection{Upper bounds for {\Hedge{}}}
\label{ssec:hedge-upper}

Now, we show that the lower bound of \cref{thm:hedge-lb} is tight. For a \playerpolicyname{} $\predpolicydum$ that may be distinct from the actual \playerpolicyname{} $\predpolicy$ the \playername{} is using,
we define the \emph{\playerexpregretname{}} (with respect to $\predpolicydum$) at time $T$ by
\*[
  \playerexpregretdum(T)
  &= \sum_{t=1}^T \int \loss(\pred(t), \data(t))
  \predkerneldum{t}\Big((h(t-1),\exppred(t)), \dee \pred(t) \Big)
  - \min_{\expidx\in\experts}\sum_{t=1}^T \loss(\expidxpred(t),\data(t)).
\]
\playerexpregretName{} replaces the actual loss at each round $t$ with the conditional expectation of the \playername{}'s loss \emph{had that \playername{} played according to $\predkerneldum{t}$ on round $t$}; the histories correspond, however, to the \emph{actual} predictions made by $\predpolicy$.
This allows us to quantify the performance of $\predpolicydum$ even when the entire sequence of predictions is governed by $\predpolicy$.

Clearly, $\EEboth \, \playerexpregret(T) = \EEboth \, \fullregret(T)$. However, we can prove almost sure results about $\playerexpregretdum(T)$ for some \playerpolicyname{} $\predpolicydum$, and then state expectation results of the form $\EEboth \playerexpregretdum(T)$, where the expectation is with respect to a possibly different \playerpolicyname{} $\predpolicy$. Results of this nature are crucial in the proof of \cref{thm:metacare},
where we use them to control the \regretname{} accumulated by \Hedge{} and \FTRLCARE{} when the actual \playerpolicyname{} is \MetaCARE{}.

\begin{theorem}
\label{thm:hedge-bound-quant}\label{THM:HEDGE-BOUND-QUANT}
For all $\hedgelr:\Nats\to\PosReals$ used to parametrize \Hedge{}, all $\numexperts\geq 2$, \playerpolicynames{} $\predpolicy\in\predpolicyspace_\numexperts$, convex $\distnball{}\subseteq \meas(\expspace\times\dataspace)$, and $T\in\Nats$,
\*[
  \sup_{\policy\in\policyspace(\distnball{})} \EEboth \playerexpregretHedge(T)
  \leq \sqrt{T + 1} \, \rbra{\frac{\log\numexperts}{\hedgelr(\numexperts)} +
  \hedgelr(\numexperts)}.
\]
Moreover, when $T > \ceil{\frac{8 (\log\numexperts+[\hedgelr(\numexperts)]^2/4+\hedgelr(\numexperts))^2}{[\hedgelr(\numexperts)]^2\Deltaeff^2}}$
the following two cases hold:\\
If $\numeffexperts > 1$, then
\*[
  \sup_{\policy\in\policyspace(\distnball{})} \EEboth \playerexpregretHedge(T)
  &\leq
  \frac{17}{16}\sqrt{T}\rbra{\frac{\log\numeffexperts}{\hedgelr(\numexperts)} + \hedgelr(\numexperts)}  \, + \frac{32 }{\Deltaeff} \rbra{\frac{\log\numexperts}{\hedgelr(\numexperts)}} \rbra{\frac{\log\numexperts}{\hedgelr(\numexperts)} +\hedgelr(\numexperts)} \\
  &\qquad + \sqrt{2}\rbra{\frac{\log\numexperts}{\hedgelr(\numexperts)} + \hedgelr(\numexperts)},
\]
and if $\numeffexperts = 1$, then
\*[
  \sup_{\policy\in\policyspace(\distnball{})} \EEboth \playerexpregretHedge(T)
  &\leq
  \frac{5}{\Deltaeff} \sbra{\rbra{\frac{\log\numexperts}{\hedgelr(\numexperts)}} \rbra{\frac{\log\numexperts}{\hedgelr(\numexperts)} + \hedgelr(\numexperts)} + 4 \rbra{\frac{1}{\hedgelr(\numexperts)^{2}}+\hedgelr(\numexperts)^2} } \\
  &\qquad + \sqrt{2}\rbra{\frac{\log\numexperts}{\hedgelr(\numexperts)} + \hedgelr(\numexperts)}.
\]
\end{theorem}

In order to more easily interpret this result, we also state the \expregretname{} of \Hedge{} for various natural choices of $\hedgelr$.

\begin{remark}\label{rem:hedge-asymp}\ \\
Taking $\predpolicy$ to be determined by \Hedge{},
\begin{enumerate}[(i)]
\item if $\hedgelr(\numexperts)$ is constant,
\*[
	\sup_{\smash{\policy\in\policyfamily_{\numexperts,(\numeffexperts,\Deltaeff)}}}
	\EEbothHedge \, \fullregret(T)
	\lesssim
	\,
	\log(\numeffexperts)
	\sqrt{T}
	+ \frac{(\log\numexperts)^2}{\Deltaeff};
\]
\item if
$g(\numexperts) \propto \sqrt{\log\numexperts}$,
\*[
	\sup_{\smash{\policy\in\policyfamily_{\numexperts,(\numeffexperts,\Deltaeff)}}}
	\EEbothHedge \, \fullregret(T)
	\lesssim
		\, 
		\ind{\numeffexperts\geq 2} \sqrt{T\log\numexperts}
		+ \frac{\log\numexperts}{\Deltaeff};
\]
\item 
in the oracle setting for $\numeffexperts \geq 2$, if $\hedgelr(\numexperts,\numeffexperts)\propto\sqrt{\log\numeffexperts}$,
\*[
	\hspace{-1em}
	\sup_{\smash{\policy\in\policyfamily_{\numexperts,(\numeffexperts,\Deltaeff)}}}
	\EEbothHedge \, \fullregret(T)
	\lesssim
	\,
	\sqrt{T \log\numeffexperts}
	+ \frac{(\log\numexperts)^2}{\Deltaeff\log\numeffexperts}.
\]
\end{enumerate}
\end{remark}

\begin{remark}
If $\hedgelr(\numexperts) \propto \sqrt{\log\numexperts}$, then \cref{thm:hedge-lb}(i) combined with \cref{rem:hedge-asymp}(ii) shows that the dependence on $T$ is tight in \cref{thm:hedge-bound-quant}.
If oracle knowledge of $\numeffexperts$ is used to choose $\hedgelr(\numexperts,\numeffexperts) \propto \sqrt{\log\numeffexperts}$, then \cref{thm:hedge-lb}(ii) simply matches the oracle lower bound of \cref{thm:oracle-lb}, confirming the dependence on $T$ is tight in \cref{thm:hedge-bound-quant} (see \cref{rem:hedge-asymp}(iii)).
Finally, if $\hedgelr$ is constant, then \cref{thm:hedge-lb}(ii) combined with \cref{rem:hedge-asymp}(i) shows that the dependence on $T$ is tight in \cref{thm:hedge-bound-quant}.
\end{remark}

Together with the minimax lower bounds of \cref{sec:minimax-lower-bounds}, we find that, for the stochastic and adversarial settings, our \expregretname{} bound for \Hedge{} with $\hedgelr(\numexperts) = \sqrt{\log\numexperts}$ is tight up to constants and that the algorithm achieves the \minimaxoptimal{} rates, as noted by \citet{mourtada2019optimality}.
Furthermore, we have improved upon Corollary~6 of \citet{mourtada2019optimality} in the ``adversarial-with-an-$\EE$-gap'' setting (see \cref{ex:adv-E-gap}),
having removed the extra $\Deltaeff^{-1}\log(\Deltaeff^{-1})$ dependence that separated the upper and lower bounds in their work.

\section{Beating \Hedge{} without oracle knowledge}
\label{sec:ftrl}

In \cref{sec:decreasing-hedge}, we completed the story of \Hedge{} by showing that it does not adapt \minimaxoptimally{} to all possible constraint sets without oracle knowledge of the number of effective experts.
It is natural to ask whether we can design an algorithm that adapts to the number of effective experts and has a \rateregretname{} no larger than
$\sqrt{T\log\numeffexperts}$.

In this section, we present a modified algorithm that does exactly this.
Taking inspiration from the fact that \Hedge{} can be viewed as \emph{follow-the-regularized-leader} (\FTRL) using \emph{entropic regularization} (see, for example, Section~3.6 of \aos{\citep{mcmahan2017survey}}\preprint{\citet{mcmahan2017survey}}), we introduce the \emph{constraint-adaptive root-entropic} (\CARE) regularizer.
We are able to prove upper bounds for the performance of \FTRL{} for a large class of regularizers, and then use these upper bounds to prove both the upper bound results of \cref{sec:decreasing-hedge} and the upper bounds for our improved algorithm,
by viewing \Hedge{} and \FTRLCARE{} as \FTRL{} with specifically chosen regularizers.
Our bound shows that \FTRLCARE{} achieves the oracle rate $\sqrt{T\log\numeffexperts}$ without requiring knowledge of the \constraintparams{} for the constraint set $\distnball{}$.

\subsection{\FTRL{} algorithms}
\label{ssec:generic-ftrl-algs}
\FTRL{} is a generic method for online optimization. 
In the setting of sequential prediction with expert advice, \FTRL{} is parametrized by a sequence of \emph{regularizers} $\set{\regularizer{t}:\simp(\experts) \to \Reals}_{t\in\PosInts}$. Each such sequence, subject to regularity conditions on the regularizers (see \cref{sec:olo}), determines a unique \properpredpolicyname{}.
For each time $t+1$, a \playername{} using the $\FTRLeqn(\set{\regularizer{t}}_{t\in\PosInts})$ algorithm has a \properpredpolicyname{} defined uniquely by the weight vectors given by
\[
	\weightdumvec(t+1) = \argmin_{\weightdumvec\in \simp(\experts)}\rbra{ \inner{\Lossvec(t)}{\weightdumvec}+ \cumregularizer{t}(\weightdumvec) },
\label{eqn:FTRL}
\]
where $\cumregularizer{t}(\weightdumvec) = \sum_{s=0}^t \regularizer{s}(\weightdumvec)$, and the existence and
uniqueness of the $\argmin$ is ensured by the regularity properties of the regularizer. This class of algorithms is well studied in online optimization; for specific results relevant to this work, see \cref{sec:olo}.

\subsection{The constraint-adaptive root-entropic regularizer}
\label{ssec:care}

First, we note that \FTRL{} directly generalizes \Hedge{}.
In particular,
letting $\entropy(\weightdumvec) = -\sum_{\expidx\in\experts} \weightdum_\expidx \log(\weightdum_\expidx)$
denote the \emph{entropy function},
it is well known that, for $\cumregularizer{t}(\weightdumvec) = -{\sqrt{t+1}}\,\entropy(\weightdumvec)/{\hedgelr(\numexperts)} $,
	the weights played by $\FTRLeqn(\set[0]{\regularizer{t}}_{t \in \PosInts})$ are equal to the weights played by \Hedge{}.
We modify the entropic regularizer to achieve improved performance for \envpolicynames{} strictly between stochastic and adversarial.

In order to motivate this new algorithm, we provide the following motivating intuition.
First, from \cref{rem:hedge-asymp}, playing \Hedge{} with $\hedgelr(\numexperts,\numeffexperts) \propto \sqrt{\log\numeffexperts}$ achieves the oracle rate.
Second, %
the \minimaxoptimal{} \envpolicyname{} subject to the time-homogeneous convex constraint forces the \minimaxoptimal{} \playerpolicyname{} to ``concentrate''
to $\unifdist(\effexperts)$.
Finally, for $\weightdumvec = \unifdist(\effexperts)$, $\entropy(\weightdumvec) = \log\numeffexperts$.
These three observations together suggest that, heuristically, playing \OGHedge{} with the ``adaptive'' learning rate $\lr(t) = \sqrt{\entropy(\weightdumvec(t))/t}$ may lead to an oracle \rateregretname{}.
However, $\weightdumvec(t)$ is defined in terms of $\lr(t)$, so this is an implicit system of equations to be solved at each time $t$.
In order to define our modification of \FTRL{}, we choose a regularizer such that the solution to the \FTRL{} optimization problem gives rise to a similar system of equations.
In particular, for some parameters $c_1,c_2 > 0$, the sequence of regularizers is given by
\[\label{eqn:care-regularizer}
	\cumregularizer{t}(\weightdumvec) = -\frac{\sqrt{t+1}}{c_1}\sqrt{\entropy(\weightdumvec) + c_2} \, .
\]
We call $-\regularizer{0}$ defined by \cref{eqn:care-regularizer} a \emph{root-entropy function}, and regularization with $\set[0]{\regularizer{t}}_{t \in \PosInts}$ \emph{constraint-adaptive root-entropic} (\texttt{CARE}) regularization. 
We refer to the algorithm $\FTRLeqn(\set[0]{\regularizer{t}}_{t \in \PosInts})$ with $\regularizer{t}$ induced by \cref{eqn:care-regularizer} as follow-the-regularized-leader with constraint-adaptive root-entropic regularization (or, \FTRLCARE{}).

Throughout the remainder of the paper, we will use $\weightdumvec$ for the weights output by the $\FTRLeqn(\set[0]{\regularizer{t}}_{t \in \PosInts})$ algorithm with a generic regularizer, $\hedgeweightvec$ for weights output via entropic regularization (\OGHedge{}), and $\careweightvec$ for weights output via root-entropic regularization (\FTRLCARE). Pseudocode for an efficient implementation of \FTRLCARE{} may be found in \cref{ssec:care-implementation}.

\preprint{\subsection{Performance of \FTRLCARE{}}}
\aos{\subsection{Performance of {\FTRLCARE{}}}}%
\label{ssec:ftrl-care-bound}

\begin{theorem}
\label{thm:ftrl-care-bound-quant}\label{THM:FTRL-CARE-BOUND-QUANT}
For all $c_1,c_2>0$ used to parametrize \FTRLCARE{}, there exist $\careconst_1,\dots,\careconst_4$ such that for all $\numexperts\geq 2$,  \playerpolicynames{} $\predpolicy\in\predpolicyspace_\numexperts$, convex $\distnball{}\subseteq \meas(\expspace\times\dataspace)$, and $T\in\Nats$,
\*[
  \sup_{\policy\in\policyspace(\distnball{})} \EEboth \playerexpregretCare(T)
  \leq \careconst_1 \sqrt{(T+1)[\log\numexperts + c_2]}.
\]
Moreover, when $T \geq \ceil{\frac{2 [\log\numexperts+\careconst_4]^2}{c_1^2c_2\Deltaeff^2}}$,
the following two cases hold:\\
If $\numeffexperts > 1$, then
\*[
  \sup_{\policy\in\policyspace(\distnball{})} \EEboth \playerexpregretCare(T)
  &\leq
  \frac{33\careconst_1}{32}\sqrt{(T+1)[\log\numeffexperts+c_2]} + \careconst_2\frac{[\log\numexperts+\careconst_4]^{3/2}}{\Deltaeff} + \frac{\careconst_3}{\Deltaeff},
\]
and if $\numeffexperts = 1$, then
\*[
  \sup_{\policy\in\policyspace(\distnball{})} \EEboth \playerexpregretCare(T)
  &\leq
  \careconst_2\frac{[\log\numexperts+\careconst_4]^{3/2}}{\Deltaeff} + \frac{\careconst_3+6}{\Deltaeff}.
\]
The constants $\careconst_1,\dots,\careconst_4$ appearing above are given by:
\*[
  \careconst_1
    & = \rbra{\frac{1}{c_1} + \frac{3c_1}{2}}, &
  \careconst_2
    & = \sqrt{2} \careconst_1 \rbra{\frac{1}{c_1\sqrt{c_2}} +\frac{1}{c_2}}, \\
  \careconst_3
    & = \sqrt{2} \frac{8+12c_1^2}{3c_1^2\sqrt{c_2}} , \ \andT &
  \careconst_4
    & = \max\cbra{c_2,\ 3c_1\sqrt{c_2} + \frac{5c_1^2c_2}{4}}.
  \]
	
	With $c_1=c_2=1$ this simplifies to: 
	for all $T\in\Nats$,
	\*[
	  \sup_{\policy\in\policyspace(\distnball{})} \EEboth \playerexpregretCare(T)
	  \leq 3 \sqrt{(T+1)[\log\numexperts + 1]},
	\]
	and when $T \geq \ceil{\frac{2 [\log\numexperts+5]^2}{\Deltaeff^2}}$,
	if $\numeffexperts > 1$, then
	\*[
	  \sup_{\policy\in\policyspace(\distnball{})} \EEboth \playerexpregretCare(T)
	  &\leq
	  3\sqrt{(T+1)[\log\numeffexperts+1]} + 8\frac{[\log\numexperts+5]^{3/2}}{\Deltaeff} + \frac{10}{\Deltaeff},
	\]
	and if $\numeffexperts = 1$, then
	\*[
	  \sup_{\policy\in\policyspace(\distnball{})} \EEboth \playerexpregretCare(T)
	  &\leq
	  8\frac{[\log\numexperts+5]^{3/2}}{\Deltaeff} + \frac{16}{\Deltaeff}.
	\]
\end{theorem}

\begin{remark}
Taking $\predpolicy$ to be determined by \FTRLCARE{},
\*[
	\sup_{\smash{\policy\in\policyfamily_{\numexperts,(\numeffexperts,\Deltaeff)}}}
  \EEbothCare \, \fullregret(T)
	\lesssim
	\sqrt{T \log\numeffexperts} + \frac{(\log\numexperts)^{3/2}}{\Deltaeff}.
\]
\end{remark}

\begin{remark}
	Note that in the case $\numeffexperts =1$, this is worse than \Hedge{} with learning rate $\hedgelr(\numexperts) \propto \sqrt{\log\numexperts}$, which has $\Deltaeff^{-1}\log\numexperts$ \rateregretname{}.
  We resolve this in \cref{sec:metacare} by introducing a new algorithm, \MetaCARE{}, that combines the optimality of \Hedge{} in the stochastic case and \FTRLCARE{} elsewhere.
\end{remark}

\section{Proofs of upper bounds}
\label{sec:main-quant-sketch}\label{SEC:MAIN-QUANT-SKETCH}

The proofs of \cref{thm:hedge-bound-quant,thm:ftrl-care-bound-quant} rely on several technical results regarding \emph{online linear optimization} developed in \cref{sec:olo,sec:ftrl-simp-proof}.
In order to simplify notation for \FTRL{} with regularizers that are transformations of the entropy function,
we let \FTRLHalg{$\entropyfunc$}{$\timefunc$} denote $\FTRLeqn(\set[0]{\regularizer{t}}_{t \in \PosInts})$ with $\cumregularizer{t} = -\timefunc(t)[\entropyfunc \circ \entropy]$
for any strictly increasing, concave, and twice continuously differentiable function $\entropyfunc:[0,\log\numexperts]\to\Reals$ and strictly increasing $\timefunc:\PosInts\to\PosReals$.
The important conclusions from  \cref{sec:olo,sec:ftrl-simp-proof} are summarized in the following result, the proof of which appears in \cref{ssec:ftrl-simplex-summary-proof}.
The result tells us that the weights played by a \playername{} employing the \FTRLHalg{$\entropyfunc$}{$\timefunc$} strategy are equivalent to the weights played by \OGHedge{} with an implicitly defined, non-deterministic learning rate, and also provides a second-order bound on the \playerexpregretname{} incurred.

\begin{theorem}
\label{THM:FTRL-SIMPLEX-SUMMARY}\label{thm:ftrl-simplex-summary}
For every strictly increasing $\timefunc:\PosInts\to\Reals$, and every strictly increasing, concave, and twice continuously differentiable function $\entropyfunc:[0,\log\numexperts]\to\Reals$,
the solutions to \cref{eqn:FTRL} at time $t$ for
\FTRLHalg{$\entropyfunc$}{$\timefunc$} given any history and expert predictions satisfy the system of equations
\*[
  \ilr({t+1})
  & = \frac{1}{\timefunc(t)\cdot [\entropyfunc'\circ\entropy](\weightdumvec(t+1))} \, , &
  \weightdumvec(t+1)
  & = \rbra{ \frac{\exp\{-\ilr({t+1}) \Loss_\expidx(t)\}}{\sum_{\expidxdum\in\experts}\exp\{-\ilr({t+1}) \Loss_\expidxdum(t)\}}}_{\expidx\in\experts}.
\]
Moreover, for any sequence of losses $(\lossvec(t))_{t \in \Nats} \subseteq [0,1]^\numexperts$, this system has a unique solution satisfying
\*[
  \ilr({t+1}) \in \cinter{\frac{1}{\timefunc(t)\cdot\entropyfunc'(0)},\ \frac{1}{\timefunc(t)\cdot\entropyfunc'(\log\numexperts)}},
\]
and there exists a sequence $\set[0]{\alpha_t}_{t\in\PosInts} \subseteq [0,1]$
such that the \playerexpregretname{}
satisfies
\[
  \hspace{-2em}\playerexpregretFTRLHeqn(T)
    &\leq -\timefunc(T) \entropyfunc(0) + \timefunc(0)[\entropyfunc\circ\entropy](\weightdumvec(1))
      + \sum_{t=1}^T [\timefunc(t)-\timefunc(t-1)]\cdot [\entropyfunc\circ\entropy](\weightdumvec(t+1)) \\
       &\qquad \ +\sum_{t=1}^T \frac{\sqrt{\VVar{\Expidx\sim \intweightvec(t+1)}
         \sbra{\rbra{\frac{\timefunc(t)}{\timefunc(t-1)} -1} \Loss_\Expidx(t-1) - \loss_\Expidx(t)}\times\VVar{\Expidx\sim \intweightvec(t+1)}\sbra{\loss_\Expidx(t)}}}{\timefunc(t)\cdot [\entropyfunc'\circ \entropy](\intweightvec(t+1))},
\label{eqn:main-quasi-regret-bound}
\]
where for each $t\in \PosInts$,
\[\label{eq:int-weight-def}
  \intweightvec(t+1) = \intweightvec\upper{\alpha_t}(t+1),
\]
and for every $t\in \Nats$ and $\alpha\in\cinter{0,1}$, we define
\*[
\intweightvec\upper{\alpha}(t+1) = \argmin_{\intweightvec\in \simp(\experts)} \rbra{\inner{\alpha \Lossvec(t) + (1-\alpha) \sqrt{\frac{t+1}{t}}\Lossvec(t-1)}{\intweightvec}-\sqrt{t+1} \, [\entropyfunc \circ \entropy](\intweightvec)}.
\]
\end{theorem}

Ultimately, we wish to apply \cref{thm:ftrl-simplex-summary} to both \Hedge{} and \FTRLCARE{}. Recall that \Hedge{} corresponds to
\*[
	\entropyfunc(s)
	  =\frac{s}{\hedgelr(\numexperts)}, \
	\entropyfunc'(s)
	  = \frac{1}{\hedgelr(\numexperts)}, \text{ and }
	\timefunc(t)
	  = \sqrt{t+1} \, ,
\]
and therefore, in \cref{eqn:main-quasi-regret-bound},
\*[
	\frac{1}{\timefunc(t)\cdot [\entropyfunc'\circ \entropy](\intweightvec(t+1))}
		& = \frac{\hedgelr(\numexperts)}{\sqrt{t+1}}.
\]
\FTRLCARE{} with parameters $c_1,c_2>0$ corresponds to
\*[
  \entropyfunc(s)
    =\frac{\sqrt{s+c_2}}{c_1}, \
  \entropyfunc'(s)
    = \frac{1}{2 c_1 \sqrt{s+c_2}}, \text{ and }
  \timefunc(t)
    = \sqrt{t+1} \, ,
\]
and therefore, in \cref{eqn:main-quasi-regret-bound},
\*[
	\frac{1}{\timefunc(t)\cdot [\entropyfunc'\circ \entropy](\intweightvec(t+1))}
		& = 2 c_1 \sqrt{\frac{\entropy(\intweightvec(t+1))+c_2}{t+1}}.
\]
Both correspond to the choice $\timefunc(t) = \sqrt{t+1}$, so we focus on this rather than continuing to use a generic $\timefunc(t)$. We leave $\entropyfunc$ as generic for the moment, since the following result equally applies to the algorithms' respective $\entropyfunc$ functions. Finally, we wish to move towards proving bounds on the \expregretname{}, which will require taking expectation with respect to a \envpolicyname{} $\policy$, so we fix a convex $\distnball{} \subseteq \meas(\expspace \times \dataspace)$ that characterizes the allowable \envpolicyname{}s.

In order to control the \playerexpregretname{} using \cref{thm:ftrl-simplex-summary}, we need to control the entropy of the \FTRLH{} weights $\weightdumvec$ as well as the \emph{\intweightname{}} $\intweightvec$ (defined in \cref{eq:int-weight-def}).
The following lemma provides the necessary control, which we prove in \cref{sec:modular-entropy-proof}.

\begin{lemma}\label{lem:modular-entropy}\label{LEM:MODULAR-ENTROPY}
For every $\weightdumvec \in \simp(\experts)$ and $p \in (0,1)$,
\[\label{eqn:modular_entropy}
	\entropy(\weightdumvec)
	& \leq \frac{2}{e\log2}\log\numeffexperts + \rbra{1+\frac{1}{(1-p)e}}\sum_{\expidx\in\neffexperts} [\weightdum_\expidx]^p.
\]
\end{lemma}

Our next lemma bounds the expectation of the second term on the RHS of \cref{eqn:modular_entropy} for the \FTRLH{} weights.
Combined with the previous result, this allows us to bound the expected entropy of the weights.
Crucially, the
bound on the expected weights that \FTRLH{} would produce holds regardless of whether the actual \playerpolicyname{} used is \FTRLH{} or some other policy, allowing us to control the expected \playerexpregretname{} of \FTRLH{} when a different policy is used to interact with the environment, as in the statements of \cref{thm:hedge-bound-quant,thm:ftrl-care-bound-quant}.

\begin{lemma}
\label{lem:ftrl-both-weight-intweight-bound}\label{LEM:FTRL-BOTH-WEIGHT-INTWEIGHT-BOUND}
Letting $\weightdumvec$ denote the weights output by the \FTRLHalg{$\entropyfunc$}{$t \mapsto \sqrt{t+1}$} algorithm,
for every \playerpolicyname{} $\predpolicy$, $t \in \Nats$, $p > 0$, and $\expidx\in\neffexperts$,
\*[
\sup_{\policy \in \policyspace(\distnball{})} \EEboth \Big[[\weightdum_\expidx(t+1)]^p\Big]
  & \leq \exp\left\{\frac{p^2}{2(\entropyfunc'(0))^2}- \frac{\Deltaeff p \sqrt{t}}{\sqrt{2} (\entropyfunc'(0))} \right\},
\]
and
\*[
\sup_{\policy \in \policyspace(\distnball{})} \EEboth \sup_{\alpha\in\cinter{0,1}}\Big[[\intweight\upper{\alpha}_\expidx(t+1)]^p\Big]
  & \leq \exp\left\{\frac{2p}{\entropyfunc'(0)} +\frac{p^2}{2(\entropyfunc'(0))^2}- \frac{\Deltaeff p \sqrt{t}}{\sqrt{2} (\entropyfunc'(0))} \right\}.
\]
\end{lemma}
The intuition underlying the proof of this result is as follows. 
First, let $\lblr(t+1)\defas\frac{1}{\sqrt{t+1}\cdot\entropyfunc'(0)}$.
Note that for
$(\weightdumvec(t))_{t\in\Nats}$ and $(\lr(t))_{t\in\Nats}$ given in \cref{thm:ftrl-simplex-summary}, $\lblr(t+1) \leq \ilr(t+1)$ for all $t\in\NatsO$.
Let $\expidxoptpath(t) = \argmin_{\expidx\in\experts} \Loss_\expidx(t)$, so that
for any $\expidx\in\experts$, $\Loss_{\expidxoptpath(t)}(t) \leq \Loss_\expidx(t)$.
Thus,
\*[
\Big[
  \weightdum_\expidx(t+1)\Big]^p
  \leq   \rbra{\frac{\weightdum_\expidx(t+1)}{\weightdum_{\expidxoptpath(t)}(t+1)}}^p
  \leq \min_{\effexpidx\in\effexperts}\exp\Big\{-p\ \lblr(t+1) \sbra{\Loss_\expidx(t)-\Loss_{\effexpidx}(t)}\Big\}.
\]
Applying \cref{thm:minimax_mgf},
\*[
  \sup_{\policy \in \policyspace(\distnball{})} \EEboth\Big[\weightdum_\expidx(t+1)^p\Big]
  &\leq \exp\left\{- t\lblr(t+1)\Deltaeff p + t\lblr(t+1)^2 \frac{p^2}{2} \right\}.
\]
The argument for the \intweightname{} is similar. For the complete proof, see \cref{apx:ineffective-weight-bound-proof}.

By combining \cref{lem:ftrl-both-weight-intweight-bound} with \cref{lem:modular-entropy} for $p=1/2$,
and noting that, for all $t \in \Nats$,
$2/(e \log2) < 17/16$ and $1 + 2/e < 7/4$, it holds that
\[\label{eqn:entropy-bound-ftrl}
&\hspace{-2em}\sup_{\policy \in \policyspace(\distnball{})} \EEboth \entropy(\weightdumvec(t+1)) \\
   &\qquad\leq  \frac{17}{16}\log\numeffexperts
  + \frac{7}{4}(\numexperts - \numeffexperts) \exp\bigg\{\frac{1}{8(\entropyfunc'(0))^2}- \frac{\Deltaeff\sqrt{t}}{2\sqrt{2} (\entropyfunc'(0))}\bigg\},
\]
and
\[\label{eqn:intentropy-bound-ftrl}
&\hspace{-2em}\sup_{\policy \in \policyspace(\distnball{})} \EEboth \sup_{\alpha\in\cinter{0,1}} \entropy(\intweightvec\upper{\alpha}(t+1)) \\
   &\qquad\leq  \frac{17}{16}\log\numeffexperts
  + \frac{7}{4}(\numexperts - \numeffexperts) \exp\bigg\{\frac{1}{\entropyfunc'(0)} +\frac{1}{8(\entropyfunc'(0))^2}- \frac{\Deltaeff\sqrt{t}}{2\sqrt{2} (\entropyfunc'(0))}\bigg\}.
\]

These bounds can now be used for the regularizers specific to \Hedge{} and \FTRLCARE{}. Our approach will be to break up the sums of \cref{eqn:main-quasi-regret-bound} into the first $\tmid$ rounds and then the remaining rounds for some carefully chosen $\tmid$. Note that $\tmid$ is not a parameter of the algorithm, but rather an artifact of our proof. The rounds after $\tmid$ will be handled using our entropy bounds above, but the early rounds we control with the following worst-case bound. The proof of the following result appears in \cref{apx:early-quasi-regret-bound-proof}. Note that it recovers the correct order of standard adversarial bounds for \Hedge{}.

\begin{lemma}\label{lem:worst-case-ftrl-bound}
For every $\tmid \in \Nats$ and sequence of losses $\set[0]{\lossvec(t)}_{t\in\Nats} \subseteq [0,1]^\numexperts$, the weights played by \FTRLHalg{$\entropyfunc$}{$t \mapsto \sqrt{t+1}$} satisfy
\*[
  \playerexpregretFTRLHeqn(\tmid)
    &\leq  \bigg(\entropyfunc(\log\numexperts) - \entropyfunc(0) + \frac{3}{4\entropyfunc'(\log\numexperts)} \bigg) \sqrt{\tmid+1} \, .
\]
\end{lemma}

The remainder of the proofs of \cref{thm:hedge-bound-quant,thm:ftrl-care-bound-quant} can be found in \cref{sec:main-thm-proofs}, which consists
of substituting in the specific expression for $\entropyfunc$ to \cref{thm:ftrl-simplex-summary,eqn:entropy-bound-ftrl,eqn:intentropy-bound-ftrl,lem:worst-case-ftrl-bound}. Then, the variance terms are controlled by a worst case bound for $\numeffexperts>1$, and by \cref{lem:mix-var-bd} for $\numeffexperts=1$, and the summation terms are controlled by an
integral comparison (see \cref{lem:integral-comparison-simplified}). Finally, $\tmid$ is chosen as specified by the statements of \cref{thm:hedge-bound-quant,thm:ftrl-care-bound-quant} respectively.

\section{\CARE{} if you can, \OGHedge{} if you must, or \MetaCARE{} for all}
\label{sec:metacare}

Since we have seen in \cref{thm:hedge-bound-quant} that \Hedge{} with $\hedgelr(\numexperts) = \sqrt{\log\numexperts}$ achieves the \minimaxoptimal{} order of $\log\numexperts$ when $\numeffexperts = 1$, and \cref{thm:ftrl-care-bound-quant} shows that \FTRLCARE{} is \minimaxoptimal{} in all other cases, it is natural to try to combine these two learners in order to have \minimaxoptimal{} \rateregretname{} for all values of $\numeffexperts$ and $\Deltaeff$.
To achieve this, we introduce the \MetaCARE{} algorithm.

Intuitively, \MetaCARE{} plays both \Hedge{} and \FTRLCARE{}, treating them as two \emph{meta-experts}.
\MetaCARE{} outputs the weighted average of the predictions made by the two meta-experts, where the weighting output by \Hedge{} based on their respective losses.
Consequently, \MetaCARE{} has four parameters: $c_\shortHedge, c_{\shortCare,1}, c_{\shortCare,2}, c_{\shortMeta} >0$.
Formally, for each $t\in\Nats$,
let $\hedgeweightvec(t)$ denote the weight vector produced by \Hedge{} with $\hedgelr(\numexperts) = c_\shortHedge\sqrt{\log\numexperts}$ at time $t$
and let $\careweightvec(t)$ denote the weight produced by \FTRLCARE{} with parameters $c_{\shortCare,1}, c_{\shortCare,2}$ at time $t$.
Consider the \emph{meta-losses} defined by
\*[
	\loss_\shortHedge(t)
		& = \inner{\lossvec(t)}{\hedgeweightvec(t)},
	&\loss_\shortCare(t)
			& = \inner{\lossvec(t)}{\careweightvec(t)}, \\
		\Loss_\shortHedge(t)
			& = \sum_{s=1}^t \loss_\shortHedge(t),
	&\Loss_\shortCare(t)
			& = \sum_{s=1}^t\loss_\shortCare(t).
\]
Then, for each $t \in \Nats$, \MetaCARE{} 
produces the weight vector
\*[
		\metaweightvec(t+1)
			\defas
				\frac{
				   \exp\big\{-\lr_\shortMeta(t) \Loss_\shortHedge(t)\big\}\,\hedgeweightvec(t+1)
				   + \exp\big\{-\lr_\shortMeta(t) \Loss_\shortCare(t)\big\}\,\careweightvec(t+1)
				   }
					{\exp\big\{-\lr_\shortMeta(t) \Loss_\shortHedge(t)\big\} + \exp\big\{-\lr_\shortMeta(t) \Loss_\shortCare(t)\big\}},
\]
where $\lr_\shortMeta(t) = {c_\shortMeta}/{\sqrt{t}}$. Observe that $\metaweightvec(t+1)$ will be an element of $\simp(\experts)$ since it is a convex combination of $\hedgeweightvec(t+1)$ and $\careweightvec(t+1)$, both of which are elements of $\simp(\experts)$.

\begin{theorem}\label{thm:metacare}
\MetaCARE{} parametrized by \paramRecAllsimplified{} incurs
\*[
	\sup_{\smash{\policy\in\policyfamily_{\numexperts,(\numeffexperts,\Deltaeff)}}}
	\EEbothMeta \, \fullregret(T)
	\lesssim
	\,
	\sqrt{T \log\numeffexperts}
	+ \frac{\log\numexperts}{\Deltaeff}.
\]
\end{theorem}

We do not state a detailed quantitative form of \cref{thm:metacare}, since our proof can be easily extended for any arbitrary $\predpolicy$ to a bound on $\EEboth \playerexpregretMeta(T)$ with exact constants using the statements and proofs of \cref{thm:hedge-bound-quant,thm:ftrl-care-bound-quant}.

\begin{proof}[Proof of \cref{thm:metacare}]

For $\numeffexperts \geq 2$, we decompose the \playerexpregretname{} of \MetaCARE{} into components coming from the \playerexpregretname{} due to meta-learning and the \playerexpregretname{} of the better of the two meta-experts. In particular, for any sequence of losses $(\lossvec(t))_{t \in \Nats}$, we can write
\*[
	\playerexpregretMeta(T)
		& = \sum_{t=1}^T \inner{\lossvec(t)}{\metaweightvec(t)}
				- \min_{\expidx\in\experts} \sum_{t=1}^T \loss_{\expidx}(t) \\
		& = \sbra{
			\sum_{t=1}^T \inner{\lossvec(t)}{\metaweightvec(t)}
				- \min\rbra{\sum_{t=1}^T \inner{\lossvec(t)}{\hedgeweightvec(t)},
						\sum_{t=1}^T \inner{\lossvec(t)}{\careweightvec(t)}}} \\
			&\qquad + \min\rbra{\playerexpregretHedge(T), \playerexpregretCare(T)} .
\]
Therefore, for any $\numeffexperts \leq \numexperts$ and $\Deltaeff$,
\*[
	&\hspace{-2em}
	\sup_{\smash{\policy\in\policyfamily_{\numexperts,(\numeffexperts,\Deltaeff)}}}
	\EEbothMeta \, \fullregret(T) \\
		&\leq 
		\sup_{\smash{\policy\in\policyfamily_{\numexperts,(\numeffexperts,\Deltaeff)}}}
		\EEbothMeta
		\sbra{
			\sum_{t=1}^T \inner{\lossvec(t)}{\metaweightvec(t)}
				- \min\rbra{\sum_{t=1}^T \inner{\lossvec(t)}{\hedgeweightvec(t)},
						\sum_{t=1}^T \inner{\lossvec(t)}{\careweightvec(t)}}} \\
		& \qquad + 
			\sup_{\smash{\policy\in\policyfamily_{\numexperts,(\numeffexperts,\Deltaeff)}}}
			\EEbothMeta\min\rbra{\playerexpregretHedge(T), \playerexpregretCare(T)}.
\]

First, we consider the case when $\numeffexperts \geq 2$.
Since \MetaCARE{} is \Hedge{} with two experts given by the predictions of \Hedge{} and \FTRLCARE{}, \cref{thm:hedge-bound-quant} implies that
\[\label{eqn:meta1}
&\hspace{-2em}
	\sup_{\smash{\policy\in\policyfamily_{\numexperts,(\numeffexperts,\Deltaeff)}}}
	\EEbothMeta
	\sbra{
	\sum_{t=1}^T \inner{\lossvec(t)}{\metaweightvec(t)}
	- \min\rbra{\sum_{t=1}^T \inner{\lossvec(t)}{\hedgeweightvec(t)},
			\sum_{t=1}^T \inner{\lossvec(t)}{\careweightvec(t)}}} \\
	&\leq \sqrt{T+1}\rbra{\frac{\log(2)}{c_\shortMeta}+\frac{3c_\shortMeta}{4}}.
\]
Then, since $(\log\numexperts)^{3/2}\Deltaeff^{-1}$ is lower order according to our $\lesssim$ notation when $\numeffexperts \geq 2$, from \cref{thm:ftrl-care-bound-quant} we obtain
\[\label{eqn:meta2}
	\sup_{\smash{\policy\in\policyfamily_{\numexperts,(\numeffexperts,\Deltaeff)}}}
	\EEbothMeta
	\min\rbra{\playerexpregretHedge(T), \playerexpregretCare(T)}
	&\leq 
	\sup_{\smash{\policy\in\policyfamily_{\numexperts,(\numeffexperts,\Deltaeff)}}}
		\EEbothMeta
	\playerexpregretCare(T)
		\lesssim
		\sqrt{T\log\numeffexperts}.
\]
Combining \cref{eqn:meta1,eqn:meta2} implies that, when $\numeffexperts \geq 2$,
\*[
	\sup_{\smash{\policy\in\policyfamily_{\numexperts,(\numeffexperts,\Deltaeff)}}}
	\EEbothMeta \, \fullregret(T)
		& \lesssim
		\sqrt{T\log\numeffexperts}.
\]

Now consider the case where $\numeffexperts=1$, and let $\effexperts=\set{\effexpidx}$. Using a similar decomposition to the previous case, we have
\*[
	\playerexpregretMeta(T)
		& = \sum_{t=1}^T \inner{\lossvec(t)}{\metaweightvec(t)}
				- \min_{\expidx\in\experts} \sum_{t=1}^T \loss_{\expidx}(t) \\
	& = \sbra{
		\sum_{t=1}^T \inner{\lossvec(t)}{\metaweightvec(t)}
			- \sum_{t=1}^T \inner{\lossvec(t)}{\hedgeweightvec(t)}}
		+ \playerexpregretHedge(T).
\]
Let
$\tmid$ be as in \cref{thm:ftrl-care-bound-quant} (with $(c_1,c_2) = (c_{\shortCare,1},c_{\shortCare,2})$),
so that
$\tmid\lesssim
\frac{(\log\numexperts)^2}{\Deltaeff^2}$. Expanding the \playerexpregretname{} of \MetaCARE{} and using the boundedness of the losses gives
\*[
\playerexpregretMeta(T)
 & = \sbra{
	 \sum_{t=1}^{\tmid} \inner{\lossvec(t)}{\metaweightvec(t)}
		 - \sum_{t=1}^{\tmid} \inner{\lossvec(t)}{\hedgeweightvec(t)}} \\
	 & \qquad + \sbra{
		 \sum_{t=\tmid+1}^{T} \inner{\lossvec(t)}{\metaweightvec(t)}
			 - \sum_{t=\tmid+1}^{T} \inner{\lossvec(t)}{\hedgeweightvec(t)}}
	 + \playerexpregretHedge(T) \\
	 &\leq \sbra{
		 \sum_{t=1}^{\tmid} \inner{\lossvec(t)}{\metaweightvec(t)}
			 - \min\rbra{\sum_{t=1}^{\tmid} \inner{\lossvec(t)}{\hedgeweightvec(t)},
						\sum_{t=1}^{\tmid} \inner{\lossvec(t)}{\careweightvec(t)}}} \\
		 & \qquad + \sum_{t=\tmid+1}^{T} \frac{1}{2}\norm{\careweightvec(t) - \hedgeweightvec(t)}_{L^1}
		 + \playerexpregretHedge(T).
\]

Therefore,
\*[
	&\hspace{-2em}
	\sup_{\smash{\policy\in\policyfamily_{\numexperts,(1,\Deltaeff)}}}
	\EEbothMeta \, \fullregret(T) \\
		& \leq 
		\sup_{\smash{\policy\in\policyfamily_{\numexperts,(1,\Deltaeff)}}}
		\EEbothMeta
		\sbra{
		\sum_{t=1}^{\tmid} \inner{\lossvec(t)}{\metaweightvec(t)}
			- \min\rbra{\sum_{t=1}^{\tmid} \inner{\lossvec(t)}{\hedgeweightvec(t)},
					\sum_{t=1}^{\tmid} \inner{\lossvec(t)}{\careweightvec(t)}}} \\
		& \qquad + 
		\sup_{\smash{\policy\in\policyfamily_{\numexperts,(1,\Deltaeff)}}}
		\EEbothMeta
		\sum_{t=\tmid}^\infty \frac{1}{2}\norm{\careweightvec(t) - \hedgeweightvec(t)}_{L^1}
				+ 
				\sup_{\smash{\policy\in\policyfamily_{\numexperts,(1,\Deltaeff)}}}
				\EEbothMeta \, \playerexpregretHedge(T).
\]

Again using the fact that \MetaCARE{} is \Hedge{} with two experts given by the predictions of \Hedge{} and \FTRLCARE{}, by \cref{thm:hedge-bound-quant} we have
\[\label{eqn:meta-effexp-2}
&\hspace{-2em}
	\sup_{\smash{\policy\in\policyfamily_{\numexperts,(1,\Deltaeff)}}}
		\EEbothMeta
	\sbra{
 \sum_{t=1}^{\tmid} \inner{\lossvec(t)}{\metaweightvec(t)}
	 - \min\rbra{\sum_{t=1}^{\tmid} \inner{\lossvec(t)}{\hedgeweightvec(t)},
			 \sum_{t=1}^{\tmid} \inner{\lossvec(t)}{\careweightvec(t)}}}\\
	& \lesssim
	\sqrt{\tmid} \\
	&\lesssim
	\frac{\log\numexperts}{\Deltaeff}.
\]
Next, using the triangle inequality, the fact that $1 - \hedgeweight_{\effexpidx}(t) = \sum_{\expidx\in\neffexperts} \hedgeweight_\expidx(t)$ along with the same fact for $\careweightvec$, and \cref{lem:ftrl-both-weight-intweight-bound,lem:integral-comparison-simplified} (see also the proofs of \cref{thm:hedge-bound-quant,thm:ftrl-care-bound-quant} for more details),
\[\label{eqn:meta-effexp-3}
 &\hspace{-2em}
 \sup_{\smash{\policy\in\policyfamily_{\numexperts,(1,\Deltaeff)}}}
		\EEbothMeta
 \sum_{t=\tmid}^\infty \frac{1}{2}\norm{\careweightvec(t) - \hedgeweightvec(t)}_{L^1} \\
 	& \leq 
 	\sup_{\smash{\policy\in\policyfamily_{\numexperts,(1,\Deltaeff)}}}
		\EEbothMeta
 	\sum_{t=\tmid}^\infty \frac{1}{2} \rbra{\norm{\careweightvec(t) - \delta_{\effexpidx}}_{L^1} + \norm{\hedgeweightvec(t) - \delta_{\effexpidx}}_{L^1}}\\
	&= 
	\sup_{\smash{\policy\in\policyfamily_{\numexperts,(1,\Deltaeff)}}}
		\EEbothMeta
	\sum_{t=\tmid}^T \sum_{\expidx\in\neffexperts} \Big(\hedgeweight_\expidx(t) + \careweight_\expidx(t)\Big) \\
	& \lesssim
	\frac{1}{\Deltaeff},
\]
where $\delta_{\effexpidx}$ is the point-mass on $\effexpidx$ (equivalently, the weight vector with weight $1$ on expert $\effexpidx$ and $0$ on the others).

Finally, from \cref{thm:hedge-bound-quant},
we have
\[\label{eqn:meta-effexp-1}
	\sup_{\smash{\policy\in\policyfamily_{\numexperts,(1,\Deltaeff)}}}
				\EEbothMeta \, \playerexpregretHedge(T)
 	& \lesssim
 	\frac{\log\numexperts}{\Deltaeff}.
\]

Combining \cref{eqn:meta-effexp-1,eqn:meta-effexp-2,eqn:meta-effexp-3} shows that, in the case of $\numeffexperts=1$, we have
\*[
	\sup_{\smash{\policy\in\policyfamily_{\numexperts,(1,\Deltaeff)}}}
				\EEbothMeta \, \fullregret(T)
	\lesssim
	\frac{\log\numexperts}{\Deltaeff}.
\]

\end{proof}
\section{Related work}
\label{sec:literature}

The existing literature on statistical decision making with sequential data is vast, spanning decades and at least two major fields of study:
sequential decision theory began as a sub-field of statistics, and the historical literature is rather exclusive to statistics, while the more recent literature on decision procedures without i.i.d. assumptions has largely been developed within machine learning and computer science.
In this section,
we highlight the most relevant notions of adaptivity, and how their statistical interpretations differ from each other as well as the present work.

\subsection{Distributional assumptions}\label{sec:lit-distribution}

First, note that while we use the language of prediction to describe our setting, our prediction space $\predspace$ is distinct from the observation space $\dataspace$, so we achieve the same level of generality as allowing for arbitrary decisions.
Classically, the statistical literature on sequential hypothesis testing 
\citep{wald45,robbins1952,begg1979sequential,lai1988,chambaz2017} 
and sequential parameter estimation 
\citep{wolfowitz47,anscombe53estimation,gilliland68,rasmussen1980} 
relies on assumptions on the joint dependence structure of data to obtain performance guarantees.
From a minimax perspective, removing the assumptions on the dependence structure reduces the problem to adversarially chosen data. Instead, by characterizing these arbitrary distributions in some way such that performance depends on the characterization, we can design methods for which the performance adapts to the characterization.

\citet{hanneke2017learning} provides an overview of when classical estimation procedures designed for \iid{} data will be consistent under various non-stationarity conditions.
Additionally, he considers the asymptotic performance of a broader class of algorithms, although there is no notion of \emph{adaptivity} since performance is binary: either a dependence structure admits a consistent online learning algorithm or it doesn't.
In contrast, since the present work deals with finite expert classes, there is always a consistent algorithm, and so we focus on the specific performance of algorithms beyond their convergence properties.

\citet{rakhlin2011online} consider general constraints on the \envpolicyname{} for sequential prediction.
We also use constraints on the \envpolicyname{} to define relaxations of the \iid{} assumption, but the specific constraints that we define and study are not ones studied by \citeauthor{rakhlin2011online}.
Additionally, we focus on developing methods that are minimax optimal under the constraint  even when the nature of the constraint is unknown. 
For each of the constraints analyzed by \citeauthor{rakhlin2011online}, the authors bound the minimax \regretname{} non-constructively, and consequently cannot guarantee the existence of an algorithm that is \adaptminimaxoptimal{}.
In contrast, we provide an explicit, efficient algorithm that is \adaptminimaxoptimal{} for our constraint framework.

\subsection{Notions of easy data}\label{sec:lit-easy}

Beyond quantifying the minimax performance of decision rules under distributional assumptions, significant progress has been made over the last decade towards \regretname{} bounds that depend on key summary statistics of the observed data sequence. While the terminology for these types of bounds varies in the literature, we will follow the nomenclature of \citet{cesabianchi2007secondorder}, who differentiate between zero-order, first-order, and second-order \regretname{} bounds. We use stochastic constraints to link zero- and second-order bounds in a general framework, and hence can compare with results derived in a wide range of settings.

Zero-order bounds refer to those that depend only on the time horizon, the size of the expert class, and an absolute bound on the size of the predictions (alternatively, the losses). Results of this nature have existed for many years, beginning with \citet{littlestone94} and \citet{vovk98}, and are concisely summarized by \citet{plg07}. These bounds are often dubbed \emph{worst-case} or \emph{adversarial}, since they hold for any sequence of observations subject to the aforementioned global constraints.

In contrast, first-order bounds control \regretname{} in terms of a data-dependent quantity; namely, the sum of the actual observed losses (potentially over all experts, or just the best expert for tighter results). Hence, they may lead to much tighter bounds than  zero-order guarantees if the observed losses end up being in a much tighter range than is guaranteed by some absolute bound on the size of the losses. The first bound of this form was by \citet{freund97} for the \OGHedge{} algorithm, which was later upgraded to a multiplicative rather than additive dependence on the cumulative best loss \citep[Corollary~2.4]{plg07}. Similar bounds have been developed for the bandit setting \citep{auer2003bandit,audibert10}, algorithms with adaptive parametrization \citep{hutter2004,vanerven11}, and the combination of adaptive parametrization with partial information \citep{neu2015}.

However, a limitation of first-order bounds is that they are not translation-invariant in the losses. In particular, they suggest that every expert incurring loss of $1$ on each round is much harder to compete against than every expert incurring loss of zero on each round, which is not the case. One solution is to obtain \regretname{} bounds that are similar to first-order, but rather than depending on the sum of the losses, they depend on a single first-order translation-invariant parameter that characterizes the observed loss sequence. In the bandit setting, examples of such a parameter include the effective loss range \citep{cb18,thune18easy} and the amount of corruption allowed on the mean of the losses
\citep{lykouris18corruption,gupta19corruption}. A similar analysis of corruption of experts' predictions in the full-information setting has recently appeared by \citet{amir2020prediction}.

Beyond these first-order quantities, another line of work has focused on second-order bounds, which depend on some form of variation of the observed losses. The first results of this form were derived by \citet{cesabianchi2007secondorder}, who obtain a bound in terms of the sum of the squared losses via tuning the learning rate for \Hedge. This was extended by both \citet{mcmahan2010adaptive} and \citet{hazan2010variance} to depend on the sample second moment and variance respectively of the losses (empirically along the trajectory of observations), and again by \citet{hazan2011bandits} to obtain the same in the bandit setting. Both \citet{vanerven11} and \citet{derooij14FTL} obtain similar variation bounds which are smaller for a different notion of ``easy'' data (defined by the mixability of the loss). Finally, another type of second-order bound was developed by \citet{gaillard14}, where they utilize the squared difference of algorithm losses with expert losses.

A different perspective on easy data is taken by \citet{chaudhuri2009} and \citet{luo2015}, who develop methods not only to have \regretname{} relative to the best expert of size $\Oo(\sqrt{T\log N})$, but to also have \regretname{} relative to the $\eps\numexperts$-quantile expert of size $\Oo(\sqrt{T\log(1/\eps)})$ for all $\eps\in\ointer{1/\numexperts, 1}$. The algorithms they propose are more optimistic than \Hedge{} in the sense that they trust the past data more, which leads to suboptimal performance in settings between stochastic and adversarial, exaggerating the shortcomings of the standard parametrization of \Hedge{} in this case.

Several other methods exist that tune the learning rate of \OGHedge{} adaptively based on the past interaction with the environment. Generally, these are motivated by improved second order bounds. Examples include \citet{koolen2015} and \citet{vanErven2016}, who use a prior on the learning rate and meta-experts for a discrete collection of possible learning rates respectively.

We also derive second-order (in particular, variance) bounds for the observed data sequence (see the intermediary result \cref{thm:ftrl-simplex-summary}). However, we are also able to extend this notion due to the stochastic nature of our constraints. In particular, once we take the expectation (with respect to the \envpolicyname{} and the \playername{}'s actions) of our second-order bounds, we obtain bounds directly comparable to (and tighter than) existing zero-order bounds. This provides greater insight than existing second-order bounds, which often leave a direct dependence on the variability of the chosen learning algorithm that is not \apriori{} clear, and do not explicitly characterize what an ``easy'' data sequence actually looks like.

In the full-information setting, another line of investigation describes ``easy'' stochastic data by that which satisfies a \emph{Bernstein condition}; that is, the conditional second moment of the losses are controlled by a concave function of the conditional first moment. This condition was shown to be crucial for achieving \emph{fast-rates} in the batch setting by \citet{bartlett06}, then in the online convex optimization setting (infinite expert class) by \citet{vanErven2016}, and finally for simultaneously the finite expert and infinite expert online setting by \citet{koolen16}.
Recent work by \citet{grunwald2020fastrates} provided sufficient conditions to extend these results to unbounded losses.

\subsection{Stochastic and adversarially optimal algorithms}\label{sec:lit-sao}

In addition to developing bounds for ``easy'' data, the line of work most relevant to the present paper has focused on developing algorithms that are simultaneously optimal in two key settings: worst-case adversarial observations and \iid{} (stochastic) observations. These bounds are characterized by matching the adversarial bounds mentioned above and the optimal stochastic bounds for either bandits \citep[Theorem~1]{auer2002stochastic} or full-information \citep[Theorem~11]{gaillard14}. Beginning with \citet{audibert09} and \citet{bubeck12SAO}, the bandit literature is rich in this area; contributions include removing prior knowledge of the time horizon \citep{seldin2014one}, matching lower bounds \citep{auer16}, and a simultaneously optimal algorithm 
with respect to a slightly weaker notion of \regretname{} \citep{zimmert19}.

In our discussion of the previous bounds, we have not specifically distinguished between the types of algorithms used to achieve them. However, there is an aesthetic (and computational) desire to find algorithms that achieve \regretname{} bounds that are optimal both for worst-case data and some notion of ``easy'' data, and yet are as simple as the algorithms which perform well in either just the adversarial or just the \iid{} setting. A recent breakthrough on this front was achieved by \citet{mourtada2019optimality}, who showed the standard parametrization of the \Hedge{} algorithm is optimal for both the adversarial and the stochastic settings. For the bandit setting, the \TsallisInf{} algorithm of \citet{zimmert19} has a similarly simple aesthetic; namely, it is also an analytic solution to an \FTRL{} problem with an appropriate regularizer.
One of the more surprising contributions of our work is that we show \emph{every} pre-specified parametrization of \Hedge{} is not \adaptminimaxoptimal{}.

\preprint{

\section*{Acknowledgements}\label{sec:acknowledgements}
BB is supported by an NSERC Canada Graduate Scholarship and the Vector Institute.
JN is supported by an NSERC Vanier Canada Graduate Scholarship and the Vector Institute.
DMR is supported in part by an NSERC Discovery Grant, Ontario Early Researcher Award, Canada CIFAR AI Chair funding through the Vector Institute, and a stipend provided by the Charles Simonyi Endowment.
This material is based also upon work supported by the United States Air Force under Contract No. FA850-19-C-0511. Any opinions, findings and conclusions or recommendations expressed in this material are those of the author(s) and do not necessarily reflect the views of the United States Air Force.
This research was partially carried out while all three authors were visiting the Institute for Advanced Study in Princeton, New Jersey for the Special Year on Optimization, Statistics, and Theoretical Machine Learning.
JN and BB's travel to the Institute for Advanced Study were separately funded by NSERC Michael Smith Foreign Study Supplements.
We thank Teodor Vanislavov Marinov for helpful discussions at the IAS, as well as Nicol\`{o} Campolongo, Peter D. Gr\"unwald, Francesco Orabona, Alex Stringer, Csaba Szepesv\'{a}ri, Yanbo Tang, and Julian Zimmert for their insightful comments on preliminary versions of this work.
}
\aos{

\section*{Acknowledgements}\label{sec:acknowledgements}
BB is supported by an NSERC Canada Graduate Scholarship and the Vector Institute.
JN is supported by an NSERC Vanier Canada Graduate Scholarship and the Vector Institute.
DMR is supported in part by an NSERC Discovery Grant, Ontario Early Researcher Award, Canada CIFAR AI Chair funding through the Vector Institute, and a stipend provided by the Charles Simonyi Endowment.
This material is based also upon work supported by the United States Air Force under Contract No. FA850-19-C-0511. Any opinions, findings and conclusions or recommendations expressed in this material are those of the author(s) and do not necessarily reflect the views of the United States Air Force.
This research was partially carried out while all three authors were visiting the Institute for Advanced Study in Princeton, New Jersey for the Special Year on Optimization, Statistics, and Theoretical Machine Learning.
JN and BB's travel to the Institute for Advanced Study were separately funded by NSERC Michael Smith Foreign Study Supplements.
We thank Teodor Vanislavov Marinov for helpful discussions at the IAS, as well as Nicol\`{o} Campolongo, Peter D. Gr\"unwald, Francesco Orabona, Alex Stringer, Csaba Szepesv\'{a}ri, Yanbo Tang, and Julian Zimmert for their insightful comments on preliminary versions of this work.
}

\aos{
\begin{supplement}
\stitle{Supplement to ``Relaxing the \IID{} Assumption: Adaptively Minimax Optimal Regret via Root-Entropic Regularization''}
\slink[doi]{COMPLETED BY THE TYPESETTER}
\sdatatype{.pdf}
\sdescription{Additional Proofs, Additional Proof Details, Implementation Details, and Simulations}
\end{supplement}
}

\preprint{
\bibliographystyle{abbrvnat}
\bibliography{bib-files/semi-adv}
}

\aos{
\bibliographystyle{imsart-nameyear}
\bibliography{bib-files/semi-adv}
}
\newpage
\appendix

\section{Additional details for proofs of upper bounds}
\label{sec:main-thm-proofs}\label{SEC:MAIN-THM-PROOFS}

In this section, we complete the argument sketched in \cref{sec:main-quant-sketch}.

\subsection{Details for \cref{thm:hedge-bound-quant}}
Substituting in that \Hedge{} with parameter $\hedgelr$ corresponds to, for a given $\numexperts\in\Nats$, $\entropyfunc(s) =s/\hedgelr(\numexperts)$, \cref{thm:ftrl-simplex-summary} says that the weights $\hedgeweightvec$ lead to \playerexpregretname{} bounded by
\[\label{eqn:hedge-regret-bound-1}
	\playerexpregretHedge(T) 
	  & \leq \frac{\log\numexperts}{\hedgelr(\numexperts)} +  \sum_{t=1}^T \frac{\sqrt{t+1}-\sqrt{t}}{\hedgelr(\numexperts)} \entropy(\hedgeweightvec(t+1)) \\
	&\qquad +\sum_{t=1}^T \frac{\hedgelr(\numexperts)\sqrt{\VVar{\Expidx\sim \intweightvec(t+1)} \sbra{\bigg(\frac{\sqrt{t+1}}{\sqrt{t}} -1\bigg) \Loss_\Expidx(t-1) - \loss_\Expidx(t)} \vphantom{\frac{\timefunc(t)}{\timefunc(t-1)}}\VVar{\Expidx\sim \intweightvec(t+1)}\sbra{\loss_\Expidx(t)}}}{\sqrt{t+1}},
\]

where 
\*[
	\intweightvec(t+1) = \underset{\intweightvec\in \simp(\experts)}{\argmin} \rbra{\inner{\alpha_t \Lossvec(t) + (1-\alpha_t) \frac{\sqrt{t+1}}{\sqrt{t}}\Lossvec(t-1)}{\intweightvec}-\frac{\sqrt{t+1}}{\hedgelr(\numexperts)}\entropy(\intweightvec)}
\]
for some $\alpha_t\in[0,1]$. 
Then, recalling that $\entropyfunc'(s) = 1/\hedgelr(\numexperts)$, we can split up \cref{eqn:hedge-regret-bound-1} into the rounds before some $\tmid \in \Nats$ and the rounds after by applying \cref{lem:worst-case-ftrl-bound}. 
That is, when $T \leq \tmid$, we use the bound of \cref{lem:worst-case-ftrl-bound}, and if $T > \tmid$ we have

\[ \label{eqn:hedge-regret-bound-split}
	\playerexpregretHedge(T) 
	& \leq \sqrt{\tmid+1} \bigg(\frac{\log\numexperts}{\hedgelr(\numexperts)} + \frac{3\hedgelr(\numexperts)}{4} \bigg)
	  + \sum_{t=\tmid+1}^T \frac{\sqrt{t+1}-\sqrt{t}}{\hedgelr(\numexperts)} \entropy(\hedgeweightvec(t+1)) \\
	&\quad +\sum_{t=\tmid+1}^T \frac{\hedgelr(\numexperts)\sqrt{\VVar{\Expidx\sim \intweightvec(t+1)} \sbra{\bigg(\frac{\sqrt{t+1}}{\sqrt{t}} -1\bigg) \Loss_\Expidx(t-1) - \loss_\Expidx(t)} \vphantom{\frac{\timefunc(t)}{\timefunc(t-1)}}\VVar{\Expidx\sim \intweightvec(t+1)}\sbra{\loss_\Expidx(t)}}}{\sqrt{t+1}}.
\]

Next, substituting $\entropyfunc$ and $\entropyfunc'$ for \Hedge{} into \cref{eqn:entropy-bound-ftrl}, we get
\[\label{eqn:entropy-bound-hedge}
&\hspace{-2em}\sup_{\policy \in \policyspace(\distnball{})} \EEboth \entropy(\hedgeweightvec(t+1)) \\
   &\qquad\leq  \frac{17}{16} \log\numeffexperts
  + \frac{7}{4}(\numexperts - \numeffexperts) \exp\left\{\frac{[\hedgelr(\numexperts)]^2}{8}\right\} \exp\left\{- \frac{\hedgelr(\numexperts) \Deltaeff}{2\sqrt{2}}\sqrt{t} \right\}.
\]
Thus,
\[\label{eqn:entropy-bound-hedge-summed}
	&\hspace{-2em}\sum_{t=\tmid+1}^T \frac{\sqrt{t+1}-\sqrt{t}}{c} \sup_{\policy \in \policyspace(\distnball{})} \EEboth \entropy(\hedgeweightvec(t+1)) \\
	&\leq \frac{17\log\numeffexperts  \Big[\sqrt{T+1} - \sqrt{\tmid+1} \Big]}{16\hedgelr(\numexperts)} \\
  	&\qquad\qquad+ \frac{7(\numexperts - \numeffexperts) \exp\left\{\frac{[\hedgelr(\numexperts)]^2}{8}\right\} }{8\hedgelr(\numexperts)}\sum_{t=\tmid+1}^T \frac{\exp\left\{- \frac{\hedgelr(\numexperts) \Deltaeff}{2\sqrt{2}}\sqrt{t} \right\} }{\sqrt{t}} \\
  &\leq \frac{17\log\numeffexperts \Big[\sqrt{T+1} - \sqrt{\tmid+1} \Big]}{16\hedgelr(\numexperts)}  \\
  &\qquad\qquad+ \frac{7(\numexperts - \numeffexperts) \exp\left\{\frac{[\hedgelr(\numexperts)]^2}{8} - \frac{\hedgelr(\numexperts) \Deltaeff}{2\sqrt{2}}\sqrt{\tmid} \right\}}{\sqrt{2} \, [\hedgelr(\numexperts)]^2 \Deltaeff},
\]
where the last step comes from applying \cref{lem:integral-comparison-simplified} to bound the summation.
For the last term of \cref{eqn:hedge-regret-bound-split}, we consider the cases of $\numeffexperts>1$ and $\numeffexperts=1$ separately. 
For both, however, we will use $\tmid = \ceil{\frac{8 (\log(\numexperts)+[\hedgelr(\numexperts)]^2/4+\hedgelr(\numexperts))^2}{[\hedgelr(\numexperts)]^2\Deltaeff^2}}$.

\ \\

\textbf{\OGHedge{} upper bound: $\numeffexperts > 1$.}\\
If $\numeffexperts>1$, 
using \cref{lem:simple-var-bd} to bound the variances gives
\[\label{eqn:variance-bound-hedge-summed}
	&\hspace{-2em}\sum_{t=\tmid+1}^T \frac{\hedgelr(\numexperts)}{\sqrt{t+1}}\sqrt{\VVar{\Expidx\sim \intweightvec(t+1)} \sbra{\bigg(\frac{\sqrt{t+1}}{\sqrt{t}} -1\bigg) \Loss_\Expidx(t-1) - \loss_\Expidx(t)}\vphantom{\left(\frac{\timefunc(t)}{\timefunc(t-1)}\right)}\VVar{\Expidx\sim \intweightvec(t+1)}\sbra{\loss_\Expidx(t)}} \\
	&\leq \frac{3\hedgelr(\numexperts)}{8} \sum_{t=\tmid+1}^T \frac{1}{\sqrt{t+1}} \\
	&\leq \frac{3\hedgelr(\numexperts)}{4}\Big[\sqrt{T+1} - \sqrt{\tmid+1} \Big].
\]
Combining \cref{eqn:hedge-regret-bound-split,eqn:entropy-bound-hedge-summed,eqn:variance-bound-hedge-summed} 
gives that
\[\label{eqn:multiple-hedge-summed} 
	&\hspace{-2em}\sup_{\policy \in \policyspace(\distnball{})} \EEboth \playerexpregretHedge(T) \\ 
	&\leq \sqrt{\tmid+1} \bigg(\frac{\log\numexperts}{\hedgelr(\numexperts)} + \frac{3\hedgelr(\numexperts)}{4} \bigg) 
	+ \frac{17\log\numeffexperts \Big[\sqrt{T+1} - \sqrt{\tmid+1} \Big]}{16\hedgelr(\numexperts)}  \\
  	& \qquad + \frac{7(\numexperts - \numeffexperts) \exp\left\{\frac{[\hedgelr(\numexperts)]^2}{8} - \frac{\hedgelr(\numexperts)\Deltaeff}{2\sqrt{2}}\sqrt{\tmid} \right\}}{\sqrt{2} \, [\hedgelr(\numexperts)]^2 \Deltaeff} 
  	+ \frac{3\hedgelr(\numexperts)\Big[\sqrt{T+1} - \sqrt{\tmid+1} \Big]}{4} \\
  	&= \sqrt{T + 1} \bigg(\frac{3\hedgelr(\numexperts)}{4} + \frac{17\log\numeffexperts}{16\hedgelr(\numexperts)}  \bigg)
  	+ \sqrt{\tmid + 1} \bigg(\frac{\log\numexperts}{\hedgelr(\numexperts)} - \frac{17\log\numeffexperts }{16\hedgelr(\numexperts)}  \bigg) \\
  	& \qquad + \frac{7(\numexperts - \numeffexperts) \exp\left\{\frac{[\hedgelr(\numexperts)]^2}{8} - \frac{\hedgelr(\numexperts) \Deltaeff}{2\sqrt{2}}\sqrt{\tmid} \right\}}{\sqrt{2} \, [\hedgelr(\numexperts)]^2 \Deltaeff}. 
\]

Substituting $\tmid$ into \cref{eqn:multiple-hedge-summed} gives
\[\label{eqn:multiple-hedge-summed-tmid} 
	&\hspace{-2em}\sup_{\policy \in \policyspace(\distnball{})} \EEboth \playerexpregretHedge(T)\\
	&\leq \sqrt{T + 1} \bigg(\frac{3\hedgelr(\numexperts)}{4} + \frac{17}{16\hedgelr(\numexperts)} \log\numeffexperts \bigg) \\
  	&\qquad + \sqrt{\frac{8 (\log(\numexperts)+[\hedgelr(\numexperts)]^2/4+\hedgelr(\numexperts))^2}{[\hedgelr(\numexperts)]^2\Deltaeff^2} + 2} \ \bigg(\frac{\log(\numexperts) - \log(\numeffexperts )}{\hedgelr(\numexperts)}  \bigg) \\
  	& \qquad + \frac{7(\numexperts - \numeffexperts) \exp\left\{\frac{[\hedgelr(\numexperts)]^2}{8} - \frac{\hedgelr(\numexperts) \Deltaeff}{2\sqrt{2}}\sqrt{\frac{8 (\log(\numexperts)+[\hedgelr(\numexperts)]^2/4+\hedgelr(\numexperts))^2}{[\hedgelr(\numexperts)]^2\Deltaeff^2}} \right\}}{\sqrt{2} \, [\hedgelr(\numexperts)]^2 \Deltaeff} \\
  	&\leq \sqrt{T} \bigg(\frac{3\hedgelr(\numexperts)}{4} + \frac{17}{16\hedgelr(\numexperts)} \log(\numeffexperts ) \bigg)
  	+ \frac{\sqrt{2} \, \log(\numexperts)}{\hedgelr(\numexperts)} + \frac{3\hedgelr(\numexperts)}{4} \\
  	& \qquad 
  	+ \frac{2\sqrt{2} \, [\log(\numexperts)]^2}{[\hedgelr(\numexperts)]^2\Deltaeff}
  	+ \frac{\log(\numexperts)}{\sqrt{2} \, \Deltaeff}
  	+ \frac{2\sqrt{2} \, \log(\numexperts)}{\hedgelr(\numexperts)\Deltaeff}
  	+ \frac{7}{\sqrt{2} \, [\hedgelr(\numexperts)]^2 \Deltaeff}.
\]

\ \\

\textbf{\OGHedge{} upper bound: $\numeffexperts = 1$}\\
If $\effexperts = \{\effexpidx\}$, we control the variance terms using \cref{lem:mix-var-bd}
\*[
	&\hspace{-2em} 
	\EEboth \sqrt{\VVar{\Expidx\sim \intweightvec(t+1)} \sbra{\bigg(\frac{\sqrt{t+1}}{\sqrt{t}} -1\bigg) \Loss_\Expidx(t-1) - \loss_\Expidx(t)} \vphantom{\left(\frac{\timefunc(t)}{\timefunc(t-1)}\right)}\VVar{\Expidx\sim \intweightvec(t+1)}\sbra{\loss_\Expidx(t)}}\\
  &\leq \frac{9}{4} \EEboth\Bigg[\sum_{\expidx\neq\effexpidx} \intweight_\expidx(t+1)\Bigg].
\]
  
We control this using \cref{lem:ftrl-both-weight-intweight-bound} with $p=1$, which gives 
\*[
	\EEboth \sbra{\sum_{\expidx\neq\effexpidx} \intweight_\expidx(t+1)}
	&\leq (\numexperts - 1) \exp\left\{2\hedgelr(\numexperts) +\frac{[\hedgelr(\numexperts)]^2}{2}\right\} \exp\left\{- \frac{\hedgelr(\numexperts)\Deltaeff}{\sqrt{2}}\sqrt{t} \right\}.
\]

Thus,
\[\label{eqn:variance-bound-hedge-summed2}
	&\hspace{-2em} \sup_{\policy \in \policyspace(\distnball{})} \EEboth \sum_{t=\tmid+1}^T \frac{\hedgelr(\numexperts)\sqrt{\VVar{\Expidx\sim \intweightvec(t+1)} \sbra{\bigg(\frac{\sqrt{t+1}}{\sqrt{t}} -1\bigg) \Loss_\Expidx(t-1) - \loss_\Expidx(t)}  \vphantom{\left(\frac{\timefunc(t)}{\timefunc(t-1)}\right)}\VVar{\Expidx\sim \intweightvec(t+1)}\sbra{\loss_\Expidx(t)}}}{\sqrt{t+1}} \\
	&\leq \frac{9\hedgelr(\numexperts)}{4}(\numexperts - 1) \exp\left\{2\hedgelr(\numexperts) +\frac{[\hedgelr(\numexperts)]^2}{2}\right\} \sum_{t=\tmid+1}^T \frac{1}{\sqrt{t}} \exp\left\{- \frac{\hedgelr(\numexperts)\Deltaeff}{\sqrt{2}}\sqrt{t} \right\} \\
	&\leq \frac{9}{\sqrt{2} \, \Deltaeff}(\numexperts - 1) \exp\left\{2\hedgelr(\numexperts) +\frac{[\hedgelr(\numexperts)]^2}{2}\right\} \exp\left\{- \frac{\hedgelr(\numexperts)\Deltaeff}{\sqrt{2}}\sqrt{\tmid} \right\},
\]
where the last step follows from again applying \cref{lem:integral-comparison-simplified}.
Combing \cref{eqn:hedge-regret-bound-split,eqn:entropy-bound-hedge-summed,eqn:variance-bound-hedge-summed2} gives that when $\numeffexperts = 1$,

\[\label{eqn:single-hedge-summed} 
	\sup_{\policy \in \policyspace(\distnball{})} \EEboth \playerexpregretHedge(T)
	&\leq \sqrt{\tmid+1} \bigg(\frac{\log\numexperts}{\hedgelr(\numexperts)} + \frac{3\hedgelr(\numexperts)}{4} \bigg) 
	+ \frac{17(\log 1) \Big[\sqrt{T+1} - \sqrt{\tmid+1} \Big] }{16\hedgelr(\numexperts)} \\
  	&\qquad + \frac{7(\numexperts - 1) \exp\left\{\frac{[\hedgelr(\numexperts)]^2}{8} - \frac{\hedgelr(\numexperts) \Deltaeff}{2\sqrt{2}}\sqrt{\tmid} \right\} }{\sqrt{2} \, [\hedgelr(\numexperts)]^2 \Deltaeff} \\
  	&\qquad + \frac{9(\numexperts - 1) \exp\left\{2\hedgelr(\numexperts) +\frac{[\hedgelr(\numexperts)]^2}{2} - \frac{\hedgelr(\numexperts) \Deltaeff}{\sqrt{2}}\sqrt{\tmid} \right\}}{\sqrt{2} \, \Deltaeff}.
\]

Substituting $\tmid$ into \cref{eqn:single-hedge-summed} gives
\[\label{eqn:single-hedge-summed-tmid} 
	&\hspace{-2em}\sup_{\policy \in \policyspace(\distnball{})} \EEboth \playerexpregretHedge(T)\\
	&\leq \sqrt{\frac{8 (\log\numexperts+[\hedgelr(\numexperts)]^2/4+\hedgelr(\numexperts))^2}{[\hedgelr(\numexperts)]^2\Deltaeff^2} + 2} \ \bigg(\frac{\log\numexperts}{\hedgelr(\numexperts)} + \frac{3\hedgelr(\numexperts)}{4} \bigg) \\
  	& \qquad + \frac{7\exp\left\{\frac{[\hedgelr(\numexperts)]^2}{8} - \frac{\hedgelr(\numexperts) \Deltaeff}{2\sqrt{2}}\sqrt{\frac{8 (\log\numexperts+[\hedgelr(\numexperts)]^2/4+\hedgelr(\numexperts))^2}{[\hedgelr(\numexperts)]^2\Deltaeff^2}} \right\}}{\sqrt{2} \, [\hedgelr(\numexperts)]^2 \Deltaeff}(\numexperts - 1)  \\
  	& \qquad + \frac{9(\numexperts - 1) \exp\left\{2\hedgelr(\numexperts) +\frac{[\hedgelr(\numexperts)]^2}{2} - \frac{\hedgelr(\numexperts)\Deltaeff}{\sqrt{2}}\sqrt{\frac{8 (\log\numexperts+[\hedgelr(\numexperts)]^2/4+\hedgelr(\numexperts))^2}{[\hedgelr(\numexperts)]^2\Deltaeff^2}} \right\}}{\sqrt{2} \, \Deltaeff} \\
  	&\leq
  	\frac{2\sqrt{2}(\log\numexperts)^2}{[\hedgelr(\numexperts)]^2\Deltaeff}
  	+ \frac{2\sqrt{2}\log\numexperts}{\hedgelr(\numexperts) \Deltaeff} 
  	+ \frac{4 \log\numexperts}{\sqrt{2} \, \Deltaeff} \\
  	& \qquad 
  	+ \frac{7/[\hedgelr(\numexperts)]^2 + 9 + 3[\hedgelr(\numexperts)]^2/4 + 3\hedgelr(\numexperts)}{\sqrt{2} \, \Deltaeff}
  	+ \sqrt{2} \Big(\frac{\log\numexperts}{\hedgelr(\numexperts)} + \frac{3\hedgelr(\numexperts)}{4} \Big).
\]
\manualendproof

\subsection{Details for \cref{thm:ftrl-care-bound-quant}}
This argument follows the same logical structure as the one for \cref{thm:hedge-bound-quant}.
Using that \FTRLCARE{} with parameters $c_1,c_2>0$ corresponds to $\entropyfunc(s) =\frac{\sqrt{s+c_2}}{c_1}$, \cref{thm:ftrl-simplex-summary} says that the weights $\careweightvec$ lead to \playerexpregretname{} bounded by
\[\label{eqn:care-regret-bound-1}
  \playerexpregretCare(T) 
    & \leq -\frac{\sqrt{(T+1)c_2}}{c_1} + \sum_{t=0}^T \frac{\sqrt{t+1}-\sqrt{t}}{c_1}\cdot \sqrt{\entropy(\careweightvec(t+1))+c_2} \\
    &\quad +\sum_{t=1}^T \frac{2c_1\sqrt{\entropy(\intweightvec(t+1))+c_2}}{\sqrt{t+1}} \\
		&\qquad\qquad\times \sqrt{{\VVar{\Expidx\sim \intweightvec(t+1)} \sbra{\rbra{\frac{\sqrt{t+1}}{\sqrt{t}} -1} \Loss_\Expidx(t-1) - \loss_\Expidx(t)}\vphantom{\frac{\timefunc(t)}{\timefunc(t-1)}}\VVar{\Expidx\sim \intweightvec(t+1)}\sbra{\loss_\Expidx(t)}}}\, ,
\]
where 
\*[
	\intweightvec(t+1) = \underset{\intweightvec\in \simp(\experts)}{\argmin} \rbra{\inner{\alpha_t \Lossvec(t) + (1-\alpha_t) \frac{\sqrt{t+1}}{\sqrt{t}}\Lossvec(t-1)}{\intweightvec}-\frac{\sqrt{t+1}}{c_1}\sqrt{\entropy(\intweightvec) + c_2}}
\]  
for some $\alpha_t\in[0,1]$. 
Then, recalling that $\entropyfunc'(s) = \frac{1}{2c_1\sqrt{s+c_2}}$, we can split up \cref{eqn:care-regret-bound-1} into the rounds before some $\tmid \in \Nats$ and the rounds after by applying \cref{lem:worst-case-ftrl-bound}. That is, when $T \leq \tmid$, we use the bound of \cref{lem:worst-case-ftrl-bound}, and if $T > \tmid$ we have

\[\label{eqn:care-regret-bound-split}
  \playerexpregretCare(T) 
    & \leq \sqrt{(\tmid+1)[\log\numexperts+c_2]}\Big(\frac{1}{c_1} + \frac{3c_1}{2} \bigg) -\frac{\sqrt{(T+1)c_2}}{c_1} \\
		&\quad + \sum_{t=\tmid}^T \frac{\sqrt{t+1}-\sqrt{t}}{c_1}\cdot \sqrt{\entropy(\careweightvec(t+1))+c_2} \\
    &\quad +\sum_{t=\tmid+1}^T \frac{2c_1\sqrt{\entropy(\intweightvec(t+1))+c_2}}{\sqrt{t+1}}\\
		&\qquad\qquad\times\sqrt{{\VVar{\Expidx\sim \intweightvec(t+1)} \sbra{\rbra{\frac{\sqrt{t+1}}{\sqrt{t}} -1} \Loss_\Expidx(t-1) - \loss_\Expidx(t)} \vphantom{\frac{\timefunc(t)}{\timefunc(t-1)}}\VVar{\Expidx\sim \intweightvec(t+1)}\sbra{\loss_\Expidx(t)}}}.
\]

Next, substituting $\entropyfunc$ and $\entropyfunc'$ for \FTRLCARE{} into \cref{eqn:entropy-bound-ftrl}, using Jensen's inequality with the concavity of square root, and the fact that $\sqrt{x+y} \leq \sqrt{x}+\sqrt{y}$ for all $x,y>0$ gives
\[\label{eqn:entropy-bound-care}
	&\hspace{-2em}
  \sup_{\policy \in \policyspace(\distnball{})} \EEboth \sqrt{\entropy(\careweightvec(t+1)) + c_2} \\
  &\leq \sqrt{\sup_{\policy \in \policyspace(\distnball{})} \EEboth \entropy(\careweightvec(t+1)) + c_2} \\ 
  &\leq \sqrt{\frac{17\log\numeffexperts}{16} + c_2} \
  + \frac{4\sqrt{(\numexperts - \numeffexperts)} \exp\left\{\frac{c_1^2 c_2}{4}- \frac{c_1 \sqrt{c_2} \, \Deltaeff\sqrt{t}}{2\sqrt{2}} \right\}}{3}.
\]
Thus,
\[\label{eqn:entropy-bound-care-summed}
  &\hspace{-2em}\sum_{t=\tmid+1}^T \frac{\sqrt{t+1}-\sqrt{t}}{c_1} \sup_{\policy \in \policyspace(\distnball{})} \EEboth \sqrt{\entropy(\careweightvec(t+1)) + c_2} \\
  &\leq \frac{\sqrt{{\frac{17}{16}} \log\numeffexperts + c_2} \Big[\sqrt{T+1} - \sqrt{\tmid+1} \Big]}{c_1} \\
  &\qquad + \sum_{t=\tmid+1}^T \frac{4\sqrt{(\numexperts - \numeffexperts)}\exp\left\{\frac{c_1^2 c_2}{4}- \frac{c_1 \sqrt{c_2} \, \Deltaeff}{2\sqrt{2}}\sqrt{t} \right\}}{3c_1\sqrt{t}}  \\
  &\leq \frac{\sqrt{{\frac{17}{16}} \log\numeffexperts + c_2} \Big[\sqrt{T+1} - \sqrt{\tmid+1} \Big]}{c_1} \\
  &\qquad + \frac{8\sqrt{2}\sqrt{(\numexperts - \numeffexperts)} \exp\left\{\frac{c_1^2 c_2}{4} - \frac{c_1 \sqrt{c_2} \, \Deltaeff}{2\sqrt{2}}\sqrt{\tmid} \right\}}{3c_1^2\sqrt{c_2} \, \Deltaeff},
\]
where the last step used \cref{lem:integral-comparison-simplified}.
Similarly, we use these same properties and \cref{eqn:intentropy-bound-ftrl} to obtain
\[\label{eqn:intentropy-bound-care}
  &\hspace{-2em}\sup_{\policy \in \policyspace(\distnball{})} \EEboth \sqrt{\entropy(\intweightvec(t+1)) + c_2}\\
  &\leq \sqrt{\sup_{\policy \in \policyspace(\distnball{})} \EEboth \entropy(\intweightvec(t+1)) + c_2} \\ 
  &\leq \sqrt{{\frac{17\log\numeffexperts}{16}}  + c_2} \
  + \frac{4\sqrt{(\numexperts - \numeffexperts)} \exp\left\{c_1\sqrt{c_2} + \frac{c_1^2 c_2}{4}- \frac{c_1 \sqrt{c_2} \, \Deltaeff\sqrt{t}}{2\sqrt{2}} \right\}}{3}.
\]

For the last term of \cref{eqn:care-regret-bound-split}, we consider the cases of $\numeffexperts>1$ and $\numeffexperts=1$ separately. 
For both, however, we will use $\tmid = \ceil{\frac{2 [\log\numexperts+3c_1\sqrt{c_2} + \frac{5}{4}c_1^2c_2]^2}{c_1^2c_2\Deltaeff^2}}$ and the constant $\careconst = \max\{c_2, 3c_1\sqrt{c_2} + \frac{5}{4}c_1^2c_2\}$.

\ \\

\textbf{\FTRLCARE{} upper bound: $\numeffexperts > 1$.}\\
If $\numeffexperts>1$, we again use \cref{lem:simple-var-bd} to control the variance terms.
Then, using \cref{eqn:intentropy-bound-care} and another application of \cref{lem:integral-comparison-simplified},
\[\label{eqn:variance-bound-care-summed}
	&\hspace{-2em}\sup_{\policy \in \policyspace(\distnball{})} \EEboth\sum_{t=\tmid+1}^T \frac{2c_1\sqrt{\entropy(\intweightvec(t+1))+c_2}}{\sqrt{t+1}}\\
	&\qquad \times\sqrt{{\VVar{\Expidx\sim \intweightvec(t+1)} \sbra{\rbra{\frac{\sqrt{t+1}}{\sqrt{t}} -1} \Loss_\Expidx(t-1) - \loss_\Expidx(t)} \vphantom{\left(\frac{\timefunc(t)}{\timefunc(t-1)}\right)}\VVar{\Expidx\sim \intweightvec(t+1)}\sbra{\loss_\Expidx(t)}}} \\
	&\leq \frac{3c_1}{4} \sum_{t=\tmid+1}^T 
		\sqrt{\frac{{\frac{17}{16} \log\numeffexperts + c_2}}{t+1} } \\
	&\qquad + c_1 \sqrt{(\numexperts - \numeffexperts)} \sum_{t=\tmid+1}^T \frac{\exp\left\{{c_1\sqrt{c_2}} + \frac{c_1^2 c_2}{4}- \frac{c_1 \sqrt{c_2} \, \Deltaeff\sqrt{t}}{2\sqrt{2}} \right\}}{\sqrt{t}}  \\
	&\leq \frac{3c_1 \sqrt{{\frac{17}{16} \log\numeffexperts + c_2}}\Big[\sqrt{T+1} - \sqrt{\tmid+1} \Big]}{2} \\
	&\qquad + \frac{8\sqrt{(\numexperts - \numeffexperts)}\exp\left\{{c_1\sqrt{c_2}} + \frac{c_1^2 c_2}{4} - \frac{c_1 \sqrt{c_2} \, \Deltaeff\sqrt{\tmid}}{2\sqrt{2}} \right\}}{\sqrt{2c_2} \, \Deltaeff} .
\]

Combining \cref{eqn:care-regret-bound-split,eqn:entropy-bound-care-summed,eqn:variance-bound-care-summed} 
gives that
\[\label{eqn:multiple-care-summed} 
	&\hspace{-2em}\sup_{\policy \in \policyspace(\distnball{})} \EEboth \playerexpregretCare(T) \\
	&\leq 
	\sqrt{(\tmid+1)[\log\numexperts+c_2]}
		\Big(\frac{1}{c_1} + \frac{3c_1}{2} \Big) 
	- \frac{\sqrt{(T+1)c_2}}{c_1} \\
	&\quad + \frac{\sqrt{{\frac{17}{16} \log\numeffexperts + c_2}} 
		\Big[\sqrt{T+1} - \sqrt{\tmid+1} \Big]}{c_1} \\
		&\quad + \frac{3c_1\sqrt{{\frac{17}{16} \log\numeffexperts + c_2}}
		\Big[\sqrt{T+1} - \sqrt{\tmid+1} \Big]}{2}  \\
    & \quad + \frac{16\sqrt{(\numexperts - \numeffexperts)} 
    	\exp\left\{\frac{c_1^2 c_2}{4} - \frac{c_1 \sqrt{c_2} \, \Deltaeff\sqrt{\tmid}}{2\sqrt{2}} \right\} }{3\sqrt{2}c_1^2\sqrt{c_2} \, \Deltaeff} \\
			&\quad + \frac{8\sqrt{(\numexperts - \numeffexperts)}\exp\left\{{c_1\sqrt{c_2}} + \frac{c_1^2 c_2}{4} - \frac{c_1 \sqrt{c_2} \, \Deltaeff\sqrt{\tmid}}{2\sqrt{2}} \right\}}{\sqrt{2c_2} \, \Deltaeff} 
    	 \\
    &\leq 
    \frac{33}{32}\sqrt{(T+1)[\log\numeffexperts+c_2]}
		\Big(\frac{1}{c_1} + \frac{3c_1}{2} \Big) \\
    & \quad + \sqrt{(\tmid+1)}
		\Big(\frac{1}{c_1} + \frac{3c_1}{2} \Big)
		\Big(\sqrt{\log\numexperts+c_2} - \sqrt{\log\numeffexperts+c_2} \Big) \\
		& \quad + \frac{\sqrt{2}\sqrt{(\numexperts - \numeffexperts)}\exp\left\{\frac{c_1^2 c_2}{4} - \frac{c_1 \sqrt{c_2} \, \Deltaeff}{2\sqrt{2}}\sqrt{\tmid} \right\}}{\sqrt{c_2} \, \Deltaeff}
		\Big(\frac{8}{3c_1^2} + 4\exp\left\{c_1\sqrt{c_2} \right\} \Big).
\]

Substituting $\tmid$ into \cref{eqn:multiple-care-summed} gives
\*[
	&\hspace{-2em}\sup_{\policy \in \policyspace(\distnball{})} \EEboth \playerexpregretCare(T) \\
	&\leq
	\Big(\frac{1}{c_1} + \frac{3c_1}{2} \Big)\Bigg[
	\frac{33}{32}\sqrt{(T+1)[\log\numeffexperts+c_2]}	\\
    &\qquad\qquad\qquad\ + \sqrt{\bigg(\frac{2 [\log\numexperts+3c_1\sqrt{c_2} + \frac{5}{4}c_1^2c_2]^2}{c_1^2c_2\Deltaeff^2}+2\bigg)[\log\numexperts+c_2]} \ \Bigg] \\
	& \qquad + \frac{\sqrt{2}\sqrt{(\numexperts - \numeffexperts)}}{\sqrt{c_2} \, \Deltaeff} \Big(\frac{8}{3c_1^2} + 4\exp\left\{c_1\sqrt{c_2} \right\} \Big)\\
		&\qquad\qquad\qquad\times \exp\left\{\frac{c_1^2 c_2}{4} - \frac{c_1 \sqrt{c_2} \, \Deltaeff}{2\sqrt{2}}\sqrt{\frac{2 [\log\numexperts+3c_1\sqrt{c_2} + \frac{5}{4}c_1^2c_2]^2}{c_1^2c_2\Deltaeff^2}}\right\}
		 \\
	&\leq 
	\Big(\frac{1}{c_1} + \frac{3c_1}{2} \Big)\Bigg[
	\frac{33}{32}\sqrt{(T+1)[\log\numeffexperts+c_2]}\\
	&\qquad\qquad\qquad\ + \frac{\sqrt{2} \, [\log\numexperts+\careconst]^{3/2}}{c_1 \sqrt{c_2} \, \Deltaeff}	
	+ \sqrt{2[\log\numexperts+c_2]} \Bigg] \\
	&\qquad+ \frac{\sqrt{2} (8+12c_1^2)}{3c_1^2\sqrt{c_2} \, \Deltaeff}.
\]

\ \\

\textbf{\FTRLCARE{} upper bound: $\numeffexperts = 1$}\\
If $\effexperts = \{\effexpidx\}$, we control the variance terms using \cref{lem:mix-var-bd} In particular,
\*[
	&\hspace{-2em} 
	\EEboth \Bigg[\VVar{\Expidx\sim \intweightvec(t+1)} \sbra{\bigg(\frac{\sqrt{t+1}}{\sqrt{t}} -1\bigg) \Loss_\Expidx(t-1) - \loss_\Expidx(t)} \vphantom{\left(\frac{\timefunc(t)}{\timefunc(t-1)}\right)}\VVar{\Expidx\sim \intweightvec(t+1)}\sbra{\loss_\Expidx(t)}\Bigg]\\
  &\leq \frac{27}{32} \EEboth\Bigg[\sum_{\expidx\neq\effexpidx} \intweight_\expidx(t+1)\Bigg].
\]
We control this using \cref{lem:ftrl-both-weight-intweight-bound} with $p=1$, which gives
\*[
	\sup_{\policy \in \policyspace(\distnball{})} \EEboth \sbra{\sum_{\expidx\neq\effexpidx} \intweight_\expidx(t+1)}
  	&\leq (\numexperts - 1) \exp\Big\{4c_1\sqrt{c_2} +2c_1^2c_2- \sqrt{2} \, c_1\sqrt{c_2}\Deltaeff \sqrt{t} \Big\}.
\]
Thus, using Cauchy-Schwarz and \cref{eqn:intentropy-bound-care} (recalling $\log\numeffexperts=0$), for any $\policy \in \policyspace(\distnball{})$
\*[ 
	&\hspace{-2em}
	\EEboth 
	\sqrt{\rbra{\entropy(\intweightvec(t+1)) + c_2} {\VVar{\Expidx\sim \intweightvec(t+1)} \sbra{\rbra{\frac{\sqrt{t+1}}{\sqrt{t}} -1} \Loss_\Expidx(t-1) - \loss_\Expidx(t)} \vphantom{\left(\frac{\timefunc(t)}{\timefunc(t-1)}\right)}\VVar{\Expidx\sim \intweightvec(t+1)}\sbra{\loss_\Expidx(t)}}} \\
	&\leq \sqrt{\EEboth \entropy(\intweightvec(t+1))+c_2} \\ 
  &\qquad\qquad\times\sqrt{\EEboth {\bigg[\VVar{\Expidx\sim \intweightvec(t+1)} \sbra{\rbra{\frac{\sqrt{t+1}}{\sqrt{t}} -1} \Loss_\Expidx(t-1) - \loss_\Expidx(t)}
	\VVar{\Expidx\sim \intweightvec(t+1)}\sbra{\loss_\Expidx(t)}\bigg]}} \\
	&\leq \sqrt{\frac{7(\numexperts - 1) \exp\left\{{2c_1\sqrt{c_2}} + \frac{c_1^2 c_2}{2}- \frac{c_1 \sqrt{c_2} \, \Deltaeff\sqrt{t}}{\sqrt{2}} \right\}}{4} + c_2} \\ 
  &\qquad\qquad\times\sqrt{(\numexperts - 1)}\exp\cbra{{2c_1\sqrt{c_2}} + c_1^2 c_2- \frac{c_1 \sqrt{c_2} \, \Deltaeff\sqrt{t}}{\sqrt{2}}} \\
	&\leq \frac{3(\numexperts - 1)\exp\cbra{{3c_1\sqrt{c_2}} + \frac{5c_1^2 c_2}{4}- \frac{3c_1 \sqrt{c_2} \, \Deltaeff}{2\sqrt{2}}\sqrt{t} }}{2}	\\ 
  &\qquad\qquad+ \sqrt{c_2(\numexperts - 1)}\exp\cbra{{2c_1\sqrt{c_2}} + {c_1^2 c_2} - \frac{c_1 \sqrt{c_2} \, \Deltaeff\sqrt{t}}{\sqrt{2}}} 
		 \\
	&\leq (3/2+\sqrt{c_2})(\numexperts - 1)
		\exp\cbra{3c_1\sqrt{c_2} + \frac{5c_1^2 c_2}{4} - \frac{c_1 \sqrt{c_2} \, \Deltaeff}{\sqrt{2}}\sqrt{t}}.
\]

Summing this over $t$ and applying \cref{lem:integral-comparison-simplified} gives
\[\label{eqn:variance-bound-care-summed2}
	&\hspace{-2em}\sup_{\policy \in \policyspace(\distnball{})} \EEboth\sum_{t=\tmid+1}^T \frac{2c_1\sqrt{\entropy(\intweightvec(t+1))+c_2}}{\sqrt{t+1}}\\
	&\times\sqrt{{\VVar{\Expidx\sim \intweightvec(t+1)} \sbra{\rbra{\frac{\sqrt{t+1}}{\sqrt{t}} -1} \Loss_\Expidx(t-1) - \loss_\Expidx(t)}} \vphantom{\left(\frac{\timefunc(t)}{\timefunc(t-1)}\right)}\VVar{\Expidx\sim \intweightvec(t+1)}\sbra{\loss_\Expidx(t)}} \\
	&\leq c_1(3+2\sqrt{c_2})(\numexperts - 1) 
		\exp\Big\{3c_1\sqrt{c_2} + \frac{5c_1^2 c_2}{4}\Big\} \\
		&\qquad\qquad\times\sum_{t=\tmid+1}^T \frac{1}{\sqrt{t}}
		\exp\left\{- \frac{c_1 \sqrt{c_2} \, \Deltaeff}{\sqrt{2}}\sqrt{t} \right\} \\
	&\leq \frac{\sqrt{2} (3+2\sqrt{c_2})(\numexperts - 1)}{\sqrt{c_2} \, \Deltaeff} 
		\exp\bigg\{3c_1\sqrt{c_2} + \frac{5c_1^2 c_2}{4} - \frac{c_1 \sqrt{c_2} \, \Deltaeff}{\sqrt{2}}\sqrt{\tmid} \bigg\}.
\]
Combining \cref{eqn:care-regret-bound-split,eqn:entropy-bound-care-summed,eqn:variance-bound-care-summed2} gives that for $\numeffexperts = 1$,
\[\label{eqn:single-care-summed} 
	&\hspace{-2em}\sup_{\policy \in \policyspace(\distnball{})} \EEboth \playerexpregretCare(T) \\
	&\leq 
	\sqrt{(\tmid+1)[\log\numexperts+c_2]}
		\Big(\frac{1}{c_1} + \frac{3c_1}{2} \Big) 
	- \frac{\sqrt{(T+1)c_2}}{c_1} \\ 
	& \quad + \frac{8\sqrt{2(\numexperts - 1)}\exp\left\{\frac{c_1^2 c_2}{4} - \frac{c_1 \sqrt{c_2} \, \Deltaeff}{2\sqrt{2}}\sqrt{\tmid} \right\}}{3c_1^2\sqrt{c_2} \, \Deltaeff} \\
			&\quad + \frac{\sqrt{ c_2}}{c_1} 
		\Big[\sqrt{T+1} - \sqrt{\tmid+1} \Big] \\
  & \quad + \frac{\sqrt{2} (3+2\sqrt{c_2})(\numexperts - 1)\exp\bigg\{3c_1\sqrt{c_2} + \frac{5c_1^2 c_2}{4} - \frac{c_1 \sqrt{c_2} \, \Deltaeff}{\sqrt{2}}\sqrt{\tmid} \bigg\}}{\sqrt{c_2} \, \Deltaeff} 
		 \\
	&\leq 
    \sqrt{(\tmid+1)[\log\numexperts+c_2]}
		\Big(\frac{1}{c_1} + \frac{3c_1}{2} \Big)\\
	& \quad+ \frac{8\sqrt{2(\numexperts - 1)}\exp\left\{\frac{c_1^2 c_2}{4} - \frac{c_1 \sqrt{c_2} \, \Deltaeff\sqrt{\tmid}}{2\sqrt{2}} \right\}}{3c_1^2\sqrt{c_2} \, \Deltaeff}
    	 \\
    & \quad + \frac{\sqrt{2} (3+2\sqrt{c_2})(\numexperts - 1)}{\sqrt{c_2} \, \Deltaeff} 
		\exp\bigg\{3c_1\sqrt{c_2} + \frac{5c_1^2 c_2}{4} - \frac{c_1 \sqrt{c_2} \, \Deltaeff}{\sqrt{2}}\sqrt{\tmid} \bigg\}.
\]

Substituting $\tmid$ into \cref{eqn:single-care-summed} gives
\*[
	&\hspace{-2em}\sup_{\policy \in \policyspace(\distnball{})} \EEboth \playerexpregretCare(T) \\
	&\leq 
	\Big(\frac{1}{c_1} + \frac{3c_1}{2} \Big)
    \sqrt{\bigg(\frac{2 [\log\numexperts+3c_1\sqrt{c_2} + \frac{5}{4}c_1^2c_2]^2}{c_1^2c_2\Deltaeff^2}+2\bigg)[\log\numexperts+c_2]} \\
	& \qquad + \frac{8\sqrt{2(\numexperts - 1)}}{3c_1^2\sqrt{c_2} \, \Deltaeff}
    	\ \exp\left\{\frac{c_1^2 c_2}{4} - \frac{c_1 \sqrt{c_2} \, \Deltaeff}{2\sqrt{2}}\sqrt{\frac{2 [\log\numexperts+3c_1\sqrt{c_2} + \frac{5}{4}c_1^2c_2]^2}{c_1^2c_2\Deltaeff^2}} \right\} \\
    & \qquad + \frac{\sqrt{2} (3+2\sqrt{c_2})(\numexperts - 1)}{\sqrt{c_2} \, \Deltaeff} \\
		&\qquad\qquad\times \exp\left\{3c_1\sqrt{c_2} + \frac{5c_1^2 c_2}{4} - \frac{c_1 \sqrt{c_2} \, \Deltaeff}{\sqrt{2}}\sqrt{\frac{2 [\log\numexperts+3c_1\sqrt{c_2} + \frac{5}{4}c_1^2c_2]^2}{c_1^2c_2\Deltaeff^2}}\right\} \\
	&\leq 
	\Big(\frac{1}{c_1} + \frac{3c_1}{2} \Big)\Bigg[
	\frac{\sqrt{2} \, [\log\numexperts+\careconst]^{3/2}}{c_1 \sqrt{c_2} \, \Deltaeff}	
	+ \sqrt{2[\log\numexperts+c_2]} \Bigg] \\
	&\qquad + \frac{1}{\sqrt{c_2} \, \Deltaeff}\bigg[\frac{8\sqrt{2}}{3c_1^2} + \sqrt{2} (3+2\sqrt{c_2}) \bigg].
\]
\manualendproof	
\section{Generic \FTRL{} \regretname{} bounds with local norms}
\label{sec:olo}

\subsection{Online linear optimization with \FTRL{}}
\label{ssec:generic-ftrl}
An \emph{online linear optimization} (OLO) problem in $\Reals^\olodim$ is defined by a closed \emph{prediction domain} $\olodomain \subseteq \Reals^\olodim$ and a \emph{loss domain} $\ololossdomain \subseteq \Reals^\olodim$.
At each time $t$, the \playername{} selects $\oloplayervec(t)\in \olodomain$, then observes some $\ololossvec(t) \in \ololossdomain$ and incurs the loss $\inner{\ololossvec(t)}{\oloplayervec(t)}$.
For any sequence of losses $\ololossvec(1),\dots,\ololossvec(T) \in \ololossdomain$, the \playername{}'s \oloplayerregretname{} is defined by
\*[
	\oloplayerregret(T) = \sum_{t=1}^T \inner{\ololossvec(t)}{\oloplayervec(t)} - \inf_{\oloplayervec\in \olodomain} \sum_{t=1}^T \inner{\ololossvec(t)}{\oloplayervec}.
\]
There are many ways one could choose $\oloplayervec(t)$, but in this work we focus specifically on \FTRL, which is a generic method for online linear optimization. The \FTRL{} algorithm is parametrized by $\olodomain$, $\ololossdomain$, and a sequence of \emph{regularizers} $\set{\oloregularizer{t}:\olodomain \to \Reals}_{t\in\PosInts}$. For each time $t+1$, a \playername{} using the \FTRLalg{$\olodomain$}{$\ololossdomain$}{$(\oloregularizer{t})_{t\in\PosInts}$} algorithm outputs
\[
	\oloplayervec(t+1) \in \argmin_{\oloplayervec\in \olodomain}\rbra{ \inner{\oloLossvec(t)}{\oloplayervec}+ \olocumregularizer{t}(\oloplayervec) },
\label{eqn:oloFTRL}
\]
where $\olocumregularizer{t}(\oloplayervec) = \sum_{s=0}^t \oloregularizer{s}(\oloplayervec)$ and $\oloLossvec(t) = \sum_{s=1}^t \ololossvec(s)$.

\subsection{OLO \FTRL{} \regretname{} bounds}

The classical \regretname{} bound for \FTRL{} consists of a term that is the difference of losses incurred by consecutive player vectors and a term that looks like the regularizer evaluated at the optimal player vector in hindsight. The former is usually bounded using strong-convexity 
to obtain a norm of the consecutive weight differences. For tighter control, such as that obtained by \citet{Abernethy2009BeatingTA}, this norm may be chosen to be a \emph{local norm}.
A local norm with respect to a function $f$ will be of the form $\norm{x}_y = \sqrt{\inner{x}{\hess f(y) x}}$, and has the property that the dual is $\norm{x}_{y,\star} = \sqrt{\inner{x}{(\hess f(y))^{-1} x}}$. The natural choice of function to define the local norm with respect to is the regularizer; however, this is generally more challenging for non-constant regularizers.

Surprisingly, while both local norms and time-dependent regularizers are standard in the \FTRL{} literature, we were unable to find an explicit statement that combines them exactly as we needed. The closest seems to be Theorem~1 of \citet{mcmahan2017survey}, which requires that the regularizers are strongly convex with respect to a norm and then defines the local norm using the time-dependent strong convexity parameter. 
This strong-convexity argument is insufficient for our analysis, as the \CARE{} regularizer can be at worst only $1/\sqrt{\log\numexperts}$-strongly convex in all settings, and consequently would not lead to the adaptive rates we obtain.
We begin with a modification of \cite[Lemma~1]{mcmahan2010adaptive} to combine local norm bounds with time-dependent regularizer bounds.

\begin{lemma}
\label{lem:adaptive-local-ftrl}
For any $\olodomain$, $\ololossdomain$, $(\oloregularizer{t})_{t\in\PosInts}$, and $(\ololossvec(t))_{t\in\Nats}\subseteq \ololossdomain$,
the \FTRLalg{$\olodomain$}{$\ololossdomain$}{$(\oloregularizer{t})_{t\in\PosInts}$} algorithm has \oloplayerregretname{} bounded for all $T\in\Nats$ by
\*[
	\oloplayerregret(T)
	  & \leq \olocumregularizer{T}(\oloplayervecopt(T))
	  	- \sum_{t=0}^T \oloregularizer{t}(\oloplayervec(t+1))
	    + \sum_{t=1}^T \inner{\ololossvec(t)}{\oloplayervec(t) - \oloplayervec(t+1)},
\]
for all $\oloplayervecopt(T) \in \argmin_{\oloplayervec\in \olodomain} \inner{\oloLossvec(T)}{\oloplayervec}$.
\end{lemma}
\begin{proof}[Proof of \cref{lem:adaptive-local-ftrl}]
This follows from directly modifying the proof of \cite[Lemma~1]{mcmahan2010adaptive} by not dropping the $\oloregularizer{t}(\oloplayervec(t+1))$ term at the end of \cite[Lemma~7]{mcmahan2010adaptive}. We reproduce the argument here for completeness.

As shown by \citet{kalai05}, and restated in \cite[Lemma~6]{mcmahan2010adaptive},
\*[
	\sum_{t=0}^T f_t(x_\star(t)) \leq \sum_{t=0}^T f_t(x_\star(T))
\]
for any sequence of functions $(f_t)_{t\in\PosInts}$ and any sequence $x_\star(t) \in \argmin_x \sum_{s=0}^t f_s(x)$. Thus, by definition of $\oloplayervec(t+1)$ minimizing \cref{eqn:oloFTRL},
\*[
	\sum_{t=0}^T \left[\inner{\ololossvec(t)}{\oloplayervec(t+1)} + \oloregularizer{t}(\oloplayervec(t+1)) \right]
	&\leq \sum_{t=0}^T \left[\inner{\ololossvec(t)}{\oloplayervec(T+1)} + \oloregularizer{t}(\oloplayervec(T+1)) \right] \\
	&\leq \sum_{t=0}^T \left[\inner{\ololossvec(t)}{\oloplayervecopt(T)} + \oloregularizer{t}(\oloplayervecopt(T)) \right] \\
	&= \inner{\oloLossvec(T)}{\oloplayervecopt(T)} + \olocumregularizer{T}(\oloplayervecopt(T)).
\]
Rearranging gives that
\*[
	\oloplayerregret(T)
	&= \sum_{t=0}^T \inner{\ololossvec(t)}{\oloplayervec(t)} - \inner{\oloLossvec(T)}{\oloplayervecopt(T)} \\
	&= \sum_{t=0}^T \inner{\ololossvec(t)}{\oloplayervec(t) - \oloplayervec(t+1)} + \sum_{t=0}^T \inner{\ololossvec(t)}{\oloplayervec(t+1)} - \inner{\oloLossvec(T)}{\oloplayervecopt(T)} \\
	&\leq \sum_{t=0}^T \inner{\ololossvec(t)}{\oloplayervec(t) - \oloplayervec(t+1)} + \olocumregularizer{T}(\oloplayervecopt(T)) - \sum_{t=0}^T \oloregularizer{t}(\oloplayervec(t+1)).
\]
Finally, the indexing of $t$ in the sums of the lemma statement follows since by convention $\ololossvec(0) = 0$.
\end{proof}

An alternative to the \regretname{} expansion for \FTRL{} from \citet{mcmahan2010adaptive} has appeared in more recent literature such as that of \citet{duchi2011,shalev2012,hazan2016,orabona2019}. This alternative analysis can be tighter in certain cases, but requires controlling three terms instead of two. Additionally, it could only lead to improvements in the constants in our case (bounded losses), so we opted for the simpler approach.

\subsection{OLO \FTRL{} \regretname{} bounds with local norms}
Now, we provide a local-norm control on the inner product from \cref{lem:adaptive-local-ftrl} for time-dependent regularizers which can be defined as a function of time and a constant regularizer. The types of regularizers we will consider are \emph{convex functions of the Legendre type}, as defined by \cite[Sec.~26]{rockafellar1970convex}.

\begin{definition}[Essentially smooth, \aos{\citet{rockafellar1970convex}}\preprint{\citet{rockafellar1970convex}}, Section~26] An extended-real-valued function $f:F\to\ExtReals$ for $F\subseteq\Reals^\olodim$ is \emph{essentially smooth} on $F$ if it satisfies
	\begin{enumerate}
		\item $\interior(F)\neq \emptyset$,
		\item $f$ is differentiable on $\interior(F)$, and
		\item $x\in\boundary(F)$ and $\set{y_i}_{i\in\Nats} \subseteq\interior(F)$ with $y_i \to x$ implies $\norm{\grad f(y_i)} \to +\infty$.
	\end{enumerate}
\end{definition}

\begin{definition}[Legendre type, \aos{\citet{rockafellar1970convex}}\preprint{\citet{rockafellar1970convex}}, Section~26]
\label{defn:legendre-type}
	A closed convex function $f: F\to \Reals$ for $F\subseteq\Reals^\olodim$ is \emph{of the Legendre type} on $F$ if
	\begin{enumerate}
		\item $f$ is strictly convex on $\interior(F)$,
		\item $\interior(F)$ is convex, and
		\item $f$ is essentially smooth on $F$.
	\end{enumerate}
\end{definition}

\begin{definition}[Legendre Transform, \aos{\citet{rockafellar1970convex}}\preprint{\citet{rockafellar1970convex}}, Section~26]
\label{defn:legendre-transform}
	The \emph{Legendre transform} of a function $f:F\to \Reals$ for $F\subseteq\Reals^\olodim$  of the Legendre type on $F$ is the function $f^*: \grad f(\interior(F))\to \Reals$ defined by
	\*[
		f^*(y) = \sup_{x\in F} [\inner{x}{y}- f(x)] = \inner{[\grad f]^{-1}(y)}{y} - f([\grad f]^{-1}(y)).
	\]
\end{definition}

\begin{proposition}[\aos{\citet{rockafellar1970convex}}\preprint{\citet{rockafellar1970convex}}, Theorem~26.5]
	\label{prop:legendre-rocky}
	If $f$ is a closed convex function of the Legendre type on $F$ for $F\subseteq\Reals^\olodim$ and $F^* = \grad f(\interior(F))$,
	then $F^*$ is convex and $f^*$ is of the Legendre type on $F^*$,
	\*[
		\grad f :\interior(F)\to F^\star
	\]
	is a continuous bijection with continuous inverse, and
	$\grad[f^*] = [\grad f]^{-1}$.
\end{proposition}

\begin{corollary}
	\label{cor:legendre-interior-solution}
	If $F\subseteq \Reals^\olodim$ is convex with non-empty interior, and if $f$ is a closed, convex function of the Legendre type on $F$, then for any $y$ with $-y \in \grad f(\interior(F))$,
	\*[
		\argmin_{x\in F}\rbra{\inner{y}{x} + f(x)}
			& = \set{ [\grad f]^{-1}(-y)} = \set{ [\grad[f^*]] (-y)} \in \interior(F).
	\]
\end{corollary}
\begin{proof}
	Since the objective is convex then if a single local minimum occurs in the interior $F$ then it must be the unique optimizer on $F$. Taking the gradient of the objective, we see that a local minimum occurs when $\grad f(x) = -y$. Since $f$ is assumed to be of the Legendre type on $F$ then this equation has a unique solution in $\interior(F)$ whenever $-y \in \grad f(\interior(F))$.
\end{proof}

\begin{lemma}
\label{lem:inner-bound-ftrl-olo}
Suppose that $\olodomain\subseteq\Reals^\olodim$ is convex with non-empty interior, $\ololossdomain \subseteq \Reals^\olodim$ is arbitrary, and the regularizer $\oloregularizer{0}$ is closed, convex, of the Legendre type on $\olodomain$, and twice continuously differentiable on $\interior(\olodomain)$.
For each $t\in \Nats$, let $\olocumregularizer{t}(\oloplayervec) = \timefunc(t)\oloregularizer{0}(\oloplayervec)$ for some increasing function $\timefunc: \Nats \to \PosReals$.
Also, for any $\ololossvecgen \in \ololossdomain$ and $\oloplayervecgen \in \olodomain$, define the time-dependent local norm by $\norm{\ololossvecgen}_{t,\oloplayervecgen}^2 = \inner{\ololossvecgen}{\hess\olocumregularizer{t}(\oloplayervecgen)\ololossvecgen}$, and its dual time-dependent local norm by $\norm{\ololossvecgen}_{t,\oloplayervecgen, \star}^2 = \inner{\ololossvecgen}{[\hess\olocumregularizer{t}(\oloplayervecgen)]^{-1}\ololossvecgen}$.
Then, for any sequence of losses $(\ololossvec(t))_{t\in \Nats} \subseteq \ololossdomain$ such that $(-\frac{1}{\beta(t)}\oloLossvec(t))\in [\grad \oloregularizer{0}](\interior(\olodomain))$ for all $t\in\Nats$,
there exists a sequence $(\alpha_t)_{t\in\Nats} \subseteq [0,1]$
such that, for all $t\in\Nats$, the weights $(\oloplayervec(t))_{t\in\Nats}$ output by the \FTRLalg{$\olodomain$}{$\ololossdomain$}{$(\oloregularizer{t})_{t\in\PosInts}$} algorithm satisfy
\*[
	\inner{\ololossvec(t)}{\oloplayervec(t)-\oloplayervec(t+1)}
	\leq \frac{1}{\timefunc(t)} \norm{\left(\frac{\timefunc(t)}{\timefunc(t-1)} -1 \right)\oloLossvec(t-1) - \ololossvec(t)}_{0,\intweightvec(t+1),\star} \norm{\ololossvec(t)}_{0,\intweightvec(t+1),\star},
\]
where $\intweightvec(t+1) = \argmin_{\intweightvec\in \olodomain} \rbra{\inner{\alpha_t \oloLossvec(t) + (1-\alpha_t) \frac{\beta(t)}{\beta(t-1)}\oloLossvec(t-1)}{\intweightvec}+\olocumregularizer{t}(\intweightvec)}$.
\end{lemma}
\begin{remark}
	In our applications, $[\grad \oloregularizer{0}](\interior(\olodomain))=\Reals^\olodim$ is the whole space, so the assumption
	\*[
		(-\textstyle{\frac{1}{\beta(t)}}\oloLossvec(t))\in [\grad \oloregularizer{0}](\interior(\olodomain))
	\]
	is benign.
\end{remark}
\begin{proof}[Proof of \cref{lem:inner-bound-ftrl-olo}]
Fix some $t\in\Nats$ and observe that by \cref{cor:legendre-interior-solution}, $\oloplayervec(t+1)$ is the unique $\oloplayervec$ that solves $\grad \olocumregularizer{t}(\oloplayervec) = - \oloLossvec(t)$. Thus, applying a first-order Taylor expansion of $[\grad \olocumregularizer{t}]^{-1}$ centered at $\grad \olocumregularizer{t}(\oloplayervec(t))$,
\*[
	\oloplayervec(t+1) - \oloplayervec(t)
	&= [\grad \olocumregularizer{t}]^{-1}(\grad\olocumregularizer{t}(\oloplayervec(t+1)))
		- [\grad \olocumregularizer{t}]^{-1}(\grad\olocumregularizer{t}(\oloplayervec(t))) \\
	&= [J [\grad \olocumregularizer{t}]^{-1}](-\oloLossdumvec(t))
		\sbra{\grad\olocumregularizer{t}(\oloplayervec(t+1)) - \grad\olocumregularizer{t}(\oloplayervec(t))},
\]
where $J$ denotes the Jacobian and $-\oloLossdumvec(t) = \alpha_t \grad\olocumregularizer{t}(\oloplayervec(t+1)) + (1-\alpha_t) \grad\olocumregularizer{t}(\oloplayervec(t))$ for some $\alpha_t \in \cinter{0,1}$.
Using the inverse function theorem on $\grad \olocumregularizer{t}$ gives \*[
	[J [\grad \olocumregularizer{t}]^{-1}](-\oloLossdumvec(t)) = [\hess \olocumregularizer{t}([\grad \olocumregularizer{t}]^{-1}(-\oloLossdumvec(t)))]^{-1}.
\]
Next, observe that
\*[
	\grad \olocumregularizer{t}(\oloplayervec(t))
	= \timefunc(t) \grad \oloregularizer{0}(\oloplayervec(t))
	= \frac{\timefunc(t)}{\timefunc(t-1)} \grad \olocumregularizer{t-1}(\oloplayervec(t))
	= \frac{\timefunc(t)}{\timefunc(t-1)} (-\oloLossvec(t-1)),
\]
so $\oloLossdumvec(t)$ can be viewed as a combination of losses defined by
\*[
	\oloLossdumvec(t) = \alpha_t \oloLossvec(t) + (1-\alpha_t)\frac{\timefunc(t)}{\timefunc(t-1)} \oloLossvec(t-1).
\]
Therefore, $-\frac{\oloLossdumvec(t)}{\timefunc(t)} = \alpha_t\frac{-\oloLossvec(t)}{\timefunc(t)} + (1-\alpha_t)\frac{-\oloLossvec(t-1)}{\timefunc(t-1)} \in  \grad \oloregularizer{0}(\interior(\olodomain))$ since $\grad \oloregularizer{0}(\interior(\olodomain))$ is convex (by \cref{prop:legendre-rocky}).
This implies
\*[-\oloLossdumvec(t) \in [\timefunc(t)  \grad \oloregularizer{0}](\interior(\olodomain)) =  \grad \olocumregularizer{t}(\interior(\olodomain)),
\]
so $\intweightvec(t+1) = [\grad \olocumregularizer{t}]^{-1}(-\oloLossdumvec(t)) \in \interior(\olodomain)$ by \cref{cor:legendre-interior-solution}. Further,
\*[
	\grad\olocumregularizer{t}(\oloplayervec(t+1)) - \grad\olocumregularizer{t}(\oloplayervec(t))
	= - \oloLossvec(t) + \frac{\timefunc(t)}{\timefunc(t-1)} \oloLossvec(t-1)
	= \left(\frac{\timefunc(t)}{\timefunc(t-1)} -1 \right)\oloLossvec(t-1) - \ololossvec(t).
\]

Combining these results, along with the fact that $\hess \olocumregularizer{t} = \timefunc(t) \hess \oloregularizer{0}$, gives
\[
	\oloplayervec(t+1) - \oloplayervec(t)
	= \frac{1}{\timefunc(t)} [\hess \oloregularizer{0}(\intweightvec(t+1))]^{-1}\left[\left(\frac{\timefunc(t)}{\timefunc(t-1)} -1 \right)\oloLossvec(t-1) - \ololossvec(t) \right].
\label{eqn:weight_diff}
\]
Next, by Holder's inequality,
\*[
	\inner{\ololossvec(t)}{\oloplayervec(t)-\oloplayervec(t+1)}
	& \leq \norm{\oloplayervec(t)-\oloplayervec(t+1)}_{t,\intweightvec(t+1)} \norm{\ololossvec(t)}_{t,\intweightvec(t+1),\star} \\
	& = \norm{\oloplayervec(t)-\oloplayervec(t+1)}_{0,\intweightvec(t+1)} \norm{\ololossvec(t)}_{0,\intweightvec(t+1),\star},
\]
where the last equality follows from the fact that a $\timefunc(t)$ will factor out of the first norm and a $1/\timefunc(t)$ will factor out of the second norm. Then, substituting in \cref{eqn:weight_diff},
\*[
	&\norm{\oloplayervec(t)-\oloplayervec(t+1)}^2_{0,\intweightvec(t+1)} \\
	&= \inner{\hess \oloregularizer{0}(\intweightvec(t+1)) \sbra{\oloplayervec(t)-\oloplayervec(t+1)}}{\oloplayervec(t)-\oloplayervec(t+1)} \\
	&= \frac{1}{\timefunc(t)^2}
	\inner{[\hess \oloregularizer{0}(\intweightvec(t+1))]^{-1} \Big[\Big(\tfrac{\timefunc(t)}{\timefunc(t-1)} -1 \Big)\oloLossvec(t-1) - \ololossvec(t) \Big]}{\Big(\tfrac{\timefunc(t)}{\timefunc(t-1)} -1 \Big)\oloLossvec(t-1) - \ololossvec(t)} \\
	&= \frac{1}{\timefunc(t)^2} \norm{\left(\frac{\timefunc(t)}{\timefunc(t-1)} -1 \right)\oloLossvec(t-1) - \ololossvec(t)}^2_{0,\intweightvec(t+1),\star}.
\]
Thus,
\*[
	\inner{\ololossvec(t)}{\oloplayervec(t)-\oloplayervec(t+1)}
	\leq \frac{1}{\timefunc(t)} \norm{\left(\frac{\timefunc(t)}{\timefunc(t-1)} -1 \right)\oloLossvec(t-1) - \ololossvec(t)}_{0,\intweightvec(t+1),\star} \norm{\ololossvec(t)}_{0,\intweightvec(t+1),\star}.
\]
\end{proof}

\citet{amir2020prediction} recently made the same observation that closely related bounds have been derived before but not in the explicit form they desire, and they prove a \regretname{} bound very similar to \cref{lem:adaptive-local-ftrl,lem:inner-bound-ftrl-olo}. However, they rely on a Taylor expansion of the regularizer around the weights output by \FTRL, while we have used a Taylor expansion of the Legendre dual of the regularizer around the observed losses. This makes it easier for us to ultimately apply \cref{thm:minimax_mgf} when controlling the bound of \cref{lem:inner-bound-ftrl-olo} in expectation. 
\citet{zimmert19} have a similar expansion in their analysis, and obtain a local norm in the dual space as an intermediate step in the proof of their Lemma~11. However, the object they use this local norm to upper bound is not the same as what we upper bound, and they ultimately use a bound in the primal space to obtain their results. 
\section{\FTRL{} \regretname{} bounds on the simplex}
\label{sec:ftrl-simp-proof}

When we restrict consideration to \properpredpolicynames{} (see \cref{sec:minimax-lower-bounds}) and focus on controlling the \expregretname{}, then online linear optimization is a generalization of the online prediction problem in \cref{sec:framework}, which is just the case where $\olodomain = \simp(\experts)$, $\ololossdomain = [0,1]^{\numexperts}$, and we are interested in $\EE \, \oloplayerregret(T)$.
To bound the \expregretname{}, we choose an appropriate sequence of regularizers and then apply generic techniques for analyzing \FTRL{} in online linear optimization problems.
For clarity and to distinguish between \FTRL{} in the generic online linear optimization setting and in the specific case of online prediction on the simplex, we use $(\regularizer{t})_{t\in\PosInts}$ to denote the sequence of regularizers in the latter. Thus, the $\FTRLeqn((\regularizer{t})_{t\in\PosInts})$ notation is really shorthand in this case for $\FTRLeqn(\simp(\experts), [0,1]^\experts, (\regularizer{t})_{t\in\PosInts})$.

A significant portion of the heavy-lifting required for \cref{thm:ftrl-simplex-summary} is done in \cref{sec:olo}, which proves a very similar result for generic \FTRL{} under some technical constraints. However, we cannot directly apply \cref{lem:inner-bound-ftrl-olo} when $\olodomain=\simp(\experts)$, since this set has empty interior. Thus, we need a version of that result tailored to the simplex, which we achieve by a reparametrization of the simplex.

In particular, let $\refexp\in\experts$ be arbitrary, and let $\redexperts = \experts\setminus\set{\refexp}$.
Let
\*[
  \redolodomain = \set{\oloplayervec\in [\PosReals]^{\redexperts} \stT \sum_{\expidx\in\redexperts} \oloplayer_{\expidx} \leq 1},
\]
and observe that $\interior(\redolodomain)$ is non-empty and convex. The canonical bijection $\simpbijection:\simp(\experts)\to \redolodomain$ is given by
\*[
  \simpbijection(\weightdumvec)
    & = \weightdumvec_{-\refexp},
    & \andT &&
  \simpbijection^{-1}(\oloplayervec)
    & = \rbra{
      \begin{cases}
        \oloplayer_\expidx
          &: \expidx\in\redexperts \\
        1-\inner{\one}{\oloplayervec}
          &: \expidx = \refexp
        \end{cases}
        }_{\expidx\in\experts}
\]
where $\weightdumvec_{-i}$ is the vector obtained from $\weightdumvec$ by dropping the coordinate with index $i$.

For any function $f:\simp(\experts)\to \Yy$ for some set $\Yy$, define $\hat f : \redolodomain\to \Yy$ by
\*[
  \hat f(\oloplayervec) &=  f(\simpbijection^{-1}(\oloplayervec)).
\]
For example, if we let $\entropy : \simp(\experts) \to \PosReals$ be the entropy function defined by
\*[
  \entropy(\weightdumvec) = -
    \sum_{\expidx\in\experts}\weightdum_\expidx\log\rbra{\weightdum_\expidx},
\]
then $\redentropy: \redolodomain\to \Reals $ is defined by
\*[
  \redentropy(\oloplayervec)
    & = \entropy(\simpbijection^{-1}(\oloplayervec))
    = -\rbra{\sum_{\expidx\in\redexperts} \oloplayer_\expidx\log\rbra{\oloplayer_\expidx} } - \rbra{1-\inner{\one}{\oloplayervec}}\log\rbra{1-\inner{\one}{\oloplayervec}}
    .
\]

Note that for any sequence of regularizers $(\regularizer{t})_{t\in\PosInts}$ on $\simp(\experts)$ and any sequence of losses $(\ololossvec(t))_{t\in\Nats}$ in an arbitrary $\ololossdomain \subseteq \Reals^\experts$, for all $t \in \Nats$ we have
\*[
  \inner{\oloLossvec(t)}{\weightdumvec} + \cumregularizer{t}(\weightdumvec)
    & = \inner{\oloLossvec_{-\refexp}(t)}{\weightdumvec_{-\refexp}}+ \oloLoss_{\refexp}(t)(1-\inner{\one}{\weightdumvec_{-\refexp}})  + \redcumregularizer{t}(\weightdumvec_{-\refexp}) \\
    & = \oloLoss_{\refexp}(t) + \inner{\oloLossvec_{-\refexp}(t) - \oloLoss_{\refexp}(t)\one}{\weightdumvec_{-\refexp}} + \redcumregularizer{t}(\weightdumvec_{-\refexp}).
\]
Additionally, for any $(\constdum(t))_{t\in\Nats} \subseteq \Reals$,
\*[
  \argmin_{\weightdumvec\in\simp(\experts)} \rbra{\inner{\oloLossvec(t)}{\weightdumvec} + \cumregularizer{t}(\weightdumvec) }
  &= \argmin_{\weightdumvec\in\simp(\experts)} \rbra{\inner{\oloLossvec(t) - \constdum(t)\one}{\weightdumvec}+ \cumregularizer{t}(\weightdumvec) }
\]
by the requirement that $\weightdumvec \in \simp(\experts)$. Similarly, for any sequence $(\weightdumvec(t))_{t \in \Nats} \subseteq \simp(\experts)$, the \regretname{} is unchanged by shifting the loss vectors. That is,
\*[
  &\hspace{-2em}\sum_{t=1}^T \inner{\ololossvec(t)}{\weightdumvec(t)} - \inf_{\weightdumvec \in \simp(\experts)}\sum_{t=1}^T \inner{\ololossvec(t)}{\weightdumvec} \\
  &= \sum_{t=1}^T \inner{\ololossvec(t) - \constdum(t)\one}{\weightdumvec(t)} - \inf_{\weightdumvec \in \simp(\experts)}\sum_{t=1}^T \inner{\ololossvec(t) - \constdum(t)}{\weightdumvec}.
\]
Thus, there exist equivalence classes of the outputs from the \FTRLalg{$\simp(\experts)$}{$\ololossdomain$}{$(\regularizer{t})_{t\in\PosInts}$} algorithm modulo parallel additive shifts of the loss vectors. Further, by transforming the losses via
\*[
  \simpequivalence(\ololossvec)
    & = \ololossvec_{-\refexp} - \ololoss_{\refexp}\one
    & \text{ and } &&
  \simpequivalence^{+}(\ololossvec)
    & = \rbra{
      \begin{cases}
        \ololoss_\expidx
          &: \expidx\in\redexperts \\
        0
          &: \expidx = \refexp
        \end{cases}
        }_{\expidx\in\experts}
\]
and defining $\redololossdomain = \{\simpequivalence(\ololossvec): \ololossvec \in \ololossdomain\}$, there is a canonical correspondence between the equivalence classes of the outputs from the \FTRLalg{$\simp(\experts)$}{$\ololossdomain$}{$(\regularizer{t})_{t\in\PosInts}$} algorithm and those of the outputs from the \FTRLalg{$\redolodomain$}{$\redololossdomain$}{$(\redregularizer{t})_{t\in\PosInts}$} algorithm. Namely,
\*[
  \argmin_{\weightdumvec\in\simp(\experts)} \rbra{\inner{\oloLossvec(t)}{\weightdumvec} + \cumregularizer{t}(\weightdumvec) }
    & = \simpbijection^{-1}\rbra{\argmin_{\oloplayervec\in \redolodomain}\rbra{\inner{\simpequivalence(\oloLossvec(t))}{\oloplayervec} + \redcumregularizer{t}(\oloplayervec)} }.
\]
Under this correspondence, if $\ololossdomain = [0,1]^\numexperts$, $\playerexpregret(T) =
\oloplayerregret(T)$.

\begin{corollary}
\label{cor:inner-bound-ftrl}
Consider a regularizer $\regularizer{0}:\simp(\experts)\to\Reals$ for which $\redregularizer{0}$ is closed, convex, of the Legendre type on $\redolodomain$ (see \cref{defn:legendre-type}), and twice continuously differentiable on $\interior(\redolodomain)$.
For each $t\in \Nats$, define $\cumregularizer{t}(\weightdumvec) = \timefunc(t)\regularizer{0}(\weightdumvec)$ for some increasing function $\timefunc: \Nats \to \PosReals$. Also, for any $\ololossvecgen \in \ololossdomain$ and $\oloplayervecgen \in \simp(\experts)$, define the time-dependent local semi-norm by $\norm{\ololossvecgen}_{t,\oloplayervecgen}^2 = \inner{\simpequivalence(\ololossvecgen)}{\hess\redcumregularizer{t}(\simpbijection(\oloplayervecgen)) \simpequivalence(\ololossvecgen)}$,
and its dual time-dependent local semi-norm by $\norm{\ololossvecgen}_{t,\oloplayervecgen,\star}^2 = \inner{\simpequivalence(\ololossvecgen)}{[\hess\redcumregularizer{t}(\simpbijection(\oloplayervecgen))]^{-1} \simpequivalence(\ololossvecgen)}$.
Then, for any sequence of losses $(\ololossvec(t))_{t\in \Nats} \subseteq \ololossdomain$ such that $\simpequivalence(-\frac{1}{\timefunc(t)}\oloLossvec(t))\in \grad \redregularizer{0}(\interior(\redolodomain))$ for all $t\in\Nats$,
there exists a sequence $(\alpha_t)_{t\in\Nats} \subseteq [0,1]$
such that, for all $t\in\Nats$, the weights $(\weightdumvec(t))_{t\in\Nats}$ output by the \FTRLalg{$\simp(\experts)$}{\ololossdomain}{$(\regularizer{t})_{t \in \PosInts}$} algorithm satisfy
\*[
	\inner{\ololossvec(t)}{\weightdumvec(t)-\weightdumvec(t+1)}
	\leq \frac{1}{\timefunc(t)} \norm{\left(\frac{\timefunc(t)}{\timefunc(t-1)} -1 \right)\oloLossvec(t-1) - \ololossvec(t)}_{0,\intweightvec(t+1),\star} \norm{\ololossvec(t)}_{0,\intweightvec(t+1),\star},
\]
where $\intweightvec(t+1) = \argmin_{\intweightvec\in \simp(\experts)} \rbra{\inner{\alpha_t \oloLossvec(t) + (1-\alpha_t) \frac{\timefunc(t)}{\timefunc(t-1)}\oloLossvec(t-1)}{\intweightvec}+\cumregularizer{t}(\intweightvec)}$.
\end{corollary}
\begin{proof}[Proof of {\cref{cor:inner-bound-ftrl}}]
For all $t\in\Nats$, since $\weightdumvec(t),\weightdumvec(t+1) \in \simp(\experts)$, it holds that for any $\ololossvec(t)\in\ololossdomain$,
\*[
  \inner{\ololossvec(t)}{\weightdumvec(t) - \weightdumvec(t+1)}
  &= \inner{\ololossvec(t) - \ololoss_\refexp(t)\one}{\weightdumvec(t) - \weightdumvec(t+1)} \\
  &= \inner{\ololossvec_{-\refexp}(t) - \ololoss_{-\refexp}(t)\one}{\weightdumvec_{-\refexp}(t) - \weightdumvec_{-\refexp}(t+1)} \\
  &= \inner{\simpequivalence(\ololossvec(t))}{\simpbijection(\weightdumvec(t)) - \simpbijection(\weightdumvec(t+1))}.
\]
Thus, using that $\simpbijection(\weightdumvec(t))$ are the weights output by the \FTRLalg{$\redolodomain$}{$\redololossdomain$}{$(\redregularizer{t})_{t\in\PosInts}$} algorithm, we can apply \cref{lem:inner-bound-ftrl-olo}. The result then follows from observing that $\simpequivalence$ is linear.
\end{proof}

\begin{lemma}
\label{lem:transformed-entropy-conditions-and-localnorm}
Suppose $\regularizer{0} = -\entropyfunc \circ \entropy$ for some $\entropyfunc:[0,\log\numexperts]\to\Reals$ that is strictly increasing, concave, and twice continuously differentiable on $\simp(\experts)$.
Then $\redregularizer{0}$ is closed, strictly convex, twice continuously differentiable on $\interior(\redolodomain)$, and of the Legendre type on $\redolodomain$.

  Moreover, for all $\oloplayervecgen \in \simp(\experts)$ and $\ololossvecgen \in \ololossdomain$,
  \*[
    \norm{\ololossvecgen}_{0,\oloplayervecgen,\star}^2
      & \leq \frac{1}{\entropyfunc'(H(\oloplayervecgen))} \VVar{\Expidx\sim \oloplayervecgen}[\ololossvecgen_{\scriptscriptstyle \Expidx}] .
  \]

\end{lemma}
\begin{proof}[Proof of {\cref{lem:transformed-entropy-conditions-and-localnorm}}]

First, note that for $\expidx\neq\expidxdum\in\redexperts$ and $\oloplayervec \in \redolodomain$,
  \*[
    -\partial_{\oloplayer_\expidx} \redentropy(\oloplayervec)
      & = \log(\oloplayer_\expidx) - \log\rbra{1-\inner{\one}{\oloplayervec}}, \\
    -\partial_{\oloplayer_\expidx}^2 \redentropy(\oloplayervec)
      & = \frac{1}{\oloplayer_\expidx} + \frac{1}{1-\inner{\one}{\oloplayervec}}, \text{ and } \\
    -\partial_{\oloplayer_\expidx}\partial_{\oloplayer_\expidxdum} \redentropy(\oloplayervec)
      & = \frac{1}{1-\inner{\one}{\oloplayervec}}.
  \]
Thus,
  \*[
    - \hess \redentropy(\oloplayervec)
      & = \diag(1/\oloplayervec) + \frac{1}{1-\inner{\one}{\oloplayervec}}\one\one\Tr,
  \]
which is strictly positive-definite on $\interior(\redolodomain)$.

Therefore $\redentropy$ is strictly concave. Since a composition of a strictly concave function with a strictly increasing strictly concave function is strictly concave, $\entropyfunc\circ\redentropy$ is strictly concave, which means $\redregularizer{0}$ is strictly convex. Since $\redregularizer{0}$ is continuous and finite on $\redolodomain$, and $\redolodomain$ is closed it must also be a closed function, because a proper convex function is closed if it is lower-semi-continuous. The twice continuous differentiability of $\redregularizer{0}$ on $\interior(\redolodomain)$ follows from the twice continuous differentiability of $\entropy$ on $\redolodomain$ and the twice differentiability of $\entropyfunc$.

Since we have already observed that $\interior(\redolodomain)$ is convex and non-empty, to see that $\redregularizer{0}$ is of the Legendre type on $\redolodomain$ we need only verify that $\lim_{n\to\infty}\norm{\grad\redregularizer{0}(\oloplayervec\upper{n})} \to\infty$ for any $\set{\oloplayervec\upper{n}}_{n\in\Nats}\subseteq\interior(\redolodomain)$ such that $\oloplayervec\upper{n}\to \oloplayerdumvec\in\boundary(\redolodomain)$. The gradient of $\redregularizer{0}$ is given by
\*[
  \grad \redregularizer{0}(\oloplayervec)
    & = -[\entropyfunc'\circ \redentropy(\oloplayervec)] \grad \redentropy(\oloplayervec).
\]
Now, notice that if $\oloplayerdumvec\in\boundary(\redolodomain)$, $\redentropy(\oloplayerdumvec)\leq \log(\numexperts - 1)$. Since $\entropyfunc$ is strictly increasing and concave on $[0,\log(\numexperts)]$, this implies $\entropyfunc'(\redentropy(\oloplayerdumvec))>0$. At any $\oloplayerdumvec\in\boundary \redolodomain$, either there exists an $\expidx\in\redexperts$ such that $\oloplayerdum_\expidx =0$ or $\inner{\one}{\oloplayerdumvec} = 1$. In both cases, $\oloplayervec\upper{n} \to \oloplayerdumvec\in\boundary(\redolodomain)$ implies $\norm[0]{\grad\redentropy(\oloplayervec\upper{n})} \to +\infty$. Therefore,
$\grad \redregularizer{0}(\oloplayervec_\expidx) \to \entropyfunc'(\redentropy(\oloplayerdumvec))\cdot(+\infty) = +\infty$, which confirms that $\redregularizer{0}$ is of the Legendre type on $\redolodomain$.

To derive the semi-norm formula, first notice that using the Sherman--Morrison--Woodbury formula gives
  \[
  -[\hess \redentropy(\oloplayervec)]^{-1}
    & = \diag(\oloplayervec) -  \diag(\oloplayervec)\one \rbra{(1-\inner{\one}{\oloplayervec}) +\one\Tr \diag(\oloplayervec) \one }^{-1} \one\Tr \diag(\oloplayervec) \\
    & = \diag(\oloplayervec) -  \diag(\oloplayervec)\one \one\Tr \diag(\oloplayervec) \\
    & = \diag(\oloplayervec) -  \oloplayervec \oloplayervec\Tr.
\label{eqn:seminorm-a}
  \]

  Then,
  \*[
    \hess\redregularizer{0}(\oloplayervec)
      & = -[\entropyfunc''\circ \redentropy(\oloplayervec)] (\grad \redentropy(\oloplayervec))(\grad \redentropy(\oloplayervec))\Tr - [\entropyfunc'\circ \redentropy(\oloplayervec)] (\hess \redentropy(\oloplayervec)) \\
      &\succeq - [\entropyfunc'\circ \redentropy(\oloplayervec)] (\hess \redentropy(\oloplayervec)),
  \]
  where $A \succeq B$ means $A - B$ is positive semi-definite.
  Therefore,
  \[
    [\hess\redregularizer{0}(\oloplayervec)]^{-1}
      &\preceq \rbra{- [\entropyfunc'\circ \redentropy(\oloplayervec)] (\hess \redentropy(\oloplayervec))}^{-1}\\
      & = \frac{1}{\entropyfunc'\circ \redentropy(\oloplayervec)} \rbra{\diag(\oloplayervec) -  \oloplayervec \oloplayervec\Tr}.
\label{eqn:seminorm-b}
  \]

Applying \cref{eqn:seminorm-a,eqn:seminorm-b} to an arbitrary $\oloplayervecgen \in \simp(\experts)$ and $\ololossvecgen \in \ololossdomain$ gives
  \*[
  \norm{\ololossvecgen}_{0,\oloplayervecgen,\star}^2
    & = \inner{\simpequivalence(\ololossvecgen)}{\hess\redregularizer{0}(\simpbijection(\oloplayervecgen))^{-1} \simpequivalence(\ololossvecgen)} \\
    & = \inner{\ololossvecgen_{-\refexp} - \ololossvecgen_\refexp \one}{\hess\redregularizer{0}(\oloplayervecgen_{-\refexp})^{-1} \sbra{\ololossvecgen_{-\refexp} - \ololossvecgen_\refexp\one}} \\
    & \leq \frac{1}{\entropyfunc'\circ \redentropy(\oloplayervecgen_{-\refexp})} \inner{\ololossvecgen_{-\refexp} - \ololossvecgen_\refexp \one}{\sbra{\diag(\oloplayervecgen_{-\refexp}) -  \oloplayervecgen_{-\refexp} \oloplayervecgen_{-\refexp}\Tr} \sbra{\ololossvecgen_{-\refexp} - \ololossvecgen_\refexp\one}} \\
    & = \frac{1}{\entropyfunc'\circ \entropy(\oloplayervecgen)} \inner{\ololossvecgen - \ololossvecgen_\refexp \one}{\sbra{\diag(\oloplayervecgen) -  \oloplayervecgen \oloplayervecgen\Tr} \sbra{\ololossvecgen - \ololossvecgen_\refexp\one}} \\
    & = \frac{1}{\entropyfunc'\circ \entropy(\oloplayervecgen)} \VVar{\Expidx\sim\oloplayervecgen} [\ololossvecgen_{\scriptscriptstyle \Expidx}-\ololossvecgen_\refexp] \\
    & = \frac{1}{\entropyfunc'\circ \entropy(\oloplayervecgen)} \VVar{\Expidx\sim\oloplayervecgen}[\ololossvecgen_{\scriptscriptstyle \Expidx}].
  \]
\end{proof}

\begin{lemma}
\label{lem:transformed-entropy-implicit-form}
Suppose $\regularizer{0} = -\entropyfunc \circ \entropy$ for some $\entropyfunc:[0,\log\numexperts]\to\Reals$ that is strictly increasing, concave, and twice continuously differentiable on $\simp(\experts)$.
Further, suppose that $\cumregularizer{t} = \timefunc(t)\regularizer{0}$ for some strictly increasing $\timefunc:\Nats\to\PosReals$.
Then, $[\grad\redregularizer{0}](\interior(\redolodomain)) = \Reals^{\redexperts}$, and
the weight vectors produced by the \FTRLalg{$\simp(\experts)$}{$\ololossdomain$}{$(\regularizer{t})_{t \in \PosInts}$} algorithm are equivalent to the weights produced by \OGHedge{} with an implicitly defined learning rate. In particular, the learning rate and weights are the solution to the system of equations
\[
	\ilr({t+1})
	& = \frac{1}{\timefunc(t)\cdot \entropyfunc'\circ\entropy(\weightdumvec(t+1))} \\
	\weightdumvec(t+1)
	& = \rbra{ \frac{\exp\rbra{-\ilr({t+1}) \oloLoss_\expidx(t)}}{\sum_{\expidxdum\in\experts}\exp\rbra{-\ilr({t+1}) \oloLoss_\expidxdum(t)}}}_{\expidx\in\experts}.
\label{eqn:ftrl-system}
\]
Moreover, for any sequence of losses $(\ololossvec(t))_{t \in \Nats} \subseteq \ololossdomain$, this system has a unique solution satisfying
\*[
  \ilr({t+1}) \in \cinter{\frac{1}{\timefunc(t)\cdot\entropyfunc'(0)},\ \frac{1}{\timefunc(t)\cdot\entropyfunc'(\log\numexperts)}}.
\]
\end{lemma}
\begin{proof}[Proof of \cref{lem:transformed-entropy-implicit-form}]
First, recall that the weights output by the \FTRLalg{$\simp(\experts)$}{$\ololossdomain$}{$(\regularizer{t})_{t \in \PosInts}$} algorithm will solve
\*[
	\weightdumvec(t+1)
  &= \argmin_{\weightvec \in \simp(\experts)}\left\{\inner{\oloLossvec(t)}{\weightvec} - \timefunc(t) \entropyfunc(\entropy(\weightvec)) \right\} \\
  &= \simpbijection^{-1} \left(\argmin_{\oloplayervec \in \redolodomain}\left\{\inner{\oloLossvec_{-\refexp}(t) -\one \oloLoss_{\refexp}(t)}{\oloplayervec} - \timefunc(t) \entropyfunc(\redentropy(\oloplayervec)) \right\} \right).
\]
By \cref{lem:transformed-entropy-conditions-and-localnorm,cor:legendre-interior-solution}, we know that this means $\weightdumvec(t+1) = \simpbijection^{-1}(\oloplayervec)$ for the unique $\oloplayervec \in \interior(\redolodomain)$ such that
\*[
  \grad \redentropy(\oloplayervec) = -\frac{\oloLossvec_{-\refexp}(t) -\one \oloLoss_{\refexp}(t)}{\timefunc(t)\cdot \entropyfunc'(\redentropy(\oloplayervec))}.
\]
Thus, by the definition of $\simpbijection$ and $\redentropy$,
\[
  \grad \redentropy(\weightdumvec_{-\refexp}(t+1))
  = -\frac{\oloLossvec_{-\refexp}(t) -\one \oloLoss_{\refexp}(t)}{\timefunc(t)\cdot \entropyfunc'(\redentropy(\weightdumvec_{-\refexp}(t+1)))}
  = -\frac{\oloLossvec_{-\refexp}(t) -\one \oloLoss_{\refexp}(t)}{\timefunc(t)\cdot \entropyfunc'(\entropy(\weightdumvec(t+1)))}.
\label{eqn:ftrl-solution1}
\]

It is well known that the unique solution to
\*[
\grad \redentropy(\simpbijection(\weightdumvec))
	& = \simpequivalence(-X)
\]
is given by
\*[
	\weightdumvec_\expidx
		& = \frac{\exp(-X_\expidx)}{\sum_{\expidxdum\in\experts}\exp(-X_\expidxdum)}.
\]
Therefore, any and all solutions of \cref{eqn:ftrl-solution1}
must also be solutions of \cref{eqn:ftrl-system}.
Next, we want to show that there is a unique solution, $\lr(t+1)$, to the implicit equation
\[\label{eqn:implicit}
	\ilr(t+1) = \frac{1}{\timefunc(t)\cdot\entropyfunc'\circ\entropy\rbra{\rbra{ \frac{\exp\{-\ilr(t+1) \oloLoss_\expidx(t)\}}{\sum_{\expidxdum\in\experts}\exp\{-\ilr(t+1) \oloLoss_{\expidxdum}(t)\}}}_{\expidx\in\experts}}}.
\]

On the left hand side, we have $f_1(\lr)=\lr$, which is trivially strictly increasing from $0$ to $\frac{1}{\timefunc(t)\cdot\entropyfunc'(\log\numexperts)}$ as $\eta$ increases from $0$ to $\frac{1}{\timefunc(t)\cdot\entropyfunc'(\log\numexperts)}$. On the right hand side, we have
\*[
	f_2(\lr)
	& = \frac{1}{\timefunc(t)\cdot\entropyfunc'\circ\entropy\rbra{\rbra{ \frac{\exp\{-\ilr(t+1) \oloLoss_\expidx(t)\}}{\sum_{\expidxdum\in\experts}\exp\{-\ilr(t+1) \oloLoss_{\expidxdum}(t)\}}}_{\expidx\in\experts}}},
\]
which is non-increasing with $f_2(0) = \frac{1}{\timefunc(t)\cdot\entropyfunc'(\log\numexperts)}$. Further, by non-negativity of entropy and concavity of $\entropyfunc$, $f_2(\lr) \geq \frac{1}{\timefunc(t)\cdot\entropyfunc'(0)}$.
Thus, $f_1$ and $f_2$ must intersect at some $\lr \in \cinter{\frac{1}{\timefunc(t)\cdot\entropyfunc'(0)},\ \frac{1}{\timefunc(t)\cdot\entropyfunc'(\log\numexperts)}}$, and this intersection is unique by the monotonicity of both functions and the strict monotonicity of $f_1$.

This guarantees at least one interior point solution to the implicit equation defined in \cref{eqn:implicit}. Moreover, since the objective function optimized by the weights output by the \FTRLalg{$\simp(\experts)$}{$\ololossdomain$}{$(\regularizer{t})_{t \in \PosInts}$} algorithm is strictly convex, this interior point solution must be the unique optimizer of the objective. Finally, since the sequence of losses was arbitrary and the \FTRLalg{$\redolodomain$}{$\redololossdomain$}{$(\redregularizer{t})_{t \in \PosInts}$} algorithm outputs a unique weight vector at each time $t+1$, we conclude that $[\grad\redregularizer{0}](\interior(F)) = \Reals^{\redexperts}$ as otherwise there would be some loss vector for which the solution to \cref{eqn:implicit} does not exist.

\end{proof}

\subsection{Proof of \cref{thm:ftrl-simplex-summary}}
\label{ssec:ftrl-simplex-summary-proof}
\cref{thm:ftrl-simplex-summary} is an immediate consequence of the combination of \cref{cor:inner-bound-ftrl,lem:adaptive-local-ftrl,lem:transformed-entropy-conditions-and-localnorm,lem:transformed-entropy-implicit-form}.
In the application of \cref{lem:adaptive-local-ftrl}, we can select $\weightdumvecopt(T)\in \argmin_{\weightdumvec\in \simp(\experts)} \inner{\Lossvec(T)}{\weightdumvec}$ such that $\entropy(\weightdumvecopt(T)) = 0$ because at least one $\argmin$ occurs at a vertex of the simplex.
\manualendproof

\section{Proofs of lemmas in \cref{sec:main-quant-sketch}}
\label{sec:helper-lemmas}

\subsection{Proof of \cref{lem:modular-entropy}}
\label{sec:modular-entropy-proof}

First, observe that
\*[
	\entropy(\weightdumvec)
	& = -\sum_{\expidx\in\experts} \weightdum_\expidx \log \left(\weightdum_\expidx\right)
	= -\sum_{\effexpidx\in\effexperts} \weightdum_\effexpidx \log \left(\weightdum_\effexpidx\right) 
		- \sum_{\expidx\in\neffexperts} \weightdum_\expidx \log \left(\weightdum_\expidx\right). 
\]

To bound the first term, consider the optimization problem
\*[
	\min_{\substack{\inner{\one}{\weightdumvec} = 1 \\ \inner{\one}{\weightdumvec_{\effexperts}} \leq 1}} \sum_{\effexpidx\in\effexperts} \weightdum_\effexpidx \log \left(\weightdum_\effexpidx\right),
\]
where $\weightdumvec_{\effexperts} \defas \set[0]{\weightdumvec_{\effexpidx}}_{\effexpidx\in\effexperts}$.
This is a convex objective with linear constraints, so it can be solved using the Lagrange multiplier method. The Lagrangian is 
\*[
	L(\weightdumvec; \alpha, \beta) 
		& = \sum_{\effexpidx\in\effexperts} \weightdum_\effexpidx \log \left(\weightdum_\effexpidx\right) +\alpha \rbra{\inner{\one}{\weightdumvec}-1} +\beta \rbra{\inner{\one}{\weightdumvec_{\effexperts}}-1},
\]
and the dual problem is 
\*[
	\max_{\substack{\alpha\in\Reals \\ \beta\geq 0}} \min_{\weightdumvec\in\Reals^\numexperts} \sum_{\effexpidx\in\effexperts} \weightdum_\effexpidx \log \left(\weightdum_\effexpidx\right) +\alpha \rbra{\inner{\one}{\weightdumvec}-1} +\beta \rbra{\inner{\one}{\weightdumvec_{\effexperts}}-1}.
\]
This gives, for $\effexpidx\in\experts$ and $\expidx\in\neffexperts$,
\*[
	\partial_{\effexpidx} L(\weightdumvec;\alpha, \beta)
		& = \log(\weightdum_\effexpidx)+1+\alpha+\beta, \text{ and } \\
	\partial_{\expidx} L(\weightdumvec;\alpha, \beta)
		& = \alpha.		
\]
Then, at the saddle point, $\alpha = 0$ and $\log\weightdum_{\effexpidx} = -\frac{1}{1+\beta}$ for all $\effexpidx\in\effexperts$. 

If $\beta = 0$ then $\weightdum_{\effexpidx}= \frac{1}{\exp(1)}$ for all $\effexpidx\in\effexperts$. 
This is only feasible if $\numeffexperts \leq 2$. In this case 
\*[
	\sum_{\effexpidx\in\effexperts} \weightdum_\effexpidx \log \left(\weightdum_\effexpidx\right) 
	\geq - \sum_{\effexpidx\in\effexperts} \frac{\log(\exp(1))}{\exp(1)}
	= - \frac{\numeffexperts}{\exp(1)}.
\]

Otherwise $\beta>0$, and by the K.K.T. condition, $\inner{\one}{\weightdumvec_{\effexperts}}=1$, which implies that $\weightdumvec_{\effexpidx} = \frac{\one}{\numeffexperts}$ for all $\effexpidx \in \effexperts$. That is,
\*[
	\sum_{\effexpidx\in\effexperts} \weightdum_\effexpidx \log \left(\weightdum_\effexpidx\right) 
	\geq - \sum_{\effexpidx\in\effexperts} \frac{\log(\numeffexperts)}{\numeffexperts}
	= - \log(\numeffexperts).
\]

Thus for $\numeffexperts\geq 3$
\[
	\entropy(\weightdumvec)
	& \leq \log(\numeffexperts) - \sum_{\expidx\in\neffexperts} \weightdum_\expidx \log\weightdum_\expidx,
\label{eqn:3experts}
\]
and for $\numeffexperts \leq 2  $
\[
	\entropy(\weightdumvec)
	& \leq \frac{\numeffexperts}{\exp(1)} - \sum_{\expidx\in\neffexperts} \weightdum_\expidx \log\weightdum_\expidx .
\label{eqn:2experts}
\]
Further, if $\effexperts = \{\effexpidx\}$, since $\log(x)\geq 1-1/x$ for all $x \geq 0$,
\[
	\entropy(\weightdumvec)
	& = - \sum_{\expidx \in \experts} \weightdum_\expidx \log \left(\weightdum_\expidx\right) \\
	& = -\weightdum_\effexpidx \log \left(\weightdum_\effexpidx\right) 
		- \sum_{\expidx\in\neffexperts} \weightdum_\expidx \log \left(\weightdum_\expidx\right) \\
	& \leq -\weightdum_\effexpidx \rbra{1-\frac{1}{\weightdum_\effexpidx}} 
		- \sum_{\expidx\in\neffexperts} \weightdum_\expidx \log \left(\weightdum_\expidx\right) \\
	&= (1-\weightdum_\effexpidx) 
			- \sum_{\expidx\in\neffexperts} \weightdum_\expidx \log \left(\weightdum_\expidx\right) \\
	&= \sum_{\expidx\in\neffexperts} \weightdum_\expidx -  \sum_{\expidx\in\neffexperts} \weightdum_\expidx \log \left(\weightdum_\expidx\right).
\label{eqn:1expert}
\]

In order to control the sum over ineffective experts we use the technical result of \cref{lem:log-leq-poly-0}, which says that
\[
	- \sum_{\expidx\in\neffexperts} \weightdum_\expidx \log \left(\weightdum_\expidx\right)
	\leq \frac{1}{(1-p) \exp(1)} \sum_{\expidx\in\neffexperts} [\weightdum_\expidx]^p
\label{eqn:neffexperts}
\]

Combing \cref{eqn:3experts,eqn:2experts,eqn:1expert,eqn:neffexperts} gives for $\numeffexperts \geq 1$,
\*[
	\entropy(\weightdumvec)
	& \leq \frac{2}{\exp(1)\log(2)}\log(\numeffexperts) + \rbra{1+\frac{1}{(1-p)\exp(1)}}\sum_{\expidx\in\neffexperts} \weightdum_\expidx^p.
\]
\manualendproof

\subsection{Proof of {\cref{lem:ftrl-both-weight-intweight-bound}}}
\label{apx:ineffective-weight-bound-proof}

For the first result, observe that from \cref{thm:ftrl-simplex-summary}, $\weightdumvec(t+1)$ is the unique solution to
\*[
  \ilr({t+1})
  & = \frac{1}{\sqrt{t+1}\cdot \entropyfunc'\circ\entropy(\weightdumvec(t+1))} \\
  \weightdumvec(t+1)
  & = \rbra{ \frac{\exp\rbra{-\ilr({t+1}) \Loss_\expidx(t)}}{\sum_{\expidxdum\in\experts}\exp\rbra{-\ilr({t+1}) \Loss_\expidxdum(t)}}}_{\expidx\in\experts},
\]
and
\*[
\ilr({t+1}) \in \cinter{\frac{1}{\sqrt{t+1}\cdot[\entropyfunc'(0)]},\ \frac{1}{\sqrt{t+1}\cdot[\entropyfunc'(\log\numexperts)]}}.
\]

Now, set $\lblr(t+1) = \frac{1}{\sqrt{t+1}\cdot\entropyfunc'(0)}$.
For $\expidx\in\neffexperts$, since $\lblr(t+1)\leq \ilr(t+1)$ and $\Loss_{\expidxoptpath(t)}(t) \leq \Loss_\expidx(t)$ by definition,
\*[
\Big[\weightdum_\expidx(t+1)\Big]^p
  & \leq   \rbra{\frac{\weightdum_\expidx(t+1)}{\weightdum_{\expidxoptpath(t)}(t+1)}}^p \\
  & = \exp\Big\{-p\ \ilr(t+1) \sbra{\Loss_\expidx(t)-\Loss_{\expidxoptpath(t)}(t)}\Big\} \\
  & \leq \exp\Big\{-p\ \lblr(t+1) \sbra{\Loss_\expidx(t)-\Loss_{\expidxoptpath(t)}(t)}\Big\} \\
  & \leq \min_{\effexpidx\in\effexperts}\exp\Big\{-p\ \lblr(t+1) \sbra{\Loss_\expidx(t)-\Loss_{\effexpidx}(t)}\Big\}.
\]
Thus, using \cref{thm:minimax_mgf},
\*[
\sup_{\policy \in \policyspace(\distnball{})} \EEboth\Big[\weightdum_\expidx(t+1)^p\Big]
&\leq \sup_{\policy \in \policyspace(\distnball{})} \EEboth
	\Big[\min_{\effexpidx\in\effexperts}\exp\Big\{-p\ \lblr(t+1) \sbra{\Loss_\expidx(t)-\Loss_{\effexpidx}(t)}\Big\}\Big] \\
&\leq \exp\left\{- t\lblr(t+1)\Deltaeff p + t\lblr(t+1)^2 \frac{p^2}{2} \right\} \\
& = \exp\left\{ - t\frac{1}{\sqrt{t+1}\cdot\entropyfunc'(0)}\Deltaeff p + t\rbra{\frac{1}{\sqrt{t+1}\cdot\entropyfunc'(0)}}^2 \frac{p^2}{2} \right\} \\
& \leq \exp\left\{\frac{p^2}{2(\entropyfunc'(0))^2}\right\} \exp\left\{- \frac{\Deltaeff p}{\sqrt{2} (\entropyfunc'(0))}\sqrt{t} \right\},
\]
where in the last inequality we used the fact that $\frac{t}{t+1}\geq \frac{1}{2}$ for $t\in\Nats$.

For the second result, for each $\alpha\in\cinter{0,1}$, we define the \emph{intermediate losses} $\Lossdumvec\upper{\alpha}(t) = \alpha \Lossvec(t) + (1-\alpha) \frac{\sqrt{t+1}}{\sqrt{t}} \Lossvec(t-1)$. We define a new random expert by $\intexpidxoptpath_{\alpha}(t) = \argmin_{\expidx\in\experts} \Lossdum\upper{\alpha}_\expidx(t)$, which is analogous to $\expidxoptpath(t)$ but for $\Lossdumvec\upper{\alpha}(t)$.
Then, applying \cref{lem:transformed-entropy-implicit-form} to the intermediate losses, observe that $\intweightvec\upper{\alpha}(t+1)$ is the unique solution to
\*[
	\ilrint\upper{\alpha}(t+1)
	& = \frac{1}{\sqrt{t+1}\cdot \entropyfunc'\circ\entropy(\intweightvec\upper{\alpha}(t+1))} \\
	\intweightvec\upper{\alpha}(t+1)
	& = \rbra{ \frac{\exp\rbra{-\ilrint\upper{\alpha}({t+1}) \Lossdum\upper{\alpha}_\expidx(t)}}{\sum_{\expidxdum\in\experts}\exp\rbra{-\ilrint\upper{\alpha}({t+1}) \Lossdum\upper{\alpha}_\expidxdum(t)}}}_{\expidx\in\experts},
\]
and
\*[
\ilrint\upper{\alpha}({t+1}) \in \cinter{\frac{1}{\sqrt{t+1}\cdot[\entropyfunc'(0)]},\ \frac{1}{\sqrt{t+1}\cdot[\entropyfunc'(\log\numexperts)]}}.
\]

Next, using that $\lossvec_\expidx(t) \in [0,1]$ for all $\expidx \in \experts$,
\*[
	\Lossdum\upper{\alpha}_{\expidx}(t)
	&= \alpha \Lossvec_{\expidx}(t) + (1-\alpha)\sqrt{\frac{t+1}{t}}\Lossvec_{\expidx}(t-1)
	\geq \Lossvec_{\expidx}(t) - 1.
\]

Then, observe that since $\Loss_{\expidx}(t) \leq t$ for all $\expidx\in\experts$, $\frac{\sqrt{t+1}}{\sqrt{t}}\Loss_{\expidx}(t-1) \leq \Loss_{\expidx}(t-1)+1$ for all $t\in\Nats$. Thus, for any $\expidxdum \in \experts$,
\*[
	\Lossdum\upper{\alpha}_{\expidxdum}(t)
	&= \alpha \Lossvec_{\expidxdum}(t)  + (1-\alpha)\sqrt{\frac{t+1}{t}} \Lossvec_{\expidxdum}(t-1)
	\leq \Lossvec_{\expidxdum}(t) + 1.
\]

Combining these two facts gives that for all $\alpha\in\cinter{0,1}$,
\*[
	\Lossdum\upper{\alpha}_{\expidx}(t) - \Lossdum\upper{\alpha}_{\expidxdum}(t)
	&\geq \Lossvec_{\expidx}(t) - \Lossvec_{\expidxdum}(t) - 2.
\label{eqn:lossdum-diff}
\]

Now, for $\expidx\in\neffexperts$, taking $\lblr(t+1) = \frac{1}{\sqrt{t+1}\cdot\entropyfunc'(0)}$, and since $\lblr(t+1)\leq \ilrint(t+1)$ we have
\*[
\Big[\intweight\upper{\alpha}_\expidx(t+1)\Big]^p
  & \leq \rbra{\frac{\intweight\upper{\alpha}_\expidx(t+1)}{\intweight\upper{\alpha}_{\intexpidxoptpath_\alpha(t)}(t+1)}}^p \\
  & = \exp\left\{-p\ \ilrint\upper{\alpha}(t+1) \sbra{\Lossdum\upper{\alpha}_\expidx(t)-\Lossdum\upper{\alpha}_{\intexpidxoptpath_\alpha(t)}(t)}\right\} \\
  & \leq \exp\left\{-p\ \lblr(t+1) \sbra{\Lossdum\upper{\alpha}_\expidx(t)-\Lossdum\upper{\alpha}_{\intexpidxoptpath_\alpha(t)}(t)}\right\} \\
  &\leq  \min_{\effexpidx\in\effexperts}\exp\left\{-p\ \lblr(t+1) \sbra{\Loss_\expidx(t)-\Loss_{\effexpidx}(t)-2}\right\}.
\]
Thus, again using \cref{thm:minimax_mgf},
\*[
\sup_{\policy \in \policyspace(\distnball{})} \EEboth \sup_{\alpha\in\cinter{0,1}}\Big[\intweight_\expidx(t+1)^p\Big]
&\leq \sup_{\policy \in \policyspace(\distnball{})} \EEboth \Big[\min_{\effexpidx\in\effexperts}\exp\left\{-p\ \lblr(t+1) \sbra{\Loss_\expidx(t)-\Loss_{\effexpidx}(t)-2}\right\} \Big] \\
&\leq \exp\Big\{2 p \lblr(t+1) - t\lblr(t+1)\Deltaeff p + t\lblr(t+1)^2 \frac{p^2}{2} \Big\} \\
& = \exp\bigg\{\frac{2 p - \Deltaeff p t}{\sqrt{t+1}\cdot\entropyfunc'(0)} + \rbra{\frac{1}{\sqrt{t+1}\cdot\entropyfunc'(0)}}^2 \frac{p^2 t}{2} \bigg\} \\
& \leq \exp\left\{\frac{2p}{\entropyfunc'(0)} +\frac{p^2}{2(\entropyfunc'(0))^2}\right\} \exp\left\{- \frac{\Deltaeff p}{\sqrt{2} (\entropyfunc'(0))}\sqrt{t} \right\},
\]
where in the last inequality we again used the fact that $\frac{t}{t+1}\geq \frac{1}{2}$ for $t\in\Nats$.
\manualendproof

\subsection{Proof of \cref{lem:worst-case-ftrl-bound}}
\label{apx:early-quasi-regret-bound-proof}

Substituting the variance bounds of \cref{lem:simple-var-bd} \cref{eqn:main-quasi-regret-bound} using $\timefunc(t) = \sqrt{t+1}$, $\entropyfunc$ increasing and concave, and the fact that $\entropy(\weightdumvec) \leq \log(\numexperts)$ gives
\*[
  \playerexpregretFTRLHeqn(\tmid) 
  &\leq -\entropyfunc(0)\sqrt{\tmid+1} 
  + \sum_{t=0}^{\tmid} \Big[\sqrt{t+1}-\sqrt{t}\Big] \entropyfunc(\log(\numexperts))
  + \sum_{t=1}^{\tmid} \frac{3}{8\sqrt{t+1} \cdot \entropyfunc'(\log(\numexperts))}. \\
  &= \sqrt{\tmid+1} \Big(\entropyfunc(\log(\numexperts)) - \entropyfunc(0) \Big)
  + \sum_{t=1}^{\tmid} \frac{3}{8\sqrt{t+1} \cdot \entropyfunc'(\log(\numexperts))}.
\]

Then, since 
\*[
	\sum_{t=1}^{\tmid} \frac{1}{\sqrt{t+1}} 
	\leq \int_0^{\tmid} \frac{1}{\sqrt{t+1}} dt
	= 2\sqrt{\tmid+1},
\]
we have that
\*[
  \playerexpregretFTRLHeqn(\tmid) 
    \leq \sqrt{\tmid+1} \bigg(\entropyfunc(\log(\numexperts)) - \entropyfunc(0) + \frac{3}{4\entropyfunc'(\log(\numexperts))} \bigg).
\]
\manualendproof

\subsection{Miscellaneous stochastic and mathematical results}
\label{sec:misc-results}

Here we state a few convenient results that will be used repeatedly, but require none of the assumptions of our setting except boundedness. The first two of these lemmas allow us to control the variance of the experts' losses. 

\begin{lemma}\label{lem:simple-var-bd}
For any $\weight\in\simp(\experts)$, $(\lossvec(t))_{t\in\Nats} \subseteq [0,1]^\numexperts$, and $t\in\Nats$,
\*[
  \VVar{\Expidx\sim \weight}
    \sbra{\rbra{\sqrt{\frac{t+1}{t}} -1} \Loss_\Expidx(t-1) - \loss_\Expidx(t)}
     \leq \frac{9}{16} \ \text{ and } \
  \VVar{\Expidx\sim \weight}\sbra{\loss_\Expidx(t)}
     \leq \frac{1}{4}.
\]
\end{lemma}
\begin{proof}[Proof of \cref{lem:simple-var-bd}]
Since $\sqrt{t+1} - \sqrt{t} \leq \frac{1}{2\sqrt{t}}$ for $t\geq 1$, $\rbra{\sqrt{\frac{t+1}{t}} -1} \in \cinter{0,\frac{1}{2t}}$. Combined with $\lossvec(t) \in [0,1]^\experts$ for all $t\in\Nats$, this gives that for all $\expidx \in \experts$,
\*[
  \rbra{\sqrt{\frac{t+1}{t}} -1} \Loss_\expidx(t-1) - \loss_\expidx(t) \in \cinter{-1,\ \frac{1}{2}}.
\]
Thus, the result follows since if $a \leq X \leq b$, then $\Var(X) \leq (b-a)^2/4$.
\end{proof}

\begin{lemma}\label{lem:mix-var-bd}
For any $\policy\in\policyspace$, $\predpolicy\in\predpolicyspace$, sequence $(\weightvec(t))_{t\in\Nats}$ such that $\weightvec(t)$ is $\sigma(\history(t-1))$-measurable for all $t$, $\effexpidx\in\experts$, and $t\in\Nats$,
\*[
	&\hspace{-2em} 
	\EEboth \sqrt{\VVar{\Expidx\sim \weightvec(t+1)} \sbra{\bigg(\frac{\sqrt{t+1}}{\sqrt{t}} -1\bigg) \Loss_\Expidx(t-1) - \loss_\Expidx(t)} \vphantom{\left(\frac{\timefunc(t)}{\timefunc(t-1)}\right)}\VVar{\Expidx\sim \weightvec(t+1)}\sbra{\loss_\Expidx(t)}}\\
  &\leq \frac{9}{4} \EEboth\Bigg[\sum_{\expidx\neq\effexpidx} \weight_\expidx(t+1)\Bigg]
\]
and
\*[
	&\hspace{-2em} 
	\EEboth \Bigg[\VVar{\Expidx\sim \weightvec(t+1)} \sbra{\bigg(\frac{\sqrt{t+1}}{\sqrt{t}} -1\bigg) \Loss_\Expidx(t-1) - \loss_\Expidx(t)} \vphantom{\left(\frac{\timefunc(t)}{\timefunc(t-1)}\right)}\VVar{\Expidx\sim \weightvec(t+1)}\sbra{\loss_\Expidx(t)}\Bigg]\\
  &\leq \frac{27}{32} \EEboth\Bigg[\sum_{\expidx\neq\effexpidx} \weight_\expidx(t+1)\Bigg].
\]
\end{lemma}
\begin{proof}[Proof of \cref{lem:mix-var-bd}]
First, let $\nu$ be any distribution such that $\supp(\nu)\subset[-y,1-y]$ and $x\in[-y,1-y]$, and suppose $X \sim \alpha \delta_x +(1-\alpha)\nu$ for some $\alpha \in [0,1]$.

Since variance is invariant to shifts, we can suppose $y=0$ without loss of generality. Define $\mu_\nu = \EE_{Z\sim\nu}(Z)$ and $\sigma^2_\nu = \mathop{\Var}_{Z\sim\nu}(Z)$. Then, using the variance for a mixture distribution,
\*[
	\Var(X)
	& = \alpha x^2 +(1-\alpha)\mu_\nu^2 - (\alpha x +(1-\alpha)\mu_\nu)^2 + (1-\alpha)\sigma_\nu^2  \\
	& = \alpha(1-\alpha) x^2 +\alpha(1-\alpha)\mu_\nu^2 -2\alpha(1-\alpha)x\mu_\nu+(1-\alpha)\sigma^2_\nu \\
	& = \alpha(1-\alpha) (x-\mu_\nu)^2+ (1-\alpha)\sigma^2_\nu
\]
Now,
\*[
	\sup_{x,\nu} \Var(X)
		& = \sup_{x,\mu} \sup_{\nu: \mu_\nu =\mu} \Var(X).
\]
The inner $\sup$ is achieved by $\nu(\mu) = \bernoullidist(\mu)$ and has $\sigma^2_{\nu(\mu)} =\mu(1-\mu)$, so that
\*[
	\sup_{x,\nu} \Var(X)
		& = \sup_{\mu} \sup_{x} \alpha(1-\alpha) (x-\mu)^2+ (1-\alpha)\mu(1-\mu).
\]
Now, the inner sup is achieved by $x=0$ when $\mu\geq1/2$ and by $x=1$ when $\mu <1/2$. Due to symmetry we need only consider the case that $\mu \geq 1/2$.
\*[
	\sup_{x,\nu} \Var(X)
		& = \sup_{\mu} \alpha(1-\alpha) \mu^2+ (1-\alpha)\mu(1-\mu)\\
		& = \sup_{\mu}\sbra{ -(1-\alpha)^2 \mu^2+ (1-\alpha)\mu}.
\]
Since $\mu \in [0,1]$ this is a constrained quadratic maximum. If the unconstrained maximum occurs in interior of the region then it is equal to the constrained maximum. Otherwise the constrained maximum occurs at the boundary.

The unconstrained maximum occurs at $\mu = \frac{1}{2(1-\alpha)}$ with objective value $1/4$. This in the interior of the constraint region when $(1-\alpha)> 1/2$; equivalently $\alpha < 1/2$. The boundary values are $0$ and $\alpha(1-\alpha)$.

That is,
\[\label{eqn:generic-case-var-bd}
	\Var(X)
		& \leq \begin{cases}
							\alpha(1-\alpha) &: \alpha \geq 1/2\\
							1/4 &: \alpha < 1/2
						\end{cases}.
\]

Let $\weight\in\simp(\experts)$ be arbitrary.
We can apply \cref{eqn:generic-case-var-bd} to obtain
\*[
  \VVar{\Expidx\sim \weight}\sbra{\loss_\Expidx(t)} 
  	&\leq \frac{1}{4}\ind{\weight_\effexpidx\leq 1/2} + (1-\weight_\effexpidx).
\]
Similarly, since $\rbra{\sqrt{\frac{t+1}{t}} -1} \Loss_\Expidx(t-1) - \loss_\Expidx(t) \in \cinter{-1,\frac{1}{2}}$,
\*[
	&\VVar{\Expidx\sim \weight} \sbra{\bigg(\frac{\sqrt{t+1}}{\sqrt{t}} -1\bigg) \Loss_\Expidx(t-1) - \loss_\Expidx(t)} \\
  	& \qquad \leq \rbra{\frac{3}{2}}^2 \rbra{\frac{1}{4}\ind{\weight_\effexpidx\leq 1/2} + (1-\weight_\effexpidx)}.
\]

Thus, using Markov's inequality, 
\*[
	&\hspace{-2em} 
	\EEboth \sqrt{\VVar{\Expidx\sim \weightvec(t+1)} \sbra{\bigg(\frac{\sqrt{t+1}}{\sqrt{t}} -1\bigg) \Loss_\Expidx(t-1) - \loss_\Expidx(t)} \vphantom{\left(\frac{\timefunc(t)}{\timefunc(t-1)}\right)}\VVar{\Expidx\sim \weightvec(t+1)}\sbra{\loss_\Expidx(t)}}\\
	& \leq \frac{3}{2}\rbra{ \frac{1}{4}\Prboth\sbra{\weight_\effexpidx(t+1)\leq 1/2} + \EEboth\sbra{1-\weight_\effexpidx(t+1)}}\\
  &\leq \frac{9}{4} \EEboth\Bigg[\sum_{\expidx\neq\effexpidx} \weight_\expidx(t+1)\Bigg].
\]

Alternatively, using $\VVar{\Expidx\sim \weight}\sbra{\loss_\Expidx(t)} \leq 1/4$,
\*[
	&\hspace{-2em} 
	\EEboth \Bigg[\VVar{\Expidx\sim \weightvec(t+1)} \sbra{\bigg(\frac{\sqrt{t+1}}{\sqrt{t}} -1\bigg) \Loss_\Expidx(t-1) - \loss_\Expidx(t)} \vphantom{\left(\frac{\timefunc(t)}{\timefunc(t-1)}\right)}\VVar{\Expidx\sim \weightvec(t+1)}\sbra{\loss_\Expidx(t)}\Bigg]\\
	& \leq \frac{9}{16}\rbra{ \frac{1}{4}\Prboth\sbra{\weight_\effexpidx(t+1)\leq 1/2} + \EEboth\sbra{1-\weight_\effexpidx(t+1)}}\\
  &\leq \frac{27}{32} \EEboth\Bigg[\sum_{\expidx\neq\effexpidx} \weight_\expidx(t+1)\Bigg].
\]

\end{proof}

Next, we have a result which controls a summation term which appears often in our proofs.

\begin{lemma}
\label{lem:integral-comparison-simplified}
For any $\alpha>0$ and $\tmid \geq 1$
\*[
	\sum_{t=\tmid+1}^T \frac{1}{\sqrt{t}} \exp\left\{-\alpha \sqrt{t}\right\}
	\leq \frac{2}{\alpha} \exp(-\alpha\sqrt{t_0} ).
\]
\end{lemma}
\begin{proof}[Proof of {\cref{lem:integral-comparison-simplified}}]
\*[
	\sum_{t=\tmid+1}^T \frac{1}{\sqrt{t}} \exp\left\{-\alpha \sqrt{t}\right\}
	& \leq
		\int_{\tmid}^T \frac{1}{\sqrt{t}} \exp\left\{-\alpha\sqrt{t}\right\} dt \\
	& \leq
		\int_{\tmid}^\infty \frac{1}{\sqrt{t}} \exp\left\{-\alpha\sqrt{t}\right\} dt \\
	& =
		\int_{\sqrt{\tmid}}^\infty 2\exp\left\{-\alpha u \right\} du \\
	& = \frac{2}{\alpha} \exp(-\alpha\sqrt{t_0} ).
\]
\end{proof}

Finally, we have a simple fact about logarithms that will be useful when controlling the entropy of weight distributions.

\begin{lemma}\label{lem:log-leq-poly-0}
For $x\in(0,1]$ and $p\in(0,1)$
\*[
	- x\log(x) \leq \frac{1}{(1-p) \exp(1)} x^p.
\]
\end{lemma}
\begin{proof}[Proof of \cref{lem:log-leq-poly-0}]
	Consider $f(x) = -x^{1-p}\log(x)$. Then, $f(0^+) = f(0)$, $f(1) = 0$, and
	\*[
		f'(x) = -(1-p) x^{-p} \log(x) - x^{-p} = -x^{-p}((1-p)\log(x)+1).
	\]
	Thus, the only critical point of $f$ occurs at $x_0 = \exp(-1/(1-p))$. This is a local max since $\sign (f'(x)) = -\sign(x-x_0)$ for $x\in(0,1)$. Thus, $f$ is maximized on the interval $(0,1)$ at $x_0$. Hence $f(x)\leq f(x_0) = \frac{1}{(1-p)\exp(1)}$. Multiplying both sides by $x^p$ proves the result.

\end{proof}

\section{Proofs of lower bounds}
\label{apx:lower-bound-proofs}

\subsection{Proof of \cref{thm:oracle-lb}}
\label{ssec:oracle-lb-proof}
Our strategy is to define a simple setting with multiple experts (many of them identical), so that we can show the lower bound holds in the asymptotic limit as $T$ and $\numexperts$ tends to infinity.
Let $\dataspace = \set{(1,0,0), (0,1,0), (0,0,1)}$, $\predspace = \simp(\range{3})$, and $\loss(\pred, \data) = \frac{1}{2} \sum_{i=1}^3 \abs{\pred_i -\data_i}$. 
Observe that $\loss(\pred, \data)\in\cinter{0,1}$ for all $\pred\in\predspace$ and $\data\in\dataspace$. 
Let $\numeffexperts\leq \numexperts\in \Nats$.

In this setting, consider the distribution
\*[
	\refdistn
  	= \rbra{\rbra{\frac{1}{2} \delta_{(1,0,0)} +\frac{1}{2} \delta_{(0,1,0)}}^{\otimes \numeffexperts} \otimes (\delta_{(0,0,1)})^{\otimes(\numexperts-\numeffexperts)}}\otimes \rbra{\frac{1}{2} \delta_{(1,0,0)} +\frac{1}{2} \delta_{(0,1,0)}},
\]
and let $\distnball{} = \set{\refdistn}$. 
Then $\policyspace(\distnball{})$ contains a single policy, $\policy_\star$, given by
\*[
	\policy_\star = (\history(t)\in\historyspace^t\mapsto \refdistn)_{t\in\Nats}.
\]
Intuitively, each of the effective experts flips a coin to play the first or second element, but the observation is also either the first or second element from an independent coin toss, and the ineffective experts always output the third element.

Now, define the pushforward of the distribution through the loss function by $\refdistn^\loss = \pushfwdmeas{\loss}{\refdistn}$ to obtain the single loss distribution on the experts. 
Observe this simplifies to
\*[
	\refdistn^\loss = \bernoullidist(1/2)^{\otimes \numeffexperts} \otimes \bernoullidist(1)^{\otimes (N-\numeffexperts)}.
\]
This singleton policy space satisfies the time-homogeneous convex constraint condition with $\effexperts = \range{\numeffexperts}$, and $\Deltaeff = 1/2$.

Note that any prediction $\pred$ has $\EEE{\data\sim 
\refdistn} \loss(\pred ,\data)\geq \frac{1}{2}$.
For each $\effexpidx\in\effexperts$, let $M_{\effexpidx} \defas \sum_{t=1}^T \loss(\pred_{\effexpidx},\data)$ be the random variable corresponding to the cumulative loss of the effective expert. Then, $M_{\effexpidx}\distiidas\binomdist(T,1/2)$, and
\*[
	&\hspace{-2em}
	\inf_{\predpolicy \in \predpolicyspace_\numexperts} \
	\sup_{\policy \in \policyspace(\distnball{})} \
	\EEboth
	\max_{\expidx \in \experts}
	\sum_{t=1}^T \sbra{\loss(\pred(t), \data(t)) - \loss(\exppred_\expidx(t),\data(t))} \\
	&\geq \EEE{M\sim\binomdist(T,1/2)^{\otimes \numeffexperts}} \max_{\effexpidx\in\effexperts} \Big(T/2 - M_{\effexpidx}\Big).	
\]
Now, since $\frac{2}{\sqrt{T}}\rbra{T/2 - M_\effexpidx}$ are \iid{} and converge in Wasserstein distance to a $\normaldist(0,1)$ as $T\to\infty$ (from, for example, \cite[Theorem~3.1]{chen2010normal}), and since $\max$ is Lipschitz,
\*[
\lim_{T\to\infty} \EEE{M\sim\binomdist(T,1/2)^{\otimes \numeffexperts}} \rbra{\max_{\effexpidx\in\effexperts}\frac{1}{\sqrt{T}}\rbra{T/2 - M_\effexpidx}} 
	& = \frac{1}{2}\EEE{Z\sim\normaldist(0,1)^{\otimes \numeffexperts}} \rbra{\max_{\effexpidx\in\effexperts}Z_\effexpidx}.
\]
We now
turn to the non-asymptotic lower bound of \citet{kamath2015bounds}, which states that for all $\numeffexperts\in\Nats$
\*[
	\frac{\EEE{Z\sim\normaldist(0,1)^{\otimes \numeffexperts}} \rbra{\max_{\effexpidx\in\effexperts}Z_\effexpidx}}{0.23\sqrt{\log \numeffexperts}} \geq 1,
\]
Now, by the definition of limit, for each $\numeffexperts$ there exists a $\tmid(\numeffexperts)$ such that for $T\geq \tmid(\numeffexperts)$
\*[
\EEE{M\sim\binomdist(T,1/2)^{\otimes \numeffexperts}} \rbra{\max_{\effexpidx\in\effexperts}\frac{1}{\sqrt{T}}\rbra{T/2 - M_\effexpidx}} 
	& \geq \rbra{\frac{0.2}{0.23}} \rbra{\frac{1}{2}}\EEE{Z\sim\normaldist(0,1)^{\otimes \numeffexperts}} \rbra{\max_{\effexpidx\in\effexperts}Z_\effexpidx}\\
	& \geq \sqrt{\rbra{\log\numeffexperts} / 100} \, . 
\]

Combining these facts, 
we have that for any $\numexperts\in\Nats$, $\numeffexperts\leq\numexperts$ and $T\geq \tmid(\numeffexperts)$,
\*[
	\inf_{\predpolicy \in \predpolicyspace_\numexperts} 
	\sup_{\policy \in \policyspace(\distnball{})}
	\frac{\EEboth
	\max_{\expidx \in \experts}\sum_{t=1}^T \sbra{
	\loss(\pred(t), \data(t)) - \loss(\expidxpred(t),\data(t))}}{\sqrt{(T\log\numeffexperts) / 100}}
		& \geq 1.
\]  
\manualendproof

\subsection{Proof of \cref{thm:hedge-lb}}\label{ssec:hedge-lb-proof}

Fix $\numexperts > 0$, $\numeffexperts \leq \numexperts$, and $c>0$ within the respective constraints of either (i) or (ii) of \cref{thm:hedge-lb}.
Let $\dataspace = \{0,1\}^\numexperts$, $\predspace = [0,1]^\numexperts$, and $\loss(\pred,\data) = \inner{\pred}{\data}$, and
suppose $T \geq \frac{32\log\numexperts}{c^2}$.
In order to prove both cases of the \Hedge{} lower bound, our approach is first to define a specific example of a $\distnball{} \in \distnballsuperset{}(\numexperts,\numeffexperts)$. Then, for either case we find a specific policy $\policy \in \policyspace(\distnball{})$ which forces \Hedge{} to incur at least as much \regretname{} as the desired lower bound. It turns out that we do not need anything more complicated than a $\distnball{}$ that consists of convex combinations of deterministic experts.

For simplicity, suppose that $\numeffexperts$ is even. (The argument is the same, but with some more housekeeping, when $\numeffexperts$ is odd.) We wish to split $\experts$ up so that $\effexperts = \range{\numeffexperts}$, and thus $\neffexperts = \range{\numexperts} \setminus \range{\numeffexperts}$. To do so, we define a set of distributions on $\dataspace$ by
\*[
	U 
	= \set{\delta_{m}^{\otimes \numeffexperts/2}\otimes \delta_{1-m}^{\otimes \numeffexperts/2} \otimes \delta_1^{\otimes (N-\numeffexperts)} \stT  m \in \set{0,1}}
	\union
	\set{\delta_0\otimes \delta_1^{\otimes (N-1)}},
\]
and suppose that each expert $\expidx\in\experts$ predicts $(\exppred)_{\expidx} = \unitvec_\expidx$, the unit vector in direction $\expidx$.
Thus, the set $U$ induces three different expert loss distributions. In each of these, the incurred loss of any expert is assigned either a Dirac measure at $0$ or at $1$. Thus, the three distributions are defined by which experts incur loss of $0$ (with the rest incurring loss of $1$). These options are either: a) the first $\numeffexperts/2$ incur loss of $0$, b) the experts labelled $\numeffexperts/2 + 1$ to $\numeffexperts$ incur loss of $0$, and c) only the first expert incurs loss of $0$.

Then, we define $\distnball{}$ to be the 
convex hull of $U$. 
One can check that any convex combination of the three distributions in $U$ can only lead to an expert in $\effexperts$ being optimal in expectation, and additionally note that $\Deltaeff = 1/2$. Consequently, $\distnball{} \in \distnballsuperset{}(\numexperts,\numeffexperts)$, so it remains to find a $\policy \in \policyspace(\distnball{})$ that forces \Hedge{} with either parametrization to incur the \regretname{} of the theorem.

Before we do this, we first recall the adversarial analysis of \Hedge{} by \cite[Theorem~2.3]{plg07}. Similar to that analysis, we will analyze the telescoping series
\*[
	\Psi(t) = \frac{1}{\lr(t+1)}\log(\hedgeweight_{\expidxoptpath(t)}(t+1)) - \frac{1}{\lr(t)}\log(\hedgeweight_{\expidxoptpath(t-1)}(t)),
\]
which, for an arbitrary $\tmid$, satisfies
\*[
	\sum_{t=\tmid+1}^T \Psi(t) 
	& = \frac{1}{\lr(T+1)}\log(\hedgeweight_{\expidxoptpath(T)}(T+1)) - \frac{1}{\lr(\tmid+1)}\log(\hedgeweight_{\expidxoptpath(\tmid)}(\tmid+1)).
\]

When upper bounding, \citeauthor{plg07} used that the first term was negative and kept the second term, but we now wish to use that the second term is positive to obtain
\[\label{eqn:lb-psi-bound}
	\sum_{t=\tmid+1}^T \Psi(t)
	&\geq  \frac{1}{\lr(T+1)}\log(\hedgeweight_{\expidxoptpath(T)}(T+1)).
\]

Then, we can partition $-\Psi(t)$ into
\*[
	-\Psi(t) 
	&= \rbra{\frac{1}{\lr(t+1)} - \frac{1}{\lr(t)}} \log\left(\frac{1}{\hedgeweight_{\expidxoptpath(t)}(t+1)}\right)
	+ \frac{1}{\lr(t)}\log
		\left(
		\frac
		{
			\frac
			{\exp\left\{-\lr(t) \Loss_{\expidxoptpath(t)}(t) \right\}}
			{\sum_{\expidx\in\experts}\exp\left\{-\lr(t) \Loss_{\expidx}(t) \right\}}
		}
		{
			\frac
			{\exp\left\{-\lr(t+1) \Loss_{\expidxoptpath(t)}(t) \right\}}
			{\sum_{\expidx\in\experts}\exp\left\{-\lr(t+1) \Loss_{\expidx}(t) \right\}}
		}		
		\right) \\
	& \qquad + \frac{1}{\lr(t)} \log
		\left(
			\frac
			{\sum_{\expidx\in\experts}\exp\left\{-\lr(t) \Loss_{\expidx}(t) \right\}}
			{\sum_{\expidx\in\experts}\exp\left\{-\lr(t) \Loss_{\expidx}(t-1) \right\}}
		\right)
	+ \sbra{\Loss_{\expidxoptpath(t)}(t) - \Loss_{\expidxoptpath(t-1)}(t-1)}.
\]
Observe that
\*[
	&\hspace{-2em}\frac{1}{\lr(t)} \log
		\left(
			\frac
			{\sum_{\expidx\in\experts}\exp\left\{-\lr(t) \Loss_{\expidx}(t) \right\}}
			{\sum_{\expidx\in\experts}\exp\left\{-\lr(t) \Loss_{\expidx}(t-1) \right\}}
		\right) \\
	&= \frac{1}{\lr(t)} \log
		\left(
			\sum_{\expidx\in\experts}\frac
			{\exp\left\{-\lr(t) \Loss_{\expidx}(t-1) \right\}}
			{\sum_{\expidxdum\in\experts}\exp\left\{-\lr(t) \Loss_{\expidxdum}(t-1) \right\}} \exp\left\{-\lr(t) \loss_{\expidx}(t) \right\}
		\right) \\
	&= \frac{1}{\lr(t)} \log
		\left(
			\sum_{\expidx\in\experts}\hedgeweight_\expidx(t) \exp\left\{-\lr(t) \loss_{\expidx}(t) \right\}
		\right) \\
	&= - \sum_{\expidx \in \experts}\hedgeweight_\expidx(t) \loss_\expidx(t) \\
		&\qquad
		+ \frac{1}{\lr(t)}\log
		\left( 
			\exp\left\{\lr(t)\sum_{\expidx \in \experts}\hedgeweight_\expidx(t)\loss_\expidx(t) \right\} 
			\sum_{\expidx \in \experts}\hedgeweight_\expidx(t)\exp\left\{-\lr(t) \loss_\expidx(t) \right\}
		\right) \\
	&= - \sum_{\expidx \in \experts}\hedgeweight_\expidx(t) \loss_\expidx(t) \\
		&\qquad
		+ \frac{1}{\lr(t)}\log
		\left( 
			\sum_{\expidx \in \experts}\hedgeweight_\expidx(t)\exp\left\{\lr(t) \left[- \loss_\expidx(t) + \sum_{\expidxdum\in\experts} \hedgeweight_{\expidxdum}(t) \loss_{\expidxdum}(t)\right] \right\}
		\right).
\]
Thus, we can write
\*[
	-\Psi(t) 
	&= \rbra{\frac{1}{\lr(t+1)} - \frac{1}{\lr(t)}} \log\left(\frac{1}{\hedgeweight_{\expidxoptpath(t)}(t+1)}\right)
	+ \frac{1}{\lr(t)}\log
		\left(
		\frac
		{
			\frac
			{\exp\left\{-\lr(t) \Loss_{\expidxoptpath(t)}(t) \right\}}
			{\sum_{\expidx\in\experts}\exp\left\{-\lr(t) \Loss_{\expidx}(t) \right\}}
		}
		{
			\frac
			{\exp\left\{-\lr(t+1) \Loss_{\expidxoptpath(t)}(t) \right\}}
			{\sum_{\expidx\in\experts}\exp\left\{-\lr(t+1) \Loss_{\expidx}(t) \right\}}
		}		
		\right) \\
	& \qquad -\sum_{\expidx\in\experts} 
	\hedgeweight_{\expidx}(t)  \loss_{\expidx}(t)
		+ \frac{1}{\lr(t)} \log
		\left(
 				 \mathop{\EE}_{\Expidx \sim \hedgeweightvec(t)}\exp\left\{\lr(t) \rbra{-\loss_{\Expidx}(t) - \mathop{\EE}_{\Expidxdum \sim \hedgeweightvec(t)}[-\loss_{\Expidxdum}(t)]} \right\}
		\right) \\
	&\qquad + \sbra{\Loss_{\expidxoptpath(t)}(t) - \Loss_{\expidxoptpath(t-1)}(t-1)} \\
	& \defas A(t)+B(t)+C_1(t) +C_2(t)+D(t) .
\]

First, observe that since $\lr(t)$ is decreasing in both cases, $B(t) \geq 0$. Also,
\*[
	\sum_{t=\tmid+1}^T A(t) 
	&= \sum_{t=\tmid+1}^T\rbra{\frac{1}{\lr(t+1)} - \frac{1}{\lr(t)}} \log\left(\frac{1}{\hedgeweight_{\expidxoptpath(t)}(t+1)}\right), \\
	\sum_{t=\tmid+1}^T C_1(t) 
	&= -\sum_{t=\tmid+1}^T\sum_{\expidx\in\experts} 
		\hedgeweight_{\expidx}(t)  \loss_{\expidx}(t), \\
	\sum_{t=\tmid+1}^T C_2(t) 
	&= \sum_{t=\tmid+1}^T \frac{1}{\lr(t)} \log
		\left(
 				 \mathop{\EE}_{\Expidx \sim \hedgeweightvec(t)}\exp\left\{\lr(t) \rbra{-\loss_{\Expidx}(t) - \mathop{\EE}_{\Expidxdum \sim \hedgeweightvec(t)}[-\loss_{\Expidxdum}(t)]} \right\}
		\right), \text{ and }\\
	\sum_{t=\tmid+1}^T D(t) 
	&= \Loss_{\expidxoptpath(T)}(T) - \Loss_{\expidxoptpath(\tmid)}(\tmid).
\]

Thus, combining these with \cref{eqn:lb-psi-bound} gives
\*[
	&\hspace{-2em}-\frac{1}{\lr(T+1)}\log(\hedgeweight_{\expidxoptpath(T)}(T+1)) \\
	&\geq - \sum_{t=\tmid+1}^T \Psi(t) \\
	&\geq \sum_{t=\tmid+1}^T\rbra{\frac{1}{\lr(t+1)} - \frac{1}{\lr(t)}} \log\left(\frac{1}{\hedgeweight_{\expidxoptpath(t)}(t+1)}\right)
	-\sum_{t=\tmid+1}^T\sum_{\expidx\in\experts} 
		\hedgeweight_{\expidx}(t) \loss_{\expidx}(t) \\
	& \qquad + \sum_{t=\tmid+1}^T \frac{1}{\lr(t)} \log
		\left(
 				 \mathop{\EE}_{\Expidx \sim \hedgeweightvec(t)}\exp\left\{\lr(t) \rbra{-\loss_{\Expidx}(t) - \mathop{\EE}_{\Expidxdum \sim \hedgeweightvec(t)}[-\loss_{\Expidxdum}(t)]} \right\}
		\right) \\
	& \qquad + \Loss_{\expidxoptpath(T)}(T) - \Loss_{\expidxoptpath(\tmid)}(\tmid).
\]

Rearranging, we see that
\[\label{eqn:lb-regret-diff-bound}
	&\hspace{-2em}\playerexpregretHedge(T) - \playerexpregretHedge(\tmid) \\
	&\geq \frac{1}{\lr(T+1)}\log(\hedgeweight_{\expidxoptpath(T)}(T+1))
	+ \sum_{t=\tmid+1}^T\rbra{\frac{1}{\lr(t+1)} - \frac{1}{\lr(t)}} \log\left(\frac{1}{\hedgeweight_{\expidxoptpath(t)}(t+1)}\right) \\
	& \qquad + \sum_{t=\tmid+1}^T \frac{1}{\lr(t)} \log
		\left(
 				 \mathop{\EE}_{\Expidx \sim \hedgeweightvec(t)}\exp\left\{\lr(t) \rbra{-\loss_{\Expidx}(t) - \mathop{\EE}_{\Expidxdum \sim \hedgeweightvec(t)}[-\loss_{\Expidxdum}(t)]} \right\}
		\right).
\]
The way we bound these terms will depend on the specific parametrization and \envpolicyname{} chosen for that parametrization.

\subsubsection{\Hedge{} with adversarially optimal parametrization}
\label{ssec:lb-logN-prf}

First, we consider the case of playing \Hedge{} with $\hedgelr(\numexperts) = c\sqrt{\log\numexperts}$. We define the \envpolicyname{} $\policy \in \policyspace(\distnball{})$ such that at round $t$, the distribution on $\dataspace$ is
\*[
	\distn_t = \begin{cases}
		\delta_0^{\otimes (\numeffexperts/2)} \otimes \delta_1^{\otimes (N-\numeffexperts/2)}
			&: t \text{ odd}\\
		\delta_1^{\otimes (\numeffexperts/2)} \otimes \delta_0^{\otimes (\numeffexperts/2)} \otimes\delta_1^{\otimes (N-\numeffexperts)}
			&: t \text{ even}.
	\end{cases}
\]
That is, on even and odd rounds the data alternates between the first half of $\effexperts$ incurring loss of $0$ and the second half of $\effexperts$ incurring loss of $0$, with the remaining $\numexperts - \numeffexperts$ experts always incurring loss of $1$. Both of these distributions are actually in $U$, so they are trivially in $\distnball{}$.

Now, due to the deterministic nature of $\policy$, we can exactly determine what $\hedgeweightvec(t)$ will look like. In particular, we have that
\[\label{eqn:logN-lb-loss}
	\Loss_\expidx(t) = \begin{cases}
		\frac{t-1}{2}
			&: t \text{ odd,} \andT \expidx\in \range{\numeffexperts/2} \\
		\frac{t+1}{2}
			&: t \text{ odd,} \andT \expidx\in \range{\numeffexperts}\setminus\range{\numeffexperts/2}	\\
		\frac{t}{2}
			&: t \text{ even,} \andT \expidx\in \range{\numeffexperts} \\
		t
			&: \expidx\not\in \range{\numeffexperts}.			
	\end{cases}
\]

Thus, recognizing that $\hedgeweight_\expidx(t)$ uses $\Loss_\expidx(t-1)$ and letting $\lbhelper(t) = \exp\left\{-\lr(t) \frac{(t-1)}{2}\right\}$, we can define $\hedgeweight_\expidx(t)$ by
\*[
	\begin{cases}
		[\numeffexperts + (\numexperts-\numeffexperts) \lbhelper(t)]^{-1}
			&: t \text{ odd,} \andT \expidx\in \range{\numeffexperts} \\
		\lbhelper(t)[\numeffexperts + (\numexperts-\numeffexperts) \lbhelper(t)]^{-1}
			&: t \text{ odd,} \andT \expidx\not\in \range{\numeffexperts}	\\
		\exp(\lr(t)/2)[\numeffexperts\cosh(\lr(t)/2) + (\numexperts-\numeffexperts) \lbhelper(t)]^{-1}
			&: t \text{ even,} \andT \expidx\in \range{\numeffexperts/2} \\
		\exp(-\lr(t)/2)[\numeffexperts\cosh(\lr(t)/2) + (\numexperts-\numeffexperts) \lbhelper(t)]^{-1}
			&: t \text{ even,} \andT \expidx\in \range{\numeffexperts}\setminus \range{\numeffexperts/2} \\
		\lbhelper(t)[\numeffexperts\cosh(\lr(t)/2) + (\numexperts-\numeffexperts) \lbhelper(t)]^{-1}
			&: t \text{ even,} \andT \expidx\not\in \range{\numeffexperts}.						
	\end{cases}
\]

The next thing to observe is that for all $t$, $\expidxoptpath(t) \in [\numeffexperts]$ \as{}, and $\hedgeweight_{\expidxoptpath(t)}(t+1)$ equals
\[\label{eqn:optweight-lb-logN}
	\begin{cases}
		[\numeffexperts + (\numexperts-\numeffexperts) \lbhelper(t+1)]^{-1}
			&: t+1 \text{ odd}\\
		\exp(\lr(t+1)/2)[\numeffexperts\cosh(\lr(t+1)/2) + (\numexperts-\numeffexperts) \lbhelper(t+1)]^{-1}
			&: t+1 \text{ even}.						
	\end{cases}
\]

Now, let $\tmid = \floor{\frac{16\log\numexperts}{c^2}}$ and suppose $t\geq \tmid+1$. Then, using $\textstyle{\frac{x}{\sqrt{x+1}}} \geq \textstyle{\frac{1}{2}}\sqrt{x}$ for $x \geq 1$,
\*[
	\lbhelper(t)
	&\leq \lbhelper(t+1) \\
	&= \exp\left\{-\frac{c\sqrt{\log\numexperts}}{\sqrt{t+1}}\frac{t}{2} \right\} \\
	&\leq \exp\left\{-\frac{c\sqrt{t\log\numexperts}}{4} \right\} \\
	&\leq  \exp\left\{-\frac{c\sqrt{16(\log\numexperts)^2}}{4c} \right\} \\
	&= \frac{1}{N}.
\]
This gives
\[
	\frac{1}{\numeffexperts + (\numexperts-\numeffexperts) \lbhelper(t+1)}
	\geq \frac{1}{\numeffexperts + (\numexperts-\numeffexperts)/\numexperts}
	\geq \frac{1}{\numeffexperts + 1}.
\label{eqn:weight-c-lb}
\]
Also, $\exp\left\{\lr(t+1)/2 \right\} \geq 1$, so
\*[
	\cosh\left(\frac{\lr(t+1)}{2} \right)
	& = \frac{1}{2}\left[\exp\left\{\frac{c\sqrt{\log\numexperts}}{2\sqrt{t+1}}\right\} + \exp\left\{-\frac{c\sqrt{\log\numexperts}}{2\sqrt{t+1}}\right\} \right] \\
	&\leq \frac{1}{2}\left[\exp\left\{\frac{c^2\sqrt{\log\numexperts}}{2\sqrt{16\log\numexperts}} \right\} + 1 \right] \\
	&\leq \exp\left\{c^2/8\right\}.
\]
Thus,
\*[
	\frac{\exp\left\{\lr(t+1)/2 \right\}}{\numeffexperts\cosh(\lr(t+1)/2) + (\numexperts-\numeffexperts) \lbhelper(t+1)}
	\geq \frac{1}{\exp\left\{c^2/8\right\}\numeffexperts + 1},
\]
which combined with \cref{eqn:weight-c-lb} gives that for all $t \geq \tmid+1$,
\*[
	\hedgeweight_{\expidxoptpath(t)}(t+1) \geq \frac{1}{\exp\left\{c^2/8\right\}\numeffexperts+1}.
\]
This observation shows that if $T \geq \tmid + 1$,
\[
	\frac{1}{\lr(T+1)}\log(\hedgeweight_{\expidxoptpath(T)}(T+1)) 
		& \geq - \frac{\sqrt{T+1}}{c\sqrt{\log\numexperts}} \Big[c^2/8 + \log(\numeffexperts+1) \Big]. 
\label{eqn:reg-diff-lb1}
\]

In order to control the terms of \cref{eqn:lb-regret-diff-bound}, we first observe that 
\*[
  \sum_{t=\tmid+1}^T\rbra{\frac{1}{\lr(t+1)} - \frac{1}{\lr(t)}} \log\left(\frac{1}{\hedgeweight_{\expidxoptpath(t)}(t+1)}\right) \geq 0.
\] 
Then, we will use \cref{eqn:reg-diff-lb1} to lower bound the first term on the RHS of \cref{eqn:lb-regret-diff-bound}. We now turn to controlling the third term, again supposing $t \geq \tmid+1$. Notice that if $t$ is odd, then $\Expidx \sim \hedgeweightvec(t)$ means $\loss_{\Expidx}(t) \sim \bernoullidist\rbra{ \frac{\numeffexperts/2 +(N-\numeffexperts)\lbhelper(t)}{\numeffexperts +(N-\numeffexperts)\lbhelper(t)}}$. Therefore,
\*[
	&\hspace{-2em}\frac{1}{\lr(t)} \log
		\left(
 				 \mathop{\EE}_{\Expidx \sim \hedgeweightvec(t)}\exp\left\{\lr(t) \rbra{-\loss_{\Expidx}(t) - \mathop{\EE}_{\Expidxdum \sim \hedgeweightvec(t)}[-\loss_{\Expidxdum}(t)]} \right\}
		\right) \\
	&= \frac{1}{\lr(t)} \log
		\bigg(
 				\frac{\numeffexperts/2}{\numeffexperts + (\numexperts-\numeffexperts)\lbhelper(t)}
 				 \exp\left\{\lr(t) \frac{\numeffexperts/2 + (\numexperts-\numeffexperts)\lbhelper(t)}{\numeffexperts + (\numexperts-\numeffexperts)\lbhelper(t)} \right\} \\
 				 & \qquad \qquad \qquad+
 				 \frac{\numeffexperts/2 + (\numexperts-\numeffexperts)\lbhelper(t)}{\numeffexperts + (\numexperts-\numeffexperts)\lbhelper(t)}
 				 \exp\left\{-\lr(t) \frac{\numeffexperts/2}{\numeffexperts + (\numexperts-\numeffexperts)\lbhelper(t)} \right\}
		\bigg) \\
	&\geq \frac{1}{\lr(t)} \log
		\bigg(
 				\frac{\numeffexperts/2}{\numeffexperts + (\numexperts-\numeffexperts)\lbhelper(t)}
 				 \exp\left\{\lr(t) \frac{\numeffexperts/2}{\numeffexperts + (\numexperts-\numeffexperts)\lbhelper(t)} \right\} \\
 				 & \qquad \qquad \qquad+
 				 \frac{\numeffexperts/2}{\numeffexperts + (\numexperts-\numeffexperts)\lbhelper(t)}
 				 \exp\left\{-\lr(t) \frac{\numeffexperts/2}{\numeffexperts + (\numexperts-\numeffexperts)\lbhelper(t)} \right\}
		\bigg) \\
	&= \frac{1}{\lr(t)} \log
		\bigg(
 				\frac{\numeffexperts}{\numeffexperts + (\numexperts-\numeffexperts)\lbhelper(t)}
 				 \cosh\left\{\lr(t) \frac{\numeffexperts/2}{\numeffexperts + (\numexperts-\numeffexperts)\lbhelper(t)} \right\}
		\bigg).
\] 

Now, we observe that $\log(\cosh(x))$ is $\frac{1}{\cosh^2(x_1)}$-strongly convex on $x\in[0,x_1]$. Thus,
\*[
	\log(\cosh(x)) - \log(\cosh(0))
	\geq \frac{d}{dy}\log(\cosh(y)) \bigg\lvert_{y=0} + \frac{x^2}{2\cosh^2(x_1)},
\]
so $\cosh(x) \geq \exp\big\{\frac{x^2}{2 \cosh^2 (x_1 )} \big\}$ on this interval. Then, notice that if $t \geq \tmid+1$, $\lr(t) \leq c^2/4$. So, $\lbhelper(t) \geq 0$ gives

\*[
	\lr(t) \frac{\numeffexperts/2}{\numeffexperts + (\numexperts-\numeffexperts)\lbhelper(t)}
	\leq \frac{\lr(t)}{2}
	\leq \frac{c^2}{8}.
\]
Using this strong-convexity bound on $\cosh(x)$ along with the two inequalities $\lbhelper(t) \leq 1/\numexperts$  and $[2\cosh^2(c^2/8)]^{-1} \geq (1/2)\exp\left\{-c^2/4 \right\}$ results in
\*[
	&\hspace{-2em}\frac{1}{\lr(t)} \log
		\bigg(
 				\frac{\numeffexperts}{\numeffexperts + (\numexperts-\numeffexperts)\lbhelper(t)}
 				 \cosh\left\{\lr(t) \frac{\numeffexperts/2}{\numeffexperts + (\numexperts-\numeffexperts)\lbhelper(t)} \right\}
		\bigg) \\
	&\geq \frac{1}{\lr(t)}\log\bigg(\frac{\numeffexperts}{\numeffexperts + (\numexperts-\numeffexperts)\lbhelper(t)} \bigg)
	 + \frac{\lr(t)}{2\exp\left\{c^2/4 \right\}}\bigg(\frac{\numeffexperts/2}{\numeffexperts + (\numexperts-\numeffexperts)\lbhelper(t)} \bigg)^2 \\
	&\geq \frac{1}{\lr(t)}\log\bigg(\frac{\numeffexperts}{\numeffexperts + (\numexperts-\numeffexperts)\lbhelper(t)}\bigg)
	 + \frac{\lr(t)}{2\exp\left\{c^2/4 \right\}}\bigg(\frac{\numeffexperts}{2\numeffexperts + 1} \bigg)^2 \\
	&\geq \frac{1}{\lr(t)}\log\bigg(\frac{\numeffexperts}{\numeffexperts + (\numexperts-\numeffexperts)\lbhelper(t)} \bigg) + \frac{\lr(t)}{18\exp\left\{c^2/4 \right\}}.
\]

Finally, using $\log(x) \geq 1 - 1/x$,
\*[
	\frac{1}{\lr(t)}\log\bigg(\frac{\numeffexperts}{\numeffexperts + (\numexperts-\numeffexperts)\lbhelper(t)} \bigg)
	\geq \frac{1}{\lr(t)}\bigg(1 - \frac{\numeffexperts + (\numexperts-\numeffexperts)\lbhelper(t)}{\numeffexperts} \bigg)
	\geq - \frac{\numexperts}{\numeffexperts}\frac{\lbhelper(t)}{\lr(t)}.
\]
Thus, when $t \geq \tmid+1$ and $t$ is odd,
\[
	&\hspace{-2em}\frac{1}{\lr(t)} \log
		\left(
 				 \mathop{\EE}_{\Expidx \sim \hedgeweightvec(t)}\exp\left\{\lr(t) \rbra{-\loss_{\Expidx}(t) - \mathop{\EE}_{\Expidxdum \sim \hedgeweightvec(t)}[-\loss_{\Expidxdum}(t)]} \right\}
		\right) \\
	&\geq - \frac{\numexperts}{\numeffexperts}\frac{\lbhelper(t)}{\lr(t)} + \frac{\lr(t)}{18\exp\left\{c^2/4 \right\}}.
\label{eqn:reg-diff-lb2a}
\]

Otherwise, if $t$ is even, then $\Expidx \sim \hedgeweightvec(t)$ implies 
\*[
  \loss_{\Expidx}(t) \sim \bernoullidist\rbra{ \frac{(\numeffexperts/2)\exp\left\{\lr(t)/2 \right\} +(N-\numeffexperts)\lbhelper(t)}{\numeffexperts\cosh\left\{\lr(t)/2 \right\} +(N-\numeffexperts)\lbhelper(t)}}.
\] 
So, using $\cosh(x) \geq 1$,
\*[
	&\hspace{-2em}\frac{1}{\lr(t)} \log
		\left(
 				 \mathop{\EE}_{\Expidx \sim \hedgeweightvec(t)}\exp\left\{\lr(t) \rbra{-\loss_{\Expidx}(t) - \mathop{\EE}_{\Expidxdum \sim \hedgeweightvec(t)}[-\loss_{\Expidxdum}(t)]} \right\}
		\right) \\
	&= \frac{1}{\lr(t)} \log
		\Bigg(
 				\frac{(\numeffexperts/2)\exp\left\{-\lr(t)/2\right\}}{\numeffexperts\cosh\left\{\lr(t)/2 \right\} +(N-\numeffexperts)\lbhelper(t)} \\
 				 &\qquad\qquad\qquad\qquad\times
 				 \exp\left\{\lr(t) \frac{(\numeffexperts/2)\exp\left\{\lr(t)/2\right\} + (\numexperts-\numeffexperts)\lbhelper(t)}{\numeffexperts\cosh\left\{\lr(t)/2 \right\} +(N-\numeffexperts)\lbhelper(t)} \right\} \\
 				 & \qquad \qquad \qquad+
 				 \frac{(\numeffexperts/2)\exp\left\{\lr(t)/2\right\} + (\numexperts-\numeffexperts)\lbhelper(t)}{\numeffexperts\cosh\left\{\lr(t)/2 \right\} +(N-\numeffexperts)\lbhelper(t)} \\
 				 &\qquad\qquad\qquad\qquad\times
 				 \exp\left\{-\lr(t) \frac{(\numeffexperts/2)\exp\left\{-\lr(t)/2\right\}}{\numeffexperts\cosh\left\{\lr(t)/2 \right\} +(N-\numeffexperts)\lbhelper(t)} \right\}
		\Bigg) \\
	&\geq \frac{1}{\lr(t)} \log
		\Bigg(
 				\frac{(\numeffexperts/2)\exp\left\{-\lr(t)/2\right\}}{\numeffexperts\cosh\left\{\lr(t)/2 \right\} +(N-\numeffexperts)\lbhelper(t)} \\
 				&\qquad\qquad\qquad\qquad\times
 				 \exp\left\{\lr(t) \frac{(\numeffexperts/2)\exp\left\{-\lr(t)/2\right\}}{\numeffexperts\cosh\left\{\lr(t)/2 \right\} +(N-\numeffexperts)\lbhelper(t)} \right\} \\
 				 & \qquad \qquad \qquad+
 				 \frac{(\numeffexperts/2)\exp\left\{-\lr(t)/2\right\}}{\numeffexperts\cosh\left\{\lr(t)/2 \right\} +(N-\numeffexperts)\lbhelper(t)} \\
 				 &\qquad\qquad\qquad\qquad\times
 				 \exp\left\{-\lr(t) \frac{(\numeffexperts/2)\exp\left\{-\lr(t)/2\right\}}{\numeffexperts\cosh\left\{\lr(t)/2 \right\} +(N-\numeffexperts)\lbhelper(t)} \right\}
		\Bigg) \\
	&= \frac{1}{\lr(t)} \log
		\bigg(
 				\frac{\numeffexperts\exp\left\{-\lr(t)/2 \right\}}{\numeffexperts\cosh\left\{\lr(t)/2 \right\} + (\numexperts-\numeffexperts)\lbhelper(t)} \\
 				&\qquad\qquad\qquad\qquad\times
 				 \cosh\left\{\lr(t) \frac{(\numeffexperts/2)\exp\left\{-\lr(t)/2 \right\}}{\numeffexperts\cosh\left\{\lr(t)/2 \right\} + (\numexperts-\numeffexperts)\lbhelper(t)} \right\}
		\bigg) \\
	&\geq \frac{1}{\lr(t)} \log
		\bigg(
 				\frac{\numeffexperts\exp\left\{-\lr(t)/2 \right\}}{\numeffexperts\cosh\left\{\lr(t)/2 \right\} + (\numexperts-\numeffexperts)\lbhelper(t)}
		\bigg).
\]

Then, using $\log(x) \geq 1-1/x$,
\[
	&\hspace{-2em}\frac{1}{\lr(t)} \log
		\bigg(
 				\frac{\numeffexperts\exp\left\{-\lr(t)/2 \right\}}{\numeffexperts\cosh\left\{\lr(t)/2 \right\} + (\numexperts-\numeffexperts)\lbhelper(t)}
		\bigg) \\
	&\geq \frac{1}{\lr(t)}\bigg(1 - \frac{\numeffexperts\cosh\left\{\lr(t)/2 \right\} + (\numexperts-\numeffexperts)\lbhelper(t)}{\numeffexperts\exp\left\{-\lr(t)/2 \right\}} \bigg) \\
	&= \frac{1}{\lr(t)}\bigg(\frac{\numeffexperts\sinh\left\{-\lr(t)/2 \right\} - (\numexperts-\numeffexperts)\lbhelper(t)}{\numeffexperts\exp\left\{-\lr(t)/2 \right\}} \bigg) \\
	&\geq -\frac{1}{\lr(t)}\frac{\numexperts}{\numeffexperts}\lbhelper(t)\exp\left\{-\lr(t)/2 \right\} \\
	&= -\frac{\numexperts}{\numeffexperts}\frac{\exp\left\{-\lr(t)(\frac{t-2}{2}) \right\}}{\lr(t)}.
\label{eqn:reg-diff-lb2b}
\]

Combing \cref{eqn:reg-diff-lb2a,eqn:reg-diff-lb2b} and recognizing $\lbhelper(t) \leq \exp\left\{-\lr(t)(\frac{t-2}{2}) \right\}$ gives us
\[
	&\hspace{-2em}\frac{1}{\lr(t)} \log
		\left(
 				 \mathop{\EE}_{\Expidx \sim \hedgeweightvec(t)}\exp\left\{\lr(t) \rbra{-\loss_{\Expidx}(t) - \mathop{\EE}_{\Expidxdum \sim \hedgeweightvec(t)}[-\loss_{\Expidxdum}(t)]} \right\}
		\right) \\
	&\geq -\frac{\numexperts}{\numeffexperts}\frac{\exp\left\{-\lr(t)(\frac{t-2}{2}) \right\}}{\lr(t)}
	+ \frac{\lr(t)}{18\exp\left\{c^2/4 \right\}} \ind{t \text{ is odd}}.
\label{eqn:reg-diff-lb3}
\]

Now, we wish to sum the two terms in \cref{eqn:reg-diff-lb3}. First, using that $\textstyle{\frac{t-2}{\sqrt{t}}} \geq \textstyle{\frac{3\sqrt{t}}{2}}$ when $t \geq 6$ and $\tmid \geq 5$ since $\log\numexperts \geq 5/2$, as well as crudely lower bounding $\tmid$ by dividing by $2$, 
\[
	&\hspace{-2em}\sum_{t=\tmid+1}^T \frac{\exp\left\{-\lr(t)(\frac{t-2}{2}) \right\}}{\lr(t)} \\
	&= \sum_{t=\tmid+1}^T \frac{\sqrt{t}}{c\sqrt{\log\numexperts}} \exp\left\{-\frac{c\sqrt{\log\numexperts}}{\sqrt{t}}\Big(\frac{t-2}{2}\Big) \right\} \\
	&\leq \sum_{t=\tmid+1}^T \frac{\sqrt{t}}{c\sqrt{\log\numexperts}} \exp\left\{-\frac{3c}{4} \sqrt{t\log\numexperts} \right\} \\
	&\leq \frac{1}{c\sqrt{\log\numexperts}}\int_{\tmid}^\infty \sqrt{t}\exp\left\{-\frac{3c}{4} \sqrt{t\log\numexperts} \right\} dt \\
	&= \frac{128\Big(\tmid \frac{9c^2\log\numexperts}{16} + 2\sqrt{\tmid}\frac{3c\sqrt{\log\numexperts}}{4} + 2 \Big)\exp\left\{-(3c/4)\sqrt{\tmid\log\numexperts} \right\}}{27c^4 (\log\numexperts)^2} \\
	&\leq \frac{128\Big(\frac{16\log\numexperts}{c^2} \frac{9c^2\log\numexperts}{16} + 2\frac{4\sqrt{\log\numexperts}}{c}\frac{3c\sqrt{\log\numexperts}}{4} + 2 \Big)\exp\left\{-(3c/4)\frac{\sqrt{8}\log\numexperts}{c} \right\}}{c^4 (\log\numexperts)^2} \\
	&\leq \frac{128\Big(16 + \frac{9}{\log\numexperts} + \frac{2}{(\log\numexperts)^{2}}\Big)}{c^4 \numexperts^2}.
\label{eqn:reg-diff-lb4}
\]

Then, supposing the worst case where both $\tmid+1$ and $T$ are even, and crudely upper bounding $\tmid+2$ by multiplying by $3/2$,
\[
	\sum_{t=\tmid+1}^T \lr(t) \ind{t \text{ is odd}}
	&= \sum_{t=(\tmid+2)/2}^{(T-1)/2} \lr(2t) \\
	&= \sum_{t=(\tmid+2)/2}^{(T-1)/2} \frac{c\sqrt{\log\numexperts}}{\sqrt{2t}} \\
	&\geq \int_{(\tmid+2)/2}^{(T-1)/2} \frac{c\sqrt{\log\numexperts}}{\sqrt{2t}} dt \\
	&= c \sqrt{\log\numexperts}[\sqrt{T-1} - \sqrt{\tmid+2}] \\
	&\geq c\sqrt{(T-1)\log\numexperts} - c\sqrt{\log\numexperts}\frac{\sqrt{24\log\numexperts}}{c} \\
	&= c\sqrt{(T-1)\log\numexperts} - 2\log\numexperts\sqrt{6}.
\label{eqn:reg-diff-lb5}
\]

Thus, combining \cref{eqn:reg-diff-lb1,eqn:reg-diff-lb4,eqn:reg-diff-lb5}, we have shown that for 
$T \geq \frac{16\log\numexperts}{c^2}$,
\*[
	\playerexpregretHedge(T)
	&\geq \playerexpregretHedge(T) - \playerexpregretHedge(\tmid) \\
	&\geq - \frac{\sqrt{T+1}}{c\sqrt{\log\numexperts}} \Big[c^2/8 + \log(\numeffexperts+1) \Big]
	- \frac{128\Big(16 + \frac{9}{\log\numexperts} + \frac{2}{(\log\numexperts)^{2}}\Big)}{c^4 \numeffexperts \numexperts} \\
	&\qquad\qquad
	+ \frac{c\sqrt{(T-1)\log\numexperts}}{18\exp\left\{c^2/4 \right\}} 
	- \frac{2\log\numexperts\sqrt{6}}{18\exp\left\{c^2/4 \right\}}.
\]

Finally, rearranging the restriction on the size $\numeffexperts$ and using $\frac{\sqrt{T-1}}{\sqrt{T+1}} \geq 1/2$, since $\log(\numeffexperts + 1)< \frac{c^2 \log\numexperts}{72\exp\left\{c^2/4 \right\}} - \frac{c^2}{8}$ it holds that
\*[
	\frac{1}{c\sqrt{\log\numexperts}}\Big[c^2/8 + \log(\numeffexperts+1) \Big]
	&< \frac{1}{2} \frac{\sqrt{T-1}}{\sqrt{T+1}} \frac{c\sqrt{\log\numexperts}}{18\exp\left\{c^2/4\right\}}.
\]

Thus, using $\numexperts \geq e^9$ and $\numeffexperts \geq 1$,
\*[
	\playerexpregretHedge(T)
	&\geq \frac{c\sqrt{(T-1)\log\numexperts}}{36\exp\left\{c^2/4 \right\}}
	- \frac{128\Big(16 + \frac{9}{\log\numexperts} + \frac{2}{(\log\numexperts)^{2}}\Big)}{c^4 \numeffexperts \numexperts}
	- \frac{2\log\numexperts\sqrt{6}}{18\exp\left\{c^2/4 \right\}} \\
	&\geq
	\frac{c\sqrt{T\log\numexperts}}{72\exp\left\{c^2/4 \right\}}
	- \frac{1}{3c^2}
	- \frac{\log\numexperts}{3}.
\]

\subsubsection{\Hedge{} with stochastically optimal parametrization}
\label{ssec:lb-c-prf}

Now, we consider the case of playing \Hedge{} with the oracle-informed parameter $\hedgelr(\numexperts,\numeffexperts)$. 
We define the \envpolicyname{} $\policy \in \policyspace(\distnball{})$ such that for some even $t_1$, at round $t$ the distribution on $\dataspace$ is
\*[
	\distn_t = \begin{cases}
		\delta_0^{\otimes (\numeffexperts/2)} \otimes \delta_1^{\otimes (N-\numeffexperts/2)}
			&: t \text{ odd} \andT t\leq t_1 \\
		\delta_1^{\otimes (\numeffexperts/2)} \otimes \delta_0^{\otimes (\numeffexperts/2)} \otimes\delta_1^{\otimes (N-\numeffexperts)}
			&: t \text{ even} \andT t\leq t_1 \\
		\delta_0\otimes \delta_1^{\otimes (N-1)}
				&: t \text{ even} \andT t> t_1.
	\end{cases}
\]
That is, the data is the same as for \Hedge{} in \cref{ssec:lb-logN-prf} up to $t = t_1$, and then afterwards all experts incur loss of $1$ except the first expert, which incurs zero loss. Once again, all of these distributions are actually in $U$, so they are trivially in $\distnball{}$.

Since $t_1$ is even, for $t > t_1$ we expand on \cref{eqn:logN-lb-loss} to obtain
\*[
	\Loss_\expidx(t) = \begin{cases}
		\frac{t_1}{2}
			&: \expidx = 1 \\
		\frac{2t - t_1}{2}
			&: \expidx\in \range{\numeffexperts}\setminus\{1\} \\
		t
			&: \expidx\not\in \range{\numeffexperts}.			
	\end{cases}
\]

Thus, when $t > t_1$, 
\*[
	\hedgeweight_1(t) = \Big[1 + (\numeffexperts-1) \exp\{-\lr(t) (t-t_1-1)\} + (N-\numeffexperts) \exp\{-\lr(t) (t-t_1/2-1)\}\Big]^{-1},
\]
and for $\expidx \neq 1$, $\hedgeweight_\expidx(t)$ equals
\*[
\begin{cases}
	\Big[\exp\{-\lr(t) (t_1-t+1)\} + \numeffexperts-1  + (\numexperts-\numeffexperts) \exp\{-\lr(t) (t_1/2)\}\Big]^{-1}
		&: \expidx\in \range{\numeffexperts}\setminus \set{1} \\
	\Big[\exp\{-\lr(t) (t_1/2-t+1)\} + (\numeffexperts-1)\exp\left\{\lr(t) (t_1/2) \right\}  + \numexperts-\numeffexperts\Big]^{-1}
		&: \expidx\not\in \range{\numeffexperts}.						
\end{cases}
\]

The next thing to observe is that for $t > t_1$, $\hedgeweight_{\expidxoptpath(t)}(t+1)$ equals
\[\label{eqn:best-exp-lb}
	\Big[1 + (\numeffexperts-1) \exp\{-\lr(t+1) (t-t_1)\} + (N-\numeffexperts) \exp\{-\lr(t+1) (t-t_1/2)\}\Big]^{-1}.
\]

Now, define 
\*[
	t_2 = \ceil{4\left(\frac{\log\numexperts}{\hedgelr(\numexperts,\numeffexperts)} + t_1 \right)^2}.
\]
If $t > t_2$, it holds that
\[\label{eqn:t2-lb}
	&\sqrt{t} > \frac{2\log\numexperts}{\hedgelr(\numexperts,\numeffexperts)} + 2t_1 \\
	&\hspace{-2em}\implies \
	\sqrt{t} > \frac{2\log\numeffexperts}{\hedgelr(\numexperts,\numeffexperts)} + 2t_1 \\
	&\hspace{-2em}\implies \
	\hedgelr(\numexperts,\numeffexperts) \frac{\sqrt{t}}{2} - \hedgelr(\numexperts,\numeffexperts)t_1 > \log\numeffexperts \\
	&\hspace{-2em}\implies \
	\frac{\hedgelr(\numexperts,\numeffexperts)[t - t_1]}{\sqrt{t+1}} > \log(\numeffexperts - 1)  \\
	&\hspace{-2em}\implies \ 
	(\numeffexperts-1) \exp\{-\lr(t+1) (t-t_1)\} < 1.
\]
Similarly, 
\[\label{eqn:t3-lb}
	&\sqrt{t} > \frac{2\log\numexperts}{\hedgelr(\numexperts,\numeffexperts)} + 2t_1 \\
	&\hspace{-2em}\implies \
	\sqrt{t} > \frac{2\log\numexperts}{\hedgelr(\numexperts,\numeffexperts)} + t_1 \\
	&\hspace{-2em}\implies \
	\sqrt{t} - t_1 > \frac{2\log\numexperts}{\hedgelr(\numexperts,\numeffexperts)} \\
	&\hspace{-2em}\implies \
	\frac{c[t - t_1/2]}{\sqrt{t+1}} > \log(\numexperts-\numeffexperts) \\
	&\hspace{-2em}\implies \
	(\numexperts-\numeffexperts) \exp\{-\lr(t+1) (t-t_1/2)\} < 1.
\]

Combining \cref{eqn:best-exp-lb,eqn:t2-lb,eqn:t3-lb} shows that when $t > t_2$, since $t_2 > t_1$ by definition, we have
\*[
	\hedgeweight_{\expidxoptpath(t)}(t+1)
		&\geq 1/3.	
\]

This observation controls the first term of \cref{eqn:lb-regret-diff-bound}. For the second term of \cref{eqn:lb-regret-diff-bound}, we note that by Jensen's inequality,
\*[
	\sum_{t=\tmid+1}^T \frac{1}{\lr(t)} \log
		\left(
 				 \mathop{\EE}_{\Expidx \sim \hedgeweightvec(t)}\exp\left\{\lr(t) \rbra{-\loss_{\Expidx}(t) - \mathop{\EE}_{\Expidxdum \sim \hedgeweightvec(t)}[-\loss_{\Expidxdum}(t)]} \right\}
		\right) \geq 0.
\]

Define $\tmid = \floor{\frac{16\log\numexperts}{[\hedgelr(\numexperts,\numeffexperts)]^2}}$. Now, when $\tmid+1 \leq t \leq t_1$, $\hedgeweight_{\expidxoptpath(t)}(t+1)$ behaves as in \cref{eqn:optweight-lb-logN}. Thus, when $t+1$ is odd, $\hedgeweight_{\expidxoptpath(t)}(t+1) \leq 1/\numeffexperts$ since $\lbhelper(t+1) \geq 0$. Otherwise, when $t+1$ is even, we use that since $\log\numexperts > 5/2$,
\*[
	\exp\left\{\frac{\lr(t+1)}{2} \right\}
	\leq \exp\left\{\frac{[\hedgelr(\numexperts,\numeffexperts)]^2}{8\sqrt{\log\numexperts}} \right\}
	\leq \exp\left\{[\hedgelr(\numexperts,\numeffexperts)]^2/8\right\},
\]
as well as $\lbhelper(t+1) \geq 0$ and $\cosh(x) \geq 1$ to obtain $\hedgeweight_{\expidxoptpath(t)}(t+1) \leq \frac{\exp\left\{[\hedgelr(\numexperts,\numeffexperts)]^2/8\right\}}{\numeffexperts}$. Thus,
\[\label{eqn:term3-lb-c}
&\hspace{-2em}\sum_{t=\tmid+1}^{t_1}\rbra{\frac{1}{\lr(t+1)} - \frac{1}{\lr(t)}} \log\left(\frac{1}{\hedgeweight_{\expidxoptpath(t)}(t+1)}\right) \\
	& \geq \sum_{t=\tmid+1}^{t_1}\rbra{\frac{1}{\lr(t+1)} - \frac{1}{\lr(t)}} \Big[\log\numeffexperts- \frac{[\hedgelr(\numexperts,\numeffexperts)]^2}{8}\Big] \\
	&= \frac{\log\numeffexperts - [\hedgelr(\numexperts,\numeffexperts)]^2/8}{\hedgelr(\numexperts,\numeffexperts)} [\sqrt{t_1+1} - \sqrt{\tmid+1}].
\]

Set $t_1 = T/2$, and suppose $T > \frac{32\log\numexperts}{[\hedgelr(\numexperts,\numeffexperts)]^2}$ to ensure $t_1 > \tmid+1$. Then, substituting \cref{eqn:term3-lb-c} into \cref{eqn:lb-regret-diff-bound} gives
\*[
	\playerexpregretHedge(T)
	&\geq \playerexpregretHedge(T) - \playerexpregretHedge(\tmid) \\
	&\geq -\frac{\sqrt{T+1}}{\hedgelr(\numexperts,\numeffexperts)}\log(3)
	+ \frac{\log\numeffexperts - [\hedgelr(\numexperts,\numeffexperts)]^2/8}{\hedgelr(\numexperts,\numeffexperts)} [\sqrt{t_1+1} - \sqrt{\tmid+1}] \\
	&\geq -\frac{\sqrt{T+1}}{\hedgelr(\numexperts,\numeffexperts)}\log(3)
	+ \frac{\log\numeffexperts - [\hedgelr(\numexperts,\numeffexperts)]^2/8}{2\hedgelr(\numexperts,\numeffexperts)} \Big[\sqrt{T+1} - \frac{\sqrt{32\log\numexperts}}{\hedgelr(\numexperts,\numeffexperts)}\Big] \\
	&\geq \frac{\log\numeffexperts}{4\hedgelr(\numexperts,\numeffexperts)}\sqrt{T+1} - \frac{3[\log\numeffexperts - [\hedgelr(\numexperts,\numeffexperts)]^2/8]\log\numexperts}{[\hedgelr(\numexperts,\numeffexperts)]^2} \\
	&\geq \frac{\log\numeffexperts}{4\hedgelr(\numexperts,\numeffexperts)}\sqrt{T} - \frac{3\log\numeffexperts}{[\hedgelr(\numexperts,\numeffexperts)]^2},
\]
where we also used $\log\numeffexperts > [\hedgelr(\numexperts,\numeffexperts)]^2/4 + 4\log(3)$.

\manualendproof
\section{Implementing \FTRLCARE{} and \MetaCARE}
\label{ssec:care-implementation}

The following algorithm efficiently implements \FTRLCARE{}; its validity follows from \cref{thm:ftrl-simplex-summary}.

\begin{algorithm}[H]
	\SetAlgoLined
	\SetKwInput{KwInput}{Inputs}
	\SetKwInput{KwData}{Receive Data}
	\SetKwFunction{KwFn}{Function}
	\KwInput{\\
	$\bullet$ constants $c_1,c_2 >0$, number of experts $\numexperts$\;
	$\bullet$ a function $\texttt{root} : (f,(a,b))\in(\Reals\to\Reals)\times\Reals^2\to \texttt{maybe}(\Reals)$ which returns a root of the function $f$ on the interval $[a,b]$ when $f(a)f(b)<0$, and returns $\texttt{nothing}$ otherwise.}
	\KwResult{Infinite list of weight vectors, $\set{\careweightvec}_{t\in\Nats}$}
	$\entropy = $
		\texttt{Function}$\Big(\weightdumvec\mapsto
			\cbra{ \sum_{\expidx\in\experts} [-\weightdum_\expidx \log(\weightdum_\expidx)]}\Big)$\;
	$\weightvec= $
		\texttt{Function}$\bigg((\lr, \Lossdumvec)
			\mapsto\cbra[3]{
				\rbra[3]{\frac
					{\exp(-\lr \Lossdumvec_\expidx)}
					{\sum_{\expidxdum\in\experts}\exp(-\lr \Lossdumvec_\expidxdum)}}_{\expidx\in\experts}
			}\bigg)$\;
	$\Lossvec(0) = \texttt{zeroes}(\numexperts)$\;
	$\careweightvec(1) = \texttt{ones}(\numexperts) / \numexperts$ \;
	\For{$t\in\Nats$}{
	\KwData{vector of expert losses from round $t$, $\lossvec(t)\in\cinter{0,1}^\experts$}
	$\Lossvec(t) = \Lossvec(t-1)+\lossvec(t)$\;
	$\ilr(t+1) = \texttt{root}\bigg($
		\texttt{Function}$\Big(\lr\mapsto\cbra[2]{
				\eta - \frac{2c_1\sqrt{c_2+\entropy(\weightvec(\lr,\Lossvec(t)))}}{\sqrt{t+1}}
			}\Big)$$,\
			\ointer{\frac{2c_1\sqrt{c_2}}{\sqrt{t+1}}, \frac{2c_1\sqrt{c_2+\log(\numexperts)}}{\sqrt{t+1}}}
		\bigg)$\;
	$\careweightvec(t+1) = \weightvec(\ilr(t+1),\Lossvec(t))$
	}
	\caption{Implementation of \FTRLCARE}
\end{algorithm}

\MetaCARE{} only requires the above implementation of \FTRLCARE{} and a standard implementation of \Hedge{}.
The parameters of \MetaCARE{} can be tuned to optimize the $\numeffexperts=1$ bound of \cref{thm:hedge-bound-quant} and the leading term of \cref{thm:ftrl-care-bound-quant}, hence improving the universal constants, but it does not affect the order of the bound.

\section{Simulations}
\label{sec:simulations}
In this section, we present a brief simulation analysis of the performance of \Hedge{}, \FTRLCARE{}, and \MetaCARE{} to provide intuition for how the algorithms differ and to demonstrate the effectiveness of \MetaCARE{} that we have proved in our analysis.
Since the weights of all three algorithms can be completely determined by the expert losses, we specify each scenario using only the loss distributions rather than the distributions on $\dataspace$ and $\expspace$.
In \cref{fig:1EE-vs-2EE}, we plot the \expregretname{} against the number of rounds $T$ for two \envpolicynames{}: the left column $(\numeffexperts=1)$ corresponds to the stochastic setting, where the losses of the first expert are \iid{} $\bernoullidist(1/2)$ and the losses of all the other experts are \iid{} $\bernoullidist(1)$; the right column $(\numeffexperts=2)$ corresponds to an adversarial setting with two effective experts, where on the $t$th round the loss of the first expert is deterministically $t \text{ mod } 2$, the loss of the second expert is deterministally $(t+1) \text{ mod } 2$, and the losses of the remaining experts are all deterministically $1$.
In \cref{fig:varyN}, we plot the \expregretname{} against the number of experts $\numexperts$ for various $T$.
The \envpolicyname{} has $\numeffexperts=2$, and is the same as for the right column of \cref{fig:1EE-vs-2EE}.
For both settings, the gap between the expected losses of the best effective and ineffective experts under distributions in the convex hull of those produced by the \envpolicyname{} is $\Deltaeff = 1/2$.
For all of the simulations, the algorithms are parametrized using $c_{\shortHedge} = c_{\shortCare,1} = \sqrt{8}$, $c_{\shortCare,2}=1$, and $c_{\shortMeta} = 100$.
All of the plots display \expregretname{}; for the $\numeffexperts=1$ case of \cref{fig:1EE-vs-2EE} this is approximated by averaging over $10$ simulations, and for the remaining plots this is exact since the losses are all deterministic.
\newcommand{\subfigwidth}{0.44\textwidth}
\begin{figure}[H]
  \begin{tabular}{ccc}
    & $\numeffexperts = 1$ & $\numeffexperts = 2$
    \\
    \rotatebox{90}{$\numexperts = 2^{8}$}
    &
    \begin{subfigure}{\subfigwidth}
      \includegraphics[width=\textwidth, trim=0cm 0cm 0cm 0cm, clip]{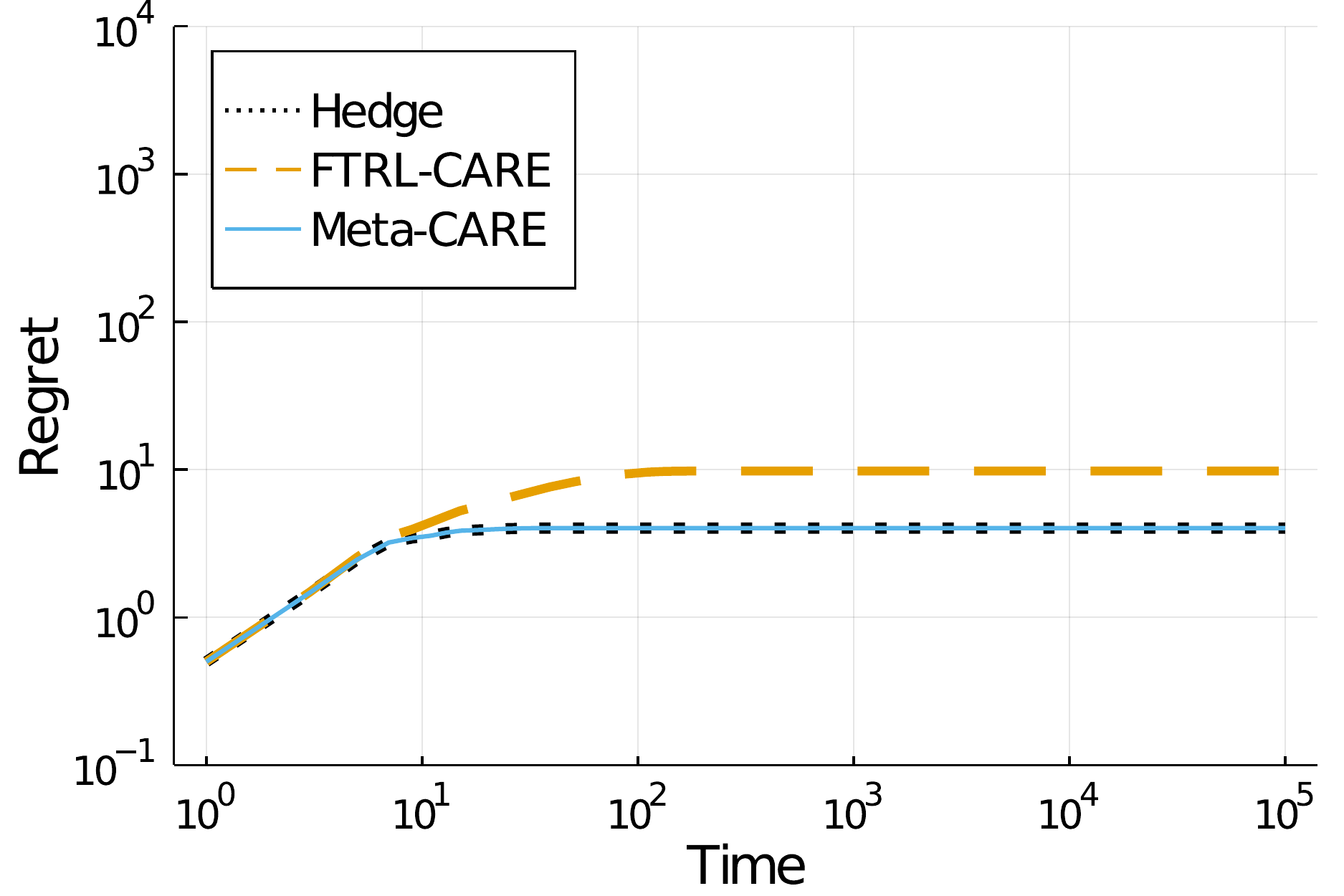}
    \end{subfigure}
    &
    \begin{subfigure}{\subfigwidth}
      \includegraphics[width=\textwidth, trim=0cm 0cm 0cm 0cm, clip]{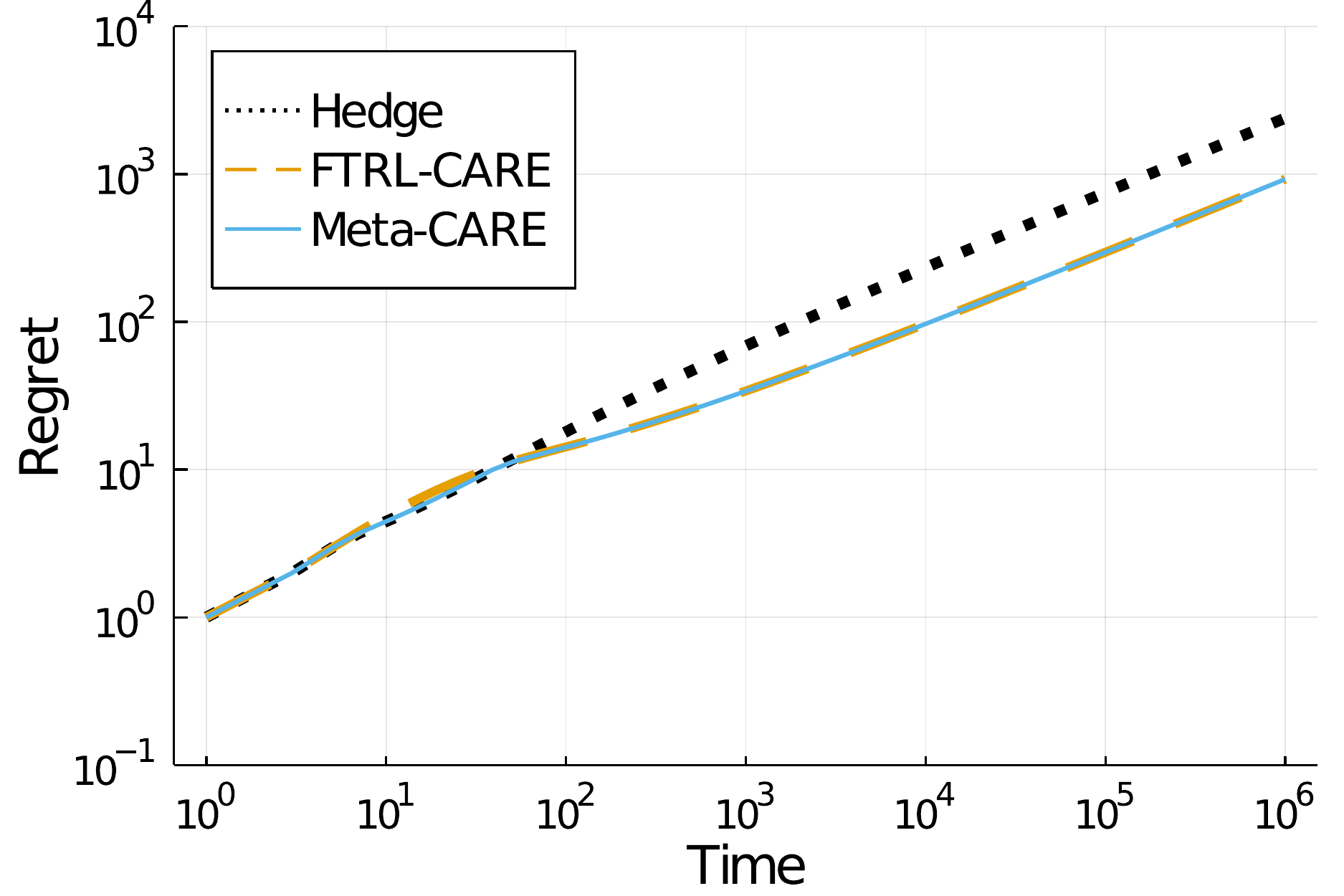}
    \end{subfigure}
    \\
    \rotatebox{90}{$\numexperts = 2^{32}$}
    &
    \begin{subfigure}{\subfigwidth}
      \includegraphics[width=\textwidth, trim=0cm 0cm 0cm 0cm, clip]{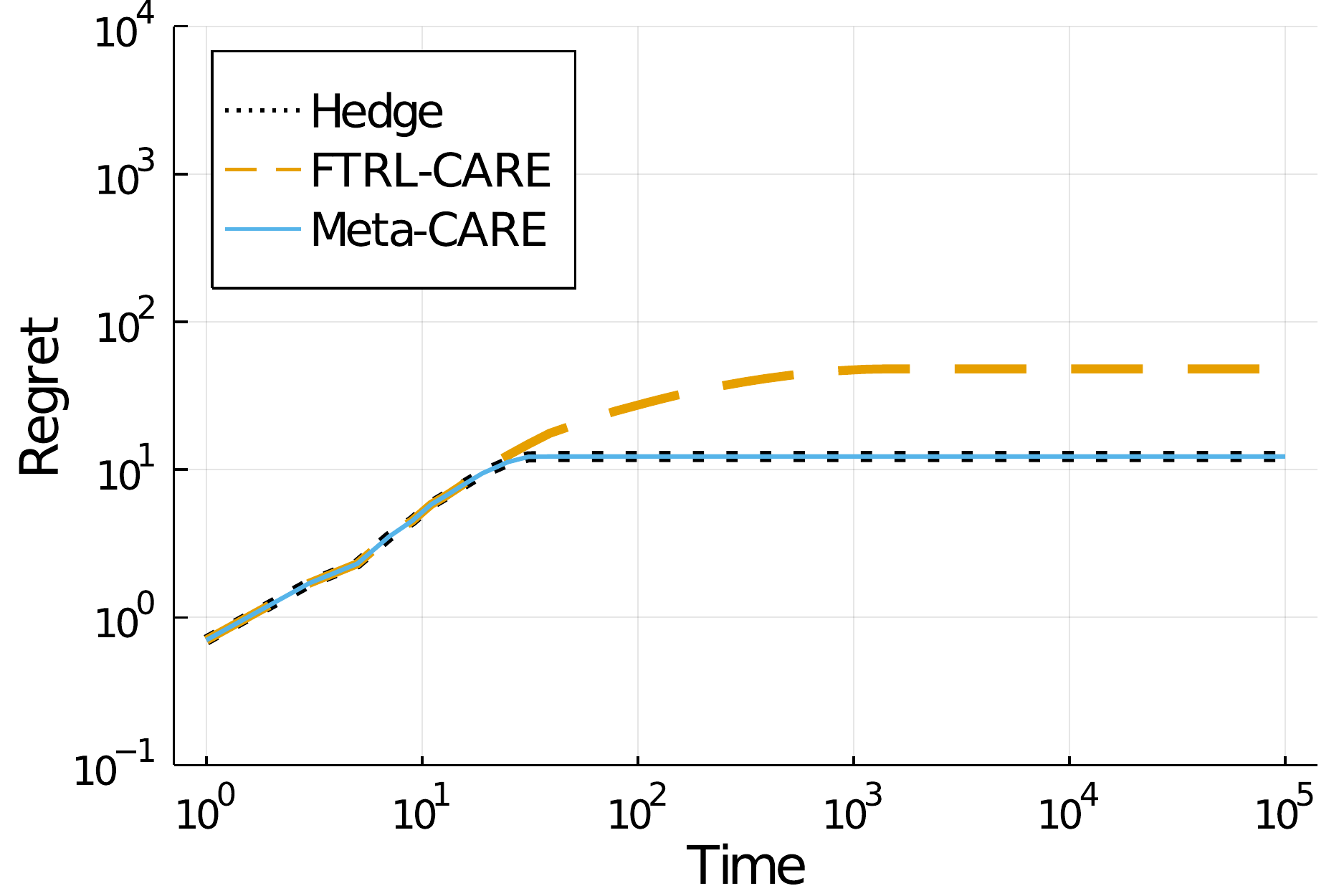}
    \end{subfigure}
    &
    \begin{subfigure}{\subfigwidth}
      \includegraphics[width=\textwidth, trim=0cm 0cm 0cm 0cm, clip]{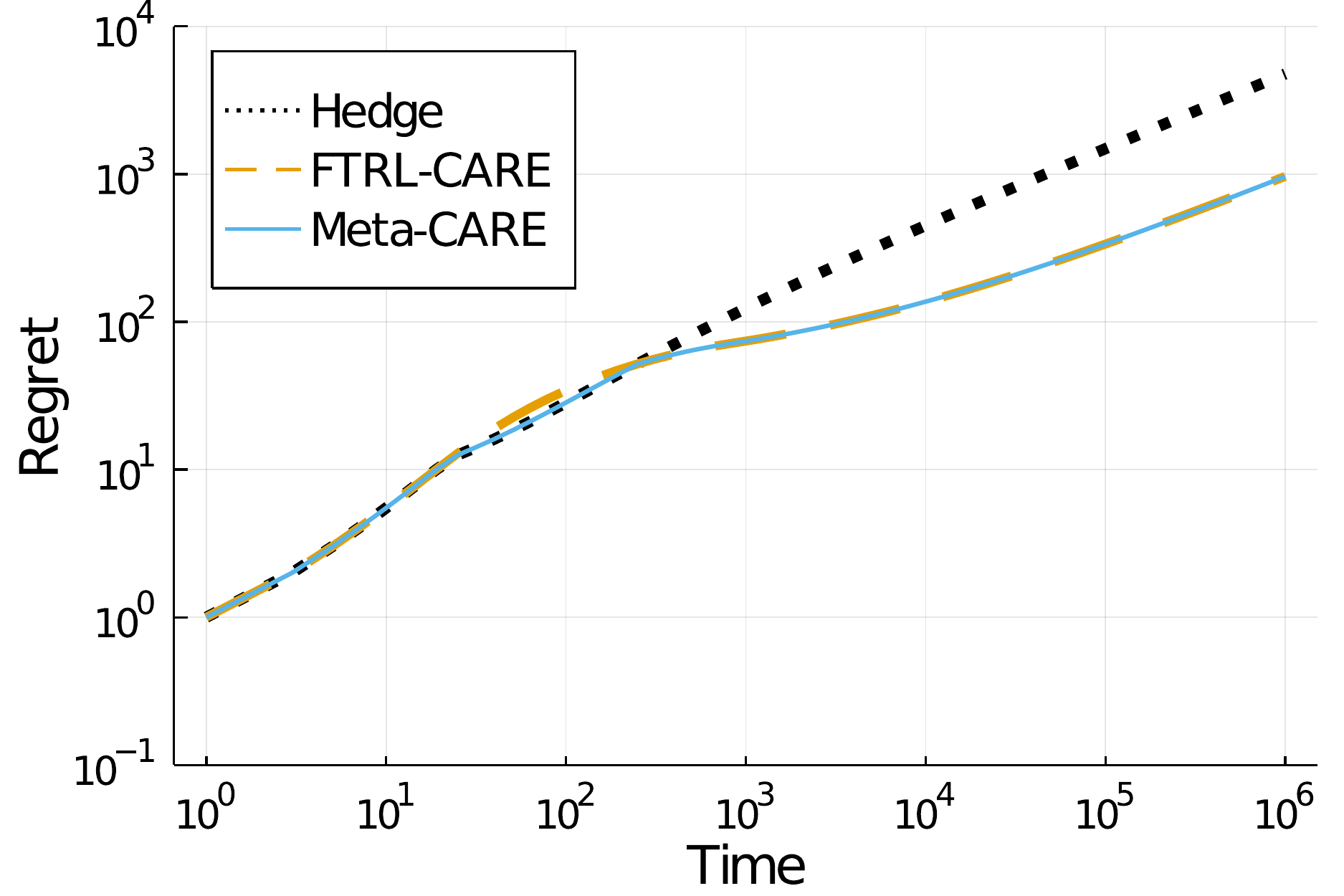}
    \end{subfigure}
    \\
    \rotatebox{90}{$\numexperts = 2^{128}$}
    &
    \begin{subfigure}{\subfigwidth}
      \includegraphics[width=\textwidth, trim=0cm 0cm 0cm 0cm, clip]{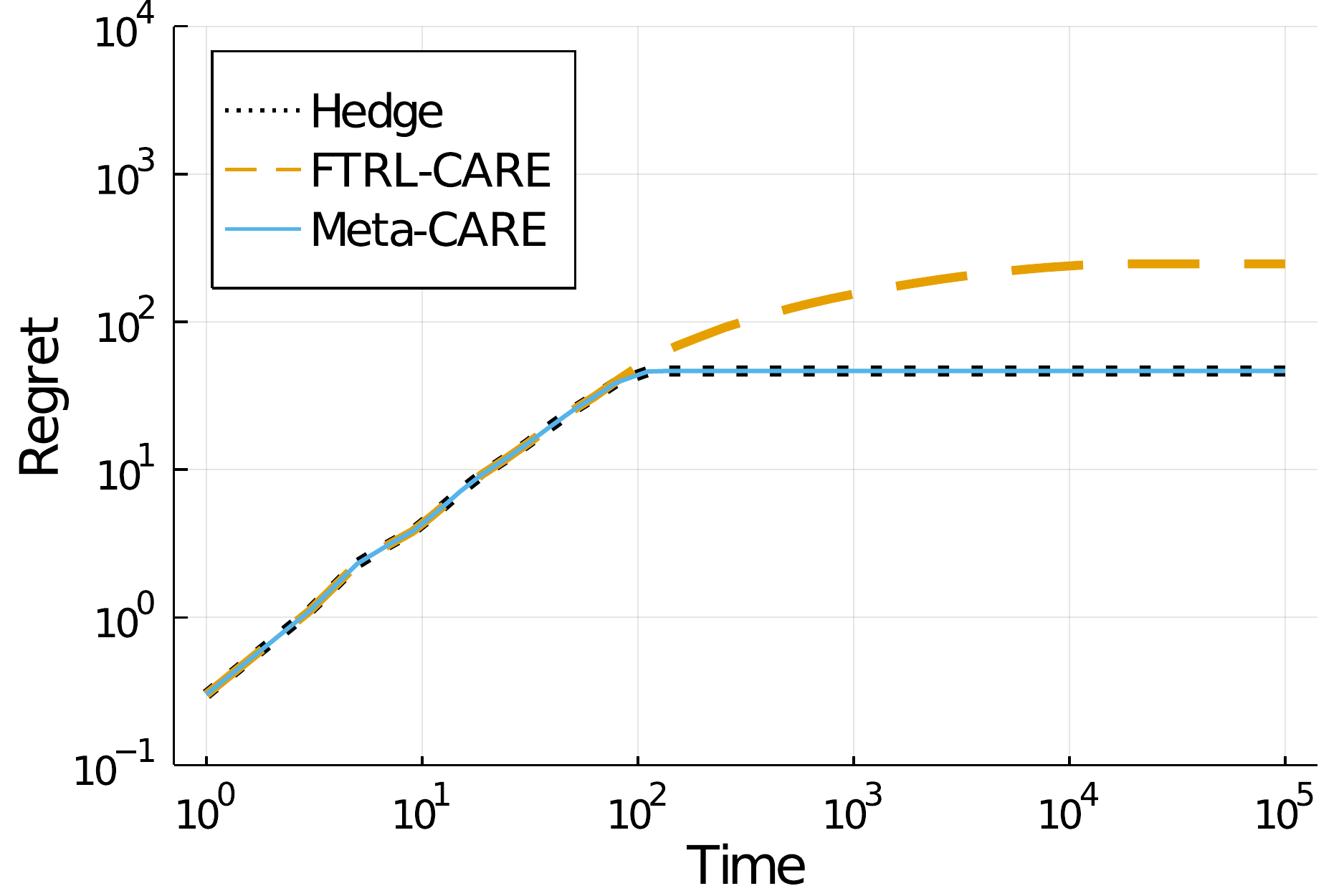}
    \end{subfigure}
    &
    \begin{subfigure}{\subfigwidth}
      \includegraphics[width=\textwidth, trim=0cm 0cm 0cm 0cm, clip]{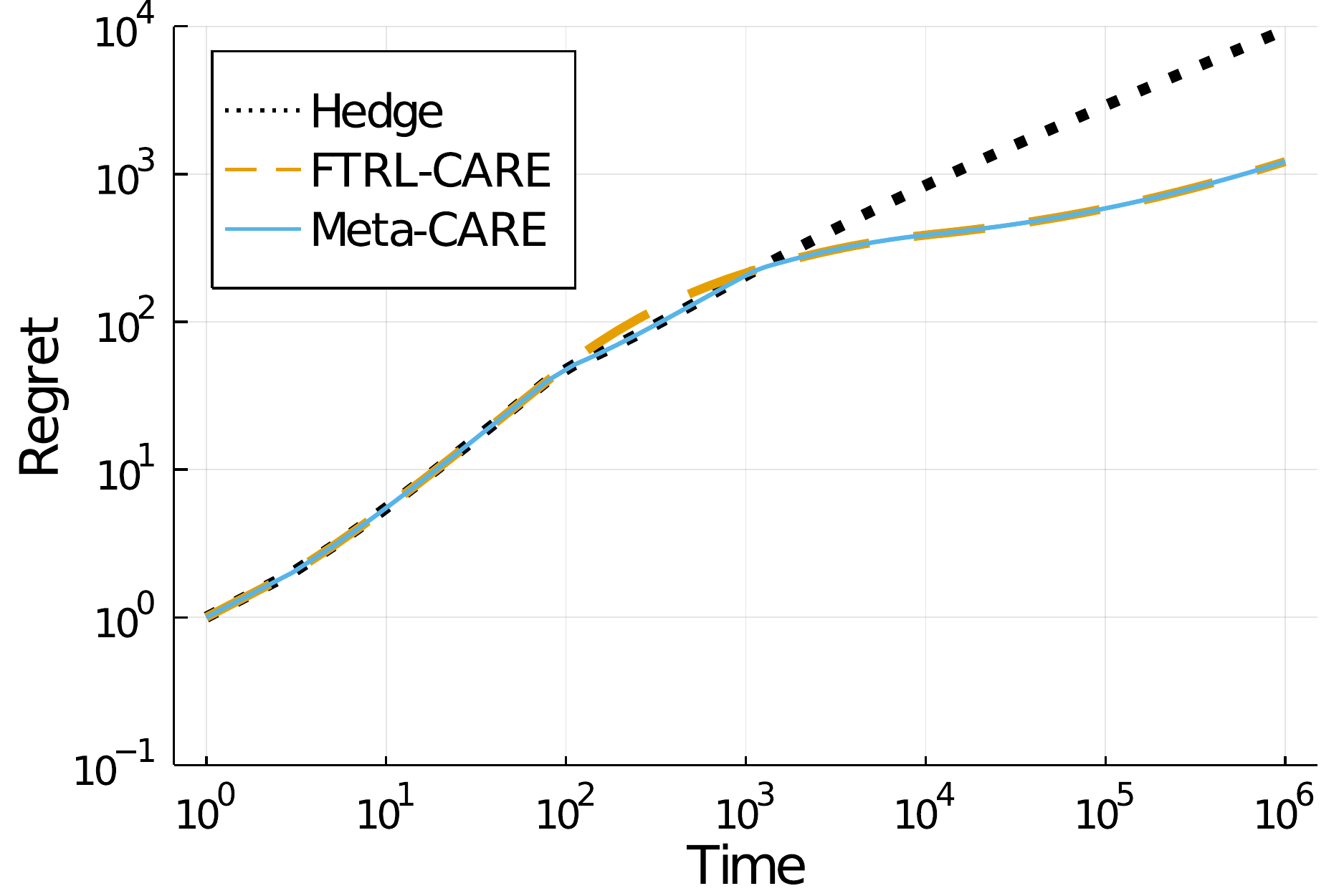}
    \end{subfigure}
\end{tabular}
  \caption{
    Comparing \expregretname{} as a function of time $T$, for number of effective experts $\numeffexperts \in\set{1,2}$ and varying total number of experts $\numexperts$.
    Plots are on a $\log$-$\log$ scale; slopes of lines correspond to polynomial powers, and intercepts of lines correspond to $\log$-(constants of proportionality).
  }
    \label{fig:1EE-vs-2EE}
\end{figure}

\begin{figure}[ht]
  \begin{tabular}{cc}
    $T = 10^3$ & $T = 10^4$
    \\
    \begin{subfigure}{\subfigwidth}
      \includegraphics[width=\textwidth, trim=0cm 0cm 0cm 0cm, clip]{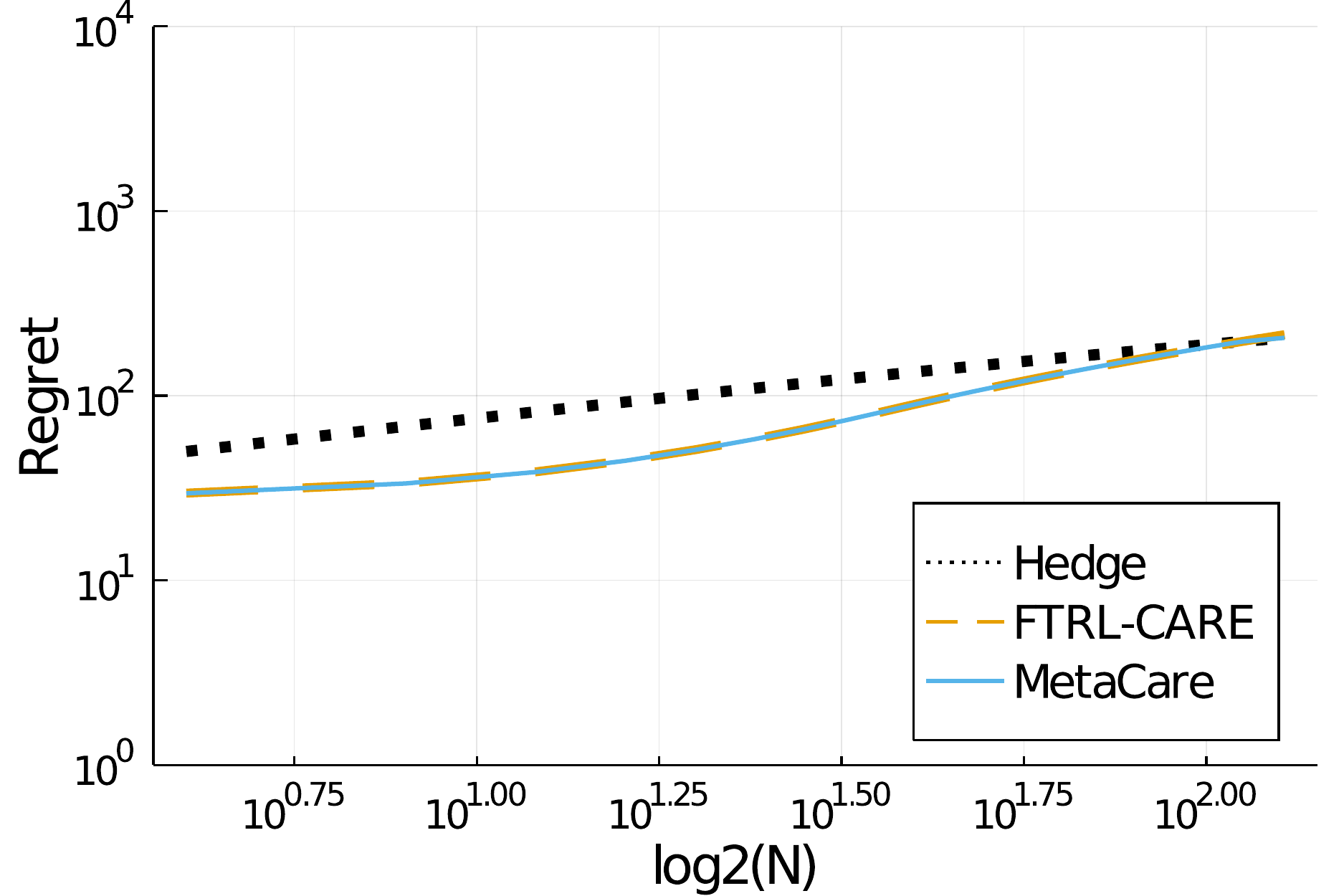}
    \end{subfigure}
    &
    \begin{subfigure}{\subfigwidth}
      \includegraphics[width=\textwidth, trim=0cm 0cm 0cm 0cm, clip]{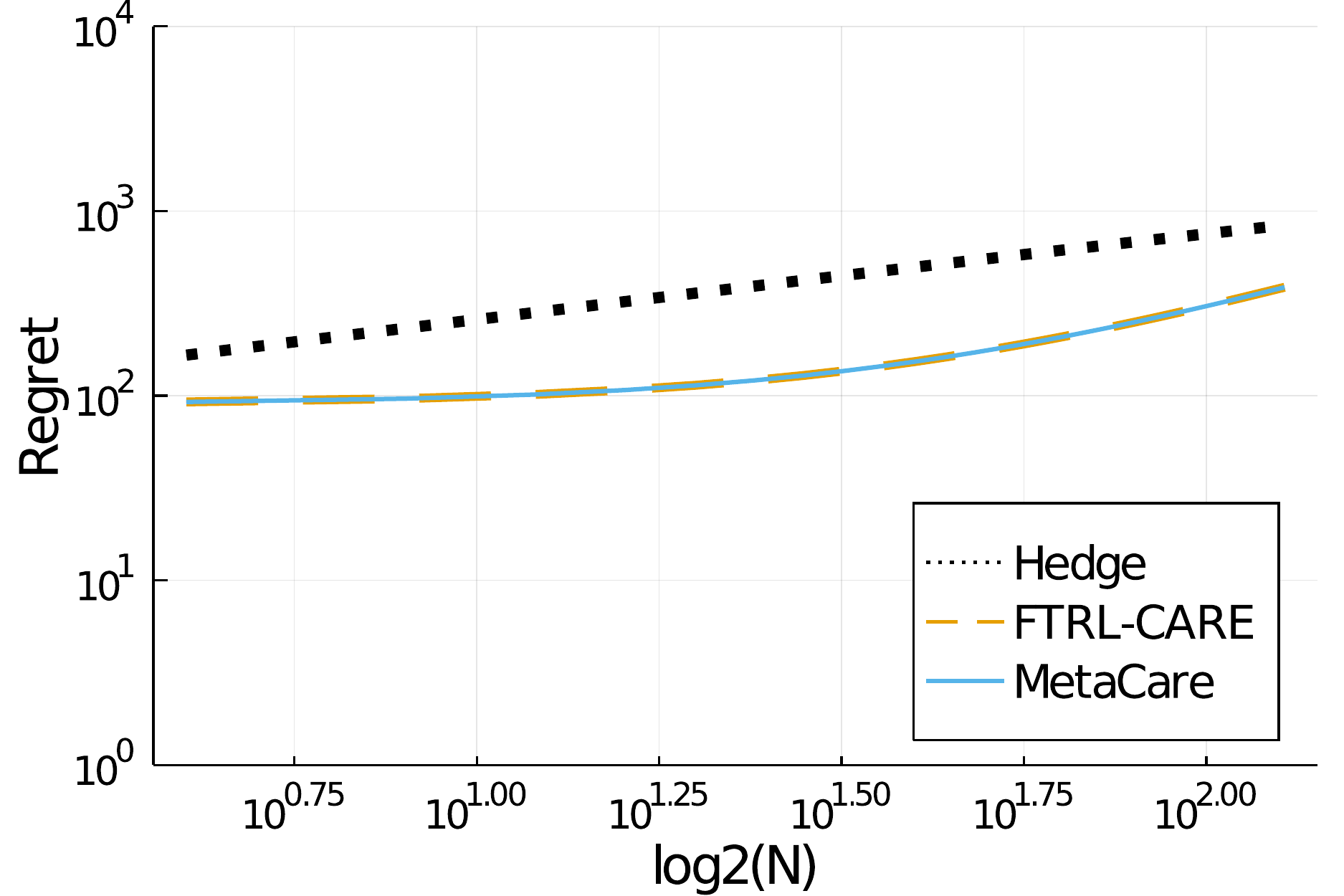}
    \end{subfigure}
    \\
    \\
    $T = 10^5$ & $T = 10^6$
    \\
    \begin{subfigure}{\subfigwidth}
      \includegraphics[width=\textwidth, trim=0cm 0cm 0cm 0cm, clip]{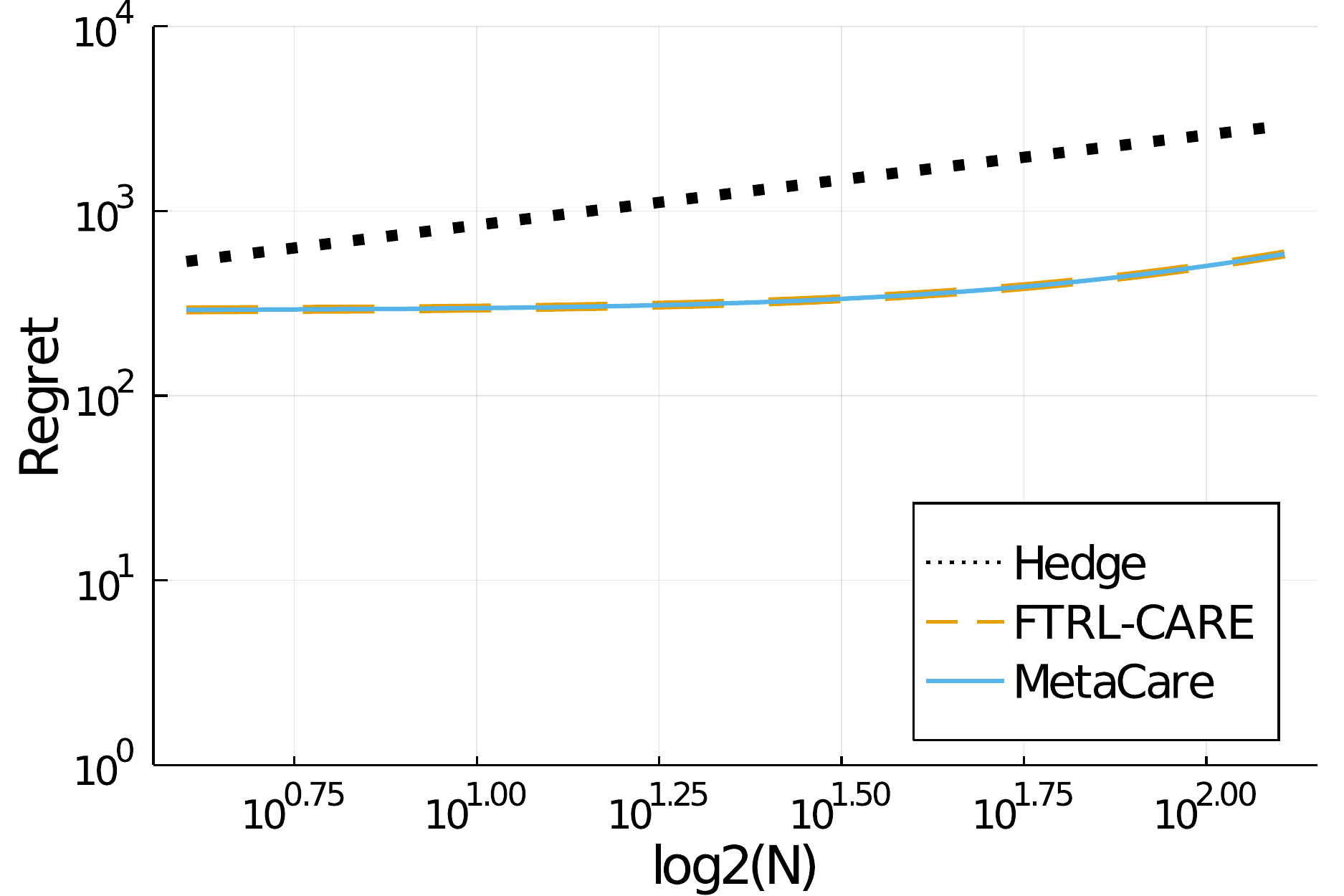}
    \end{subfigure}
    &
    \begin{subfigure}{\subfigwidth}
      \includegraphics[width=\textwidth, trim=0cm 0cm 0cm 0cm, clip]{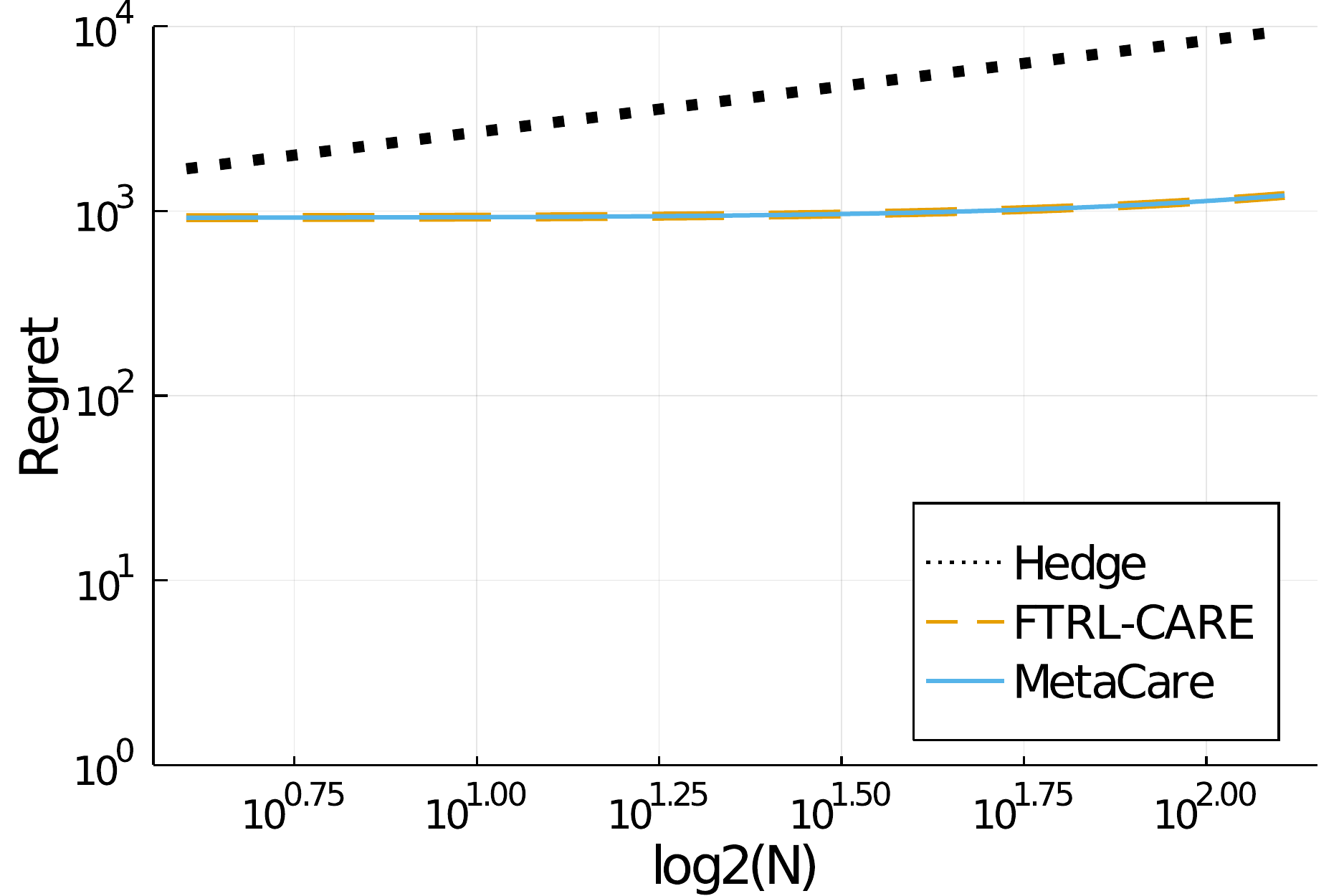}
    \end{subfigure}
\end{tabular}
  \caption{
    Comparing \expregretname{} as a function of the number of experts $\numexperts$ for $\numeffexperts = 2$ effective experts at varying times $T$.
    Plots are on a $\log$-$\log$ scale; slopes of lines correspond to polynomial powers, and intercepts of lines correspond to $\log$-(constants of proportionality).
    Note that since the $x$-axis variable in each case is $\log_2(\numexperts)$, the second tick on the $x$-axis corresponds to $\numexperts = 2^{10^{1.00}} = 1024$, and the last tick on the $x$-axis corresponds to $\numexperts = 2^{10^{2.00}} \approx 1.27\times 10^{30}$.
  }
    \label{fig:varyN}
\end{figure}

\preprint{$ $\\}
Beginning with \cref{fig:1EE-vs-2EE}, for $\numeffexperts =1$, \expregretname{} levels-off at a higher constant for \FTRLCARE{} than for \Hedge{}.
As anticipated by the theory, the period for which adversarial \regretname{} is accumulated before the \regretname{} levels off increases with $\numexperts$ for both \Hedge{} and \FTRLCARE{}, and is longer for \FTRLCARE{}, leading to higher total \expregretname{}.
For $\numeffexperts =2$, the gap between the \expregretname{} of \FTRLCARE{} and \Hedge{} widens as $\numexperts$ increases, corresponding to the $\sqrt{T\log\numexperts}$ \rateregretname{} for \Hedge{} v.s. the $\sqrt{T\log\numeffexperts}$ \rateregretname{} for \FTRLCARE{}.
As anticipated by our theoretical results, there is a phase transition in the \regretname{} accumulation for both \FTRLCARE{} and \Hedge{} at roughly the time when the respective \expregretname{}s level off in the $\numeffexperts = 1$ case.
In all cases, the \expregretname{} of \MetaCARE{} closely tracks the better of \Hedge{} and \FTRLCARE{}.

For \cref{fig:varyN}, when $T$ is small relative to $\log\numexperts$, both \FTRLCARE{} and \Hedge{} have \expregretname{} growing with $\numexperts$ according to the adversarial rate, corresponding to a slope of $1/2$.
When $T$ is large relative to $\numexperts$, so that $\sqrt{T\log\numeffexperts} \gg (\log \numexperts)^{3/2} / \Deltaeff$, the \expregretname{} of \FTRLCARE{} is approximately constant in $\numexperts$ while the \expregretname{} of \Hedge{} grows like $\sqrt{\log \numexperts}$, as anticipated by our theoretical results.
Once again, the \expregretname{} of \MetaCARE{} closely tracks the better of \Hedge{} and \FTRLCARE{}.

\end{document}